\documentclass[11pt]{article}
\usepackage[margin=1in]{geometry}
\usepackage{amsmath,amssymb,amsthm}
\usepackage{mathtools}
\usepackage{graphicx}
\usepackage{booktabs}
\usepackage{enumitem}
\usepackage{float}            % [H] placement: floats stay where written
\usepackage{microtype}
\usepackage[super,numbers,sort&compress]{natbib}
\usepackage[colorlinks=true,linkcolor=blue,citecolor=blue,urlcolor=blue]{hyperref}
\hypersetup{pdfauthor={Shuangxiu Ma, Wenhe Zhao}, pdftitle={When May a Model Replace the Experiment?}}
\allowdisplaybreaks[1]
\newtheorem{theorem}{Theorem}
\newtheorem{proposition}[theorem]{Proposition}
\newtheorem{lemma}[theorem]{Lemma}
\newtheorem{corollary}[theorem]{Corollary}
\newtheorem{definition}[theorem]{Definition}

\theoremstyle{remark}
\newtheorem{remark}[theorem]{Remark}
\theoremstyle{plain}

\newcommand{\fh}{\hat f}
\newcommand{\xh}{\hat x}
\newcommand{\X}{\mathcal{X}}
\newcommand{\eplus}{e_{+}}
\newcommand{\eminus}{e_{-}}
\newcommand{\qsel}{q_{\mathrm{sel}}}
\newcommand{\qinv}{q_{\mathrm{inv}}}
\newcommand{\E}{\mathbb{E}}
\newcommand{\Prob}{\mathrm{P}}
\newcommand{\R}{\mathbb{R}}
\newcommand{\KL}{\mathrm{KL}}
\newcommand{\TV}{d_{\mathrm{TV}}}
\newcommand{\SIref}[1]{SI Thm.~\ref{#1}}

\title{When May a Model Replace the Experiment?\\
Audits, Licenses, and the Price of Trust\\
in Surrogate-Driven Design}

\author{Shuangxiu (Max) Ma\textsuperscript{1,}\thanks{Corresponding
author: \texttt{ma.2061@osu.edu}} \and
Wenhe (Zachary) Zhao\textsuperscript{2}}
\date{%
\textsuperscript{1}William G. Lowrie Department of Chemical and
Biomolecular Engineering, The Ohio State University, Columbus, Ohio
43210, United States\\
\textsuperscript{2}Department of Physics, The Ohio State University,
Columbus, Ohio 43210, United States\\[6pt]
August 24, 2026}

\begin{document}

\maketitle

\begin{abstract}
Design campaigns in chemistry, materials science, and machine learning share
a bottleneck: determining how good a candidate truly is requires an expensive
evaluation---an experiment, a first-principles simulation, or a full training
run. Machine-learning surrogates that predict these outcomes are now used not
only to propose candidates but increasingly to grade them, and even to feed
their own predictions back into the search as though they were measurements.
Here we establish, through mathematical analysis validated on three
exhaustively ground-truthed design tasks, when this practice is safe, what
any certificate of safety must cost, and when the substitution provably pays.
We first show that predictive accuracy cannot anchor trust: near-perfect
$R^2$ is compatible with worst-possible selections, and screening $N$
candidates inflates the over-prediction at the selected candidate by a
quantifiable ``selection tax'' with matching upper and lower bounds. Safety
instead follows from an architectural rule---predictions may propose and
train without restriction, but every certified conclusion must rest on true
evaluations---which we prove is sufficient with no assumptions on the
surrogate, and necessary in the sense that admitting predictions into
certification with the standing of measurements opens a deterministic
self-confirmation failure mode. We then derive the minimal
criterion under which a model may act as an oracle (rank preservation, not
accuracy), show that trust must be purchased through selection-aware audits
that are provably optimal in query complexity---no audit design of any kind
can certify deployed selection quality with asymptotically fewer oracle
queries---and
prove a dichotomy
specifying exactly when audited surrogates reduce certified evaluation
cost (tabular fixed menus: never; generative search: by the screening
factor, conditional on the audit).
Experimentally, across 432 surrogate fits spanning six task--regime
conditions, the audit statistic derived here---a deployment-simulating
measurement, cheap by construction---tracks deployed search performance
at Spearman rank correlation $0.89$--$0.99$ at measurement budgets and
$0.80$--$0.94$ at the prescribed 120-round audit budget (NAS-Bench-201,
biased data: $0.97$, 95\% CI $0.95$--$0.98$), while the rank
correlation of $R^2$ with deployed regret falls as low as $0.33$
($0.08$--$0.54$); and
audited screening reduces certified oracle cost by a measured factor of
25 with audit and training amortized ($14\times$ all-in at a 100-design
campaign).
\end{abstract}

% ====================================================================
\section{Introduction}

Whether the object of a search is a CO$_2$-capture sorbent, a catalyst
formulation, a protein binder, or a neural-network architecture, the
governing economics are the same: candidates are cheap to enumerate and
expensive to evaluate. The evaluation---wet-lab synthesis and
characterization, a density-functional-theory calculation, a docking
campaign, or training a network to convergence---plays the role of an
\emph{oracle}, and its cost sets the pace of discovery. The universal
accelerant is a \emph{surrogate}: a machine-learning model $\fh$ trained to
approximate the oracle $f$, cheap enough to score millions of candidates in
the time a single true evaluation
takes.\cite{trabucco2021conservative,trabucco2022design,white2021powerful}

Surrogates entered these workflows as proposal engines that rank candidates
for the oracle's attention. Their role has quietly expanded. Surrogate
scores are reported as outcomes; searches terminate because the surrogate is
confident; and surrogate predictions---\emph{pseudo-labels}---are inserted
into training sets and replay buffers as though they were measurements. The
stakes of this expansion are no longer hypothetical. The largest
ML-driven materials campaign to date screened millions of candidate
crystals by predicted stability,\cite{merchant2023scaling} and an
autonomous laboratory guided by those rankings reported dozens of novel
compounds synthesized in weeks;\cite{szymanski2023autonomous} independent
re-examination of the same evidence concluded that most of the claimed
successes were misidentified or already
known,\cite{leeman2024challenges,cheetham2024ai} a correction issued not by
better accuracy metrics but by expert re-audit of the pipeline's actual top
picks. The field's own benchmarking has since documented the underlying
pattern: models ranked best by average error are demonstrably not the
models that make the best discovery decisions.\cite{riebesell2025matbench}
The hazards are individually documented across machine learning as well.
Optimization pressure exploits the errors of learned
objectives,\cite{trabucco2021conservative,gao2023scaling} repeated sampling
against imperfect verifiers saturates and then
degrades,\cite{stroebl2024flaws} hyperparameter tuning routinely overfits
its own validation signal,\cite{schneider2025overtuning} and selected
results systematically disappoint---the optimizer's curse of decision
analysis\cite{smith2006optimizer} and the winner's curse of auction
theory,\cite{capen1971competitive} whose bias-correction machinery is an
active research area of its own.\cite{iyengar2023oic,mclatchie2024order}
What has been missing is a unified, quantitative treatment with guarantees. Three questions define the gap.
First, \emph{safety}: under what conditions can a search that consumes
surrogate predictions---including its own pseudo-labels---still deliver
conclusions with statistical warranties? Second, \emph{verification}: what
must it cost to check that a given surrogate is trustworthy for a given
deployment, and how should that check be designed? Third, \emph{economics}:
does the substitution of surrogate for oracle ever reduce the number of true
evaluations required to reach a \emph{certified} result, or only the number
required to reach a claimed one?

This paper answers all three questions. The theoretical framework
(Section~\ref{sec:theory}, with complete proofs in the Supporting
Information) establishes: (i) sharp separations showing that accuracy
metrics are structurally blind to decision quality, including matching upper
and lower bounds on the selection tax that argmax screening levies on any
imperfect predictor; (ii) a \emph{separation principle}---proposal and
certification must draw on disjoint information---that is provably
sufficient, with an explicit probability-one failure construction showing
that no guarantee of its kind survives admitting predictions as
measurements; (iii) the oracle
cost of certified search, governed by the density of near-ties rather than
by global difficulty; (iv) an \emph{oracle admissibility} theory
identifying rank preservation as the exact criterion for surrogate-as-oracle
use, together with the audit designs that can and cannot verify it at a
profit---including a universal lower bound proving that the selection-aware
(champion) audit is optimal among \emph{all} audit designs for this
guarantee; and (v) a dichotomy resolving the economics, with an explicit,
conditional, and gracefully degrading cost separation for audited
surrogates, robust to oracle noise through an explicit value slack. The experimental program (Sections~\ref{sec:methods}
and~\ref{sec:results}) validates the framework's deployment-facing
predictions on three design tasks for which we hold exhaustive ground truth---1{,}728
hyperparameter configurations evaluated by actually training every model,
and the 15{,}625 architectures of NAS-Bench-201\cite{dong2020nasbench}---so
that audits and deployments can be scored against exact answers rather than
estimates. A closing budget study prices the remaining free parameter, the
labels themselves: with training labels and audit rounds drawn from one
oracle ledger, it maps where in the (learning, certification) budget plane
each license becomes earnable, which labeling policies get there cheapest,
and what model error actually costs at the point of selection
(Section~\ref{sec:res-budget}).

Beyond the individual results, the paper advances a methodological position:
trust in a surrogate is not a property of the model but a \emph{measured,
perishable relationship} between a model, a task, and a deployment---earned
through audits that imitate the deployment, and voided by retraining or
distribution drift. Figure~\ref{fig:framework} summarizes the framework.

\begin{figure}[H]
\centering
\includegraphics[width=0.96\textwidth]{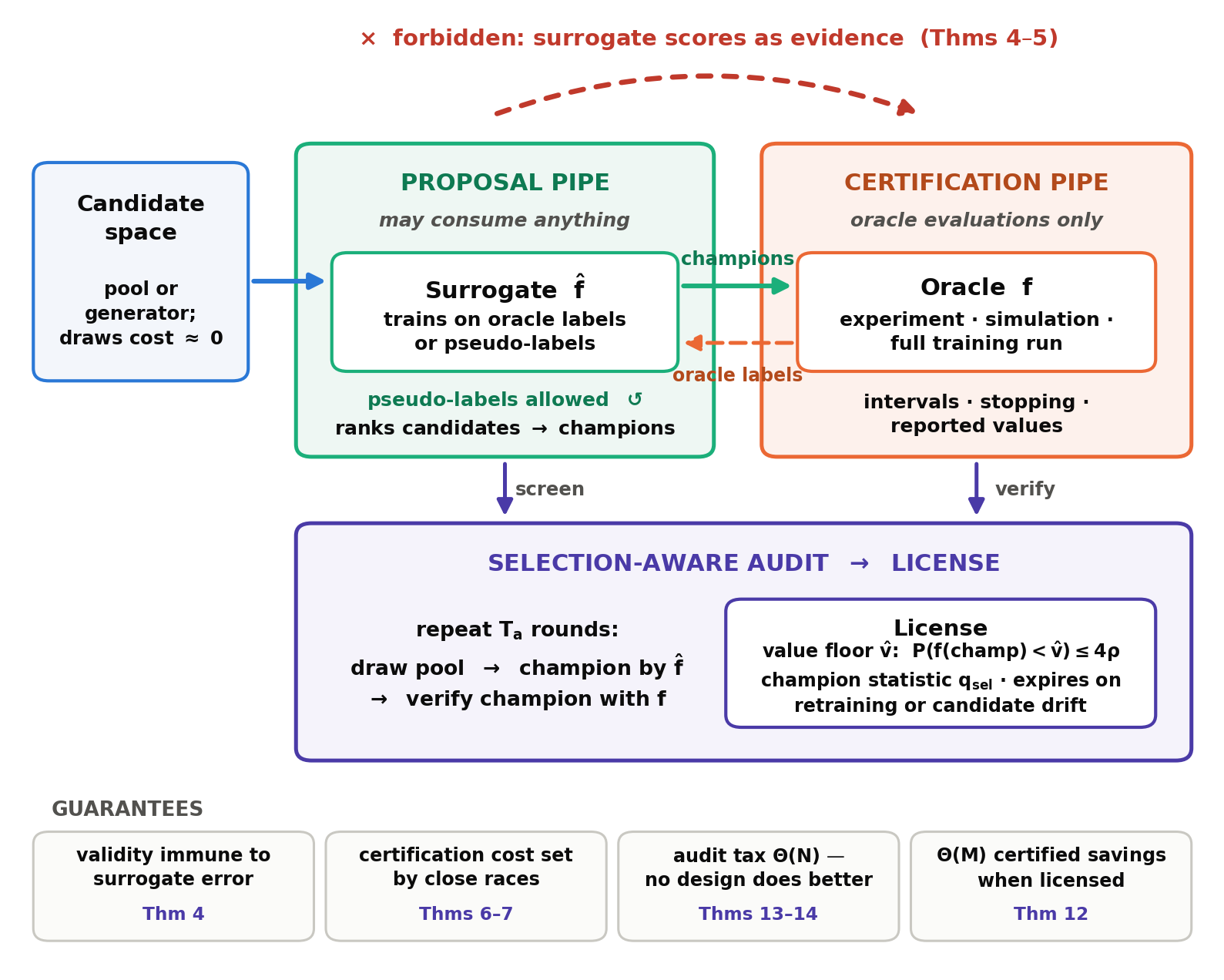}
\caption{The certified surrogate-use framework. Information flows through
two pipes with distinct privileges. The \emph{proposal pipe} (green) may
consume anything---oracle labels, its own pseudo-labels, arbitrary
retraining---and controls only what is tried next. The \emph{certification
pipe} (orange) computes confidence intervals, eliminations, stopping
decisions, and reported values from oracle evaluations alone; surrogate
output never enters it (red dashed arrow; Theorems~\ref{thm:invariance}
and~\ref{thm:contamination}). Trust in the surrogate itself is established
by a \emph{selection-aware audit} (violet): repeatedly draw a candidate
pool, let the surrogate select its champion, verify champions with the
oracle, and issue a license---a certified value floor and inversion
rate---that expires on retraining or drift.}
\label{fig:framework}
\end{figure}

% ====================================================================
\section{Theoretical Framework}\label{sec:theory}

Formal statements at full generality, all proofs, and the constants
suppressed below appear in SI Sections~S1--S7; each theorem here cites its
SI counterpart. The numerical checks that accompany the theory are
described in Section~\ref{sec:verifmeth}.

\subsection{Setting and Definitions}\label{sec:setting}

A design task consists of a candidate space $\X$, a true objective
$f:\X\to[0,B]$ accessible only through costly evaluations (exact or
sub-Gaussian noisy), and a surrogate $\fh:\X\to\mathbb{R}$ of arbitrary
provenance. A search screens candidate pools and selects by the surrogate:
for a pool $C\subseteq\X$, the \emph{champion} is
$\hat x=\arg\max_{x\in C}\fh(x)$. Decision quality is measured by
\emph{regret}, $r_C=\max_{x\in C}f(x)-f(\hat x)$: the true value forgone by
trusting the surrogate's choice, a quantity visible only to the oracle. We
write $e_+(x)=[\fh(x)-f(x)]_+$ for over-prediction (``false optimism'') and
$\Delta_x=\max_{x'} f(x')-f(x)$ for a candidate's gap to the best. Two
deployment goals are distinguished throughout, because their economics
differ qualitatively: \emph{fixed-menu selection} (certify the best of a
given finite list) and \emph{generative search} (produce some design
certified to exceed a quality bar, with unlimited cheap candidate draws).

\subsection{Predictive Accuracy Does Not Control Decision
Quality}\label{sec:accuracy}

The default trust signals for surrogates are aggregate accuracy metrics. A
search, however, does not consume its surrogate on average; it consumes it
at one adversarially selected point---the top of the surrogate's own
ranking. This mismatch admits exact statements.

\begin{theorem}[accuracy is neither sufficient nor necessary; \SIref{si:thm:r2}]
\label{thm:r2}
For every $\epsilon>0$ there exist $(f,\fh)$ with $R^2(f,\fh)\ge 1-\epsilon$
whose champion is a worst-possible candidate (regret equal to the full range
of $f$); and there exist $(f,\fh)$ with $R^2$ arbitrarily negative whose
champion is exactly optimal.
\end{theorem}

\begin{theorem}[equal error, opposite outcomes; \SIref{si:thm:mse}]
\label{thm:mse}
For every $\Delta>0$ and pool size $N>4$ there exist surrogates
$\fh_A,\fh_B$ with identical mean squared error whose regrets are,
respectively, $0$ and $\Delta$ (maximal). The distinguishing information
resides entirely in the upper tail of $e_+$, which aggregate metrics
suppress and selection amplifies.
\end{theorem}

The amplification is itself a law rather than a tendency:

\begin{theorem}[the selection tax; \SIref{si:thm:evt}]\label{thm:evt}
If prediction errors are mean-zero and $\sigma$-sub-Gaussian, then screening
$N$ candidates inflates expected false optimism at the champion to at most
$2\sigma\sqrt{\ln N}$; and there are instances with i.i.d.\ Gaussian errors
on which it is at least $(1-1/e)\,\sigma\,\Phi^{-1}(1-1/N)$. The
$\sqrt{\ln N}$ growth is therefore exact in order: even an unbiased
surrogate, screened widely, systematically over-promises at the top.
\end{theorem}

Together, Theorems~\ref{thm:r2}--\ref{thm:evt} imply that any safety
argument beginning ``the surrogate is accurate, therefore\ldots'' is
unsound at its first step. Safety must come from elsewhere.

\subsection{The Separation Principle}\label{sec:separation}

Every datum in a search serves one of two functions. \emph{Proposal}:
choosing what to try next, training models, generating candidates.
\emph{Certification}: concluding---which candidates are eliminated, when the
search stops, what value is reported, and with what warranty. Our central
architectural result is that restricting \emph{certification} to
oracle-derived statistics is exactly the right boundary: sufficient for
safety with no assumptions on the surrogate, and necessary in the precise
sense of Theorem~\ref{thm:contamination} below---pseudo-labels admitted
with the standing of measurements admit a deterministic failure mode.
What the necessity result does \emph{not} exclude is calibrated or
width-inflated use of model output inside certification with
oracle-labeled calibration (the selection-conditional conformal route
discussed in the SI); the theorem closes the door on treating predictions
as data, not on every use of predictions.

\begin{theorem}[invariance; \SIref{si:thm:invariance}]\label{thm:invariance}
Consider any search whose certification statistics (anytime-valid confidence
intervals, eliminations, stopping rules, reported results) are computed from
oracle evaluations only, while its proposal machinery---including surrogates
trained on arbitrary pseudo-labels---is unrestricted. Then with probability
at least $1-\delta$, every certified conclusion is correct (the certified
champion is optimal; an $\varepsilon$-certified output is within
$\varepsilon$ of optimal), \emph{uniformly over all surrogate behaviors},
including adversarial ones constructed with knowledge of $f$. Surrogate
error can waste oracle budget; it cannot corrupt a certificate.
\end{theorem}

\begin{theorem}[necessity: contaminated certification fails; SI
Thm.~\ref{si:thm:contamination}]\label{thm:contamination}
Conversely, call a protocol \emph{contaminated} if accepted
pseudo-labels enter its certification statistics with the standing of
measurements. There exist a protocol of this class, a two-candidate
task, and a surrogate such that the protocol certifies the wrong
candidate \emph{with probability one}, at \emph{zero} oracle cost, its
confidence generated entirely by the model's agreement with
itself---so no guarantee of Theorem~\ref{thm:invariance}'s form can
hold for any contaminated class containing this behavior.
\end{theorem}

Theorem~\ref{thm:contamination} formalizes self-confirmation as a
deterministic failure mode rather than misfortune, and closes a loophole
with practical teeth: a surrogate cannot participate in the decision to
trust the surrogate---which will return as the license-renewal rule in
Section~\ref{sec:admissibility}.

\subsection{The Oracle Cost of Certified Search}\label{sec:cost}

Given the separation principle, the price of certainty is set not by a
task's global difficulty but by its \emph{close races}.

\begin{theorem}[per-candidate cost; \SIref{si:thm:se}]\label{thm:se}
In certified elimination search with $\sigma$-noisy oracles, a candidate
with gap $\Delta_x$ is eliminated after
$O\!\big(\sigma^2\Delta_x^{-2}\log(N/\delta)\big)$ evaluations, and the
eventual winner requires only enough precision to defeat its closest rival.
\end{theorem}

\begin{theorem}[three cost regimes; \SIref{si:thm:margin}]\label{thm:margin}
If the fraction of candidates within $u$ of optimal is bounded by $Cu^\alpha$
(a margin condition), the total certified cost to precision $\varepsilon$
scales as $\varepsilon^{-(2-\alpha)}$ for $\alpha<2$, as
$\log(1/\varepsilon)$ for $\alpha=2$, and is \emph{independent of}
$\varepsilon$ for $\alpha>2$: beyond a fixed ambiguous core of genuinely
contested candidates, additional precision is free.
\end{theorem}

Practically, Theorem~\ref{thm:margin} converts an early measurement of a
task's gap distribution into a forecast of its oracle budget---before the
campaign is run.

\subsection{Oracle Admissibility: Criteria, Audits, and
Licenses}\label{sec:admissibility}

When may the surrogate itself act as the oracle? We answer in three parts:
the criterion, the impossibility of verifying it retail, and the audit that
verifies it wholesale.

\begin{theorem}[the criterion is ordinal; \SIref{si:thm:ordinal}]\label{thm:ordinal}
A surrogate's selections are within $\varepsilon$ of optimal on
\emph{every} candidate pool if and only if it never inverts the order of two
candidates whose true values differ by more than $\varepsilon$. Rank
preservation on separated pairs is thus the exact entrance criterion for
oracle duty; numerical accuracy enters only insofar as it purchases
order.
\end{theorem}

\noindent This is the criterion the learning-to-rank school of offline
design has independently converged on from the training
side.\cite{lyu2025rank}

\begin{theorem}[no profitable retail verification; \SIref{si:thm:noprofit}]
\label{thm:noprofit}
Certifying that the surrogate resolved one contested comparison correctly
requires, in oracle evaluations, as much as resolving that comparison
with the oracle directly. Trust therefore cannot be bought per decision;
it must be amortized.
\end{theorem}

\noindent The proof is a change-of-measure lower bound in the tradition
of best-arm identification.\cite{kaufmann2016complexity,karnin2016verification}

\begin{theorem}[the audit--deployment gap, and its repair; SI
Thm.~\ref{si:thm:audit}]\label{thm:audit}
Let $p$ be the fraction of candidates a surrogate over-predicts beyond
tolerance, as measured by a conventional audit on random samples. Under
argmax deployment over pools of $N$ candidates, the probability that the
\emph{champion} is such a candidate can reach $\Theta(\min(1,Np))$, with a
matching construction: selection seeks out precisely the errors the audit
certified as rare. Consequently, guaranteeing deployed failure $\rho$ from
passive audits costs $\Omega(N/\rho)$ oracle evaluations---whereas a
\emph{selection-aware audit}, which repeatedly draws a pool, lets the
surrogate select, and verifies only champions, certifies the deployed
failure rate directly with $O(\rho^{-1}\log(1/\delta))$ evaluations.
\end{theorem}

\noindent Audit the workflow, not the distribution.

\begin{theorem}[one-sided licenses; \SIref{si:thm:pessimism}]\label{thm:pessimism}
An audit certifying only the one-sided property ``over-predicts by more than
$\tau$ on at most a fraction $p$ of candidates'' yields a
\emph{no-disappointment guarantee}: with probability at least $1-Np$, the
true value of any selected candidate is at least its predicted value minus
$\tau$.
\end{theorem}

\noindent This is the cheapest useful license, and it grounds the common
practice of conservative surrogate
modeling\cite{trabucco2021conservative,jin2023selection,jinren2025focal,fannjiang2022conformal}
in a checkable certificate rather than a training heuristic.

Licenses are perishable: they bind one trained model to one candidate
distribution, degrade in proportion to distribution drift, and---by
Theorem~\ref{thm:contamination}---cannot be renewed by the surrogate's own
labels.

\subsection{The Economics of Surrogate Use}\label{sec:economics}

Certification consumes only oracle data; can surrogates nevertheless reduce
the certified bill? The answer is a dichotomy that turns on the deployment
goal, and it resolves what we regard as the central open question of
surrogate-driven design.

\begin{theorem}[strict-improvement dichotomy; SI Thms.~\ref{si:thm:si1}--\ref{si:cor:si3}]
\label{thm:dichotomy}
(a) For fixed-menu selection in the tabular model (candidates carry no
shared structure), no use of pseudo-labels can reduce certified oracle
cost: every menu item requires its own oracle evidence. (b) For
generative search to a quality bar met by a fraction $p_v$ of candidates,
any oracle-only algorithm requires $\Omega(1/p_v)$ evaluations per certified
design, whereas
screening $M$ free candidates through a surrogate holding a selection-aware
license and verifying only champions certifies designs at that bar for
$O(1)$ evaluations each, audit and training amortized---a
$\Theta(M/\log(1/\rho))$ separation, \emph{conditional on the audit
passing}, with graceful degradation to the lower bound otherwise.
\end{theorem}

\noindent The part-(b) baseline is classical (quantile
bandits\cite{chaudhuri2017quantile,aziz2018pure}); part (a)'s tabular
qualifier matters---structured classes (linear, RKHS) escape the
tabular floor and remain the open problem of SI Section~S10.
The separation is thus real but never free: it is exactly as large as the
audit says it is. Where multi-fidelity optimization frameworks
\emph{assume} a bound on the cheap signal's
error,\cite{kandasamy2019mfbo,poiani2024mfbai} and proxy-assisted
best-arm methods assume a known bias or correlation
structure,\cite{ao2026llmjudges,ma2026probe} the present framework
\emph{measures} the cheap signal's deployed quality and prices the
measurement. The guarantee also survives oracle noise without extra audit
rounds: a noisy audit licenses the same quantile with an explicit value
slack of order $\sigma\sqrt{\log(1/\rho)}$ (\SIref{si:thm:noisy}).

\begin{theorem}[ranking audits: champion versus pairwise; SI
Thms.~\ref{si:thm:ag1}--\ref{si:thm:ag2}]\label{thm:auditgap}
For the ordinal criterion of Theorem~\ref{thm:ordinal}, auditing random
candidate \emph{pairs} for inversions carries a deployment blind spot of
exactly $\Theta(N^2)$: failures requiring two rare candidates to co-occur in
a pool are priced by the pairwise audit at the product of their rarities.
Auditing the \emph{champion}---one fresh candidate against the winner of an
independent pool, two oracle calls per round---has a blind spot of exactly
$\Theta(N)$, provably tight.
\end{theorem}

\noindent Concretely, the two statistics are: $\qinv$, the probability
that an ordered i.i.d.\ pair is inverted by the surrogate despite an
$\varepsilon$-separation in true value; and $\qsel$, the probability that
one fresh draw beats the champion of an independent pool by more than
$\varepsilon$---each measurable at two oracle calls per round (formal
definitions in SI Section~S7). The audit statistic, not the audit
budget, determines what an audit can see; the same principle underlies
risk-limiting election audits.\cite{stark2012rla}

\begin{theorem}[no audit design does better; \SIref{si:thm:lb}]
\label{thm:lb}
The champion audit's $\Theta(N)$ price is intrinsic rather than an artifact
of the two statistics above: \emph{any} auditor---adaptive, randomized,
informed of the surrogate and the candidate distribution, drawing candidates
at will---that certifies deployed selection failure below $\rho_0$, and
that retains an even chance of passing a designated perfectly safe
instance (one on which the surrogate's mistakes are provably harmless),
must spend $\Omega(N/\rho_0)$ oracle queries in the worst case. Selection-aware champion auditing is therefore
optimal among \emph{all} audit designs for this guarantee, up to a
$\log(1/\delta)$ factor.
\end{theorem}

\begin{remark}[what $\qsel$ does and does not measure]
\label{rem:conservative}
Theorem~\ref{thm:lb} is a statement about \emph{query complexity}: no
audit design certifies deployed selection quality with asymptotically
fewer oracle calls. It is not a claim that
$N\,\overline{q}_\delta(\qsel)$ is the tightest attainable
\emph{numerical} certificate, and the distinction matters in practice.
By construction $\qsel$ counts a fresh draw that beats the champion
whether or not the surrogate would itself have selected that draw had it
been in the pool. Restricting the count to draws the surrogate ranks
below the champion---the only ones that can witness a regret
failure---yields a smaller statistic $\qsel^{\ast} \le \qsel$, so the certificate
reported here is conservative. Two consequences deserve emphasis. First,
the certificate does not vanish for a perfect selector: a surrogate whose
scores are the true values has regret identically zero at every $N$ and
every $\varepsilon$, yet its $\qsel$ is strictly positive whenever the
value distribution has mass above the champion. The certified bound is
therefore a property of the audit statistic together with the landscape,
and must not be read as a property of the landscape alone---in
particular, a bound exceeding $\rho_0$ for an idealized surrogate does
\emph{not} establish that the guarantee is unattainable on that
population. Second, whenever ground truth is exhaustively available, the
deployed failure probability is directly estimable at no additional
oracle cost, and that direct estimate should be reported alongside the
certificate; on the benchmark landscapes we have examined the certificate
runs several-fold above it. Live campaigns, where the pool-best is never
observed, are precisely the regime in which the conservative certificate
is the only available instrument, and are the regime
Theorem~\ref{thm:lb} addresses. This remark governs the reading of
every empirical saturation reported below, in particular
Figure~\ref{fig:budgetmaps}(c): where the certificate does not reach a
license level at any budget, the correct conclusion is that this
statistic does not earn it at these settings, never that no selector
could.
\end{remark}

\noindent Value licenses (Theorem~\ref{thm:dichotomy}b) escape the
$N$-floor only because they certify a different, weaker object---which is
precisely why the framework recommends them.

% ====================================================================
\section{Methods}\label{sec:methods}

\subsection{Ground-Truthed Design Tasks}\label{sec:tasks}

Testing statements about certified search requires tasks whose ground truth
is known exhaustively, so that deployed failure, regret, and audit validity
can be scored exactly rather than estimated. We constructed or adopted
three. \emph{Task 1 (HPO-GBM):} all 1{,}152 configurations of a
gradient-boosting regressor (learning rate, tree count, depth, subsampling,
leaf size, feature fraction) evaluated by 3-fold cross-validation on the
diabetes clinical dataset; the oracle for each configuration is an actual
training run. \emph{Task 2 (HPO-SVM):} all 576 configurations of a
support-vector classifier (kernel family and degree, $C$, $\gamma$,
$\mathrm{coef}_0$, class weighting) evaluated identically on the
breast-cancer dataset. \emph{Task 3 (NAS-Bench-201; NB201):} the complete benchmark
of 15{,}625 convolutional cells with trained accuracies on
CIFAR-10;\cite{dong2020nasbench} validation accuracy serves as the search
objective and held-out test accuracy is reported for selections, following
standard NAS practice. Across tasks the oracle tables comprise 17{,}353 real
training outcomes. Decision tolerances $\varepsilon=\tau$ were fixed at 5\%
of the objective range (Tasks 1--2) and 2 accuracy points (Task 3) before
experiments were run.

\subsection{Surrogate Models and Training Regimes}\label{sec:zoo}

Surrogate zoos span the model classes used as performance predictors in
practice: random forests, extremely randomized trees, gradient boosting,
$k$-nearest neighbors, ridge regression, and a conservative
(quantile-$0.2$) gradient-boosting variant, each fit at multiple training
sizes (40--160 labeled candidates on Tasks 1--2, 100--400 on Task 3) and
seeds; configurations are featurized
by log-scaled numeric encodings (Tasks 1--2) or one-hot cell operations
(Task 3). Each zoo is trained under two regimes: \emph{representative}
(training candidates sampled uniformly) and \emph{biased} (training
candidates restricted to the below-median half of the objective), the
latter reproducing the data condition of offline model-based
design,\cite{trabucco2022design} where surrogates must extrapolate toward
the optimum.

\subsection{Audit and Deployment Protocols}\label{sec:protocols}

Deployment is simulated as argmax selection over pools of $N$ i.i.d.\
candidates, $N\in\{3,\allowbreak 10,\allowbreak 30,\allowbreak
100,\allowbreak 300,\allowbreak 1000\}$, with 3{,}000--6{,}000
replicate pools per condition. Passive audits estimate the marginal
over-prediction rate $p$ on uniform samples; selection-aware audits estimate the deployed
failure rate on champions, and the champion-inversion statistic $\qsel$ is
estimated from independent (pool, fresh-draw) pairs at two oracle calls per
round, exactly as licensed by Theorem~\ref{thm:auditgap}. Verified-search
experiments run recommendation loops with a shared warm-up (40 oracle
evaluations) under four proposal policies (random; representative-trained
RF; biased-trained ridge; and no-verification trust), 20 seeds each: at
every budget the recommendation is the best \emph{empirical mean among
oracle-evaluated candidates}---the separation principle's constraint on
where conclusions come from---and we report its true regret. (This
exercises the invariance of recommendation validity to proposal quality;
it does not run the full certified-elimination protocol of
Theorem~\ref{thm:se}, whose stopping rules and certificates are
exercised in the numerical verification suite of
Section~\ref{sec:verifmeth}.) Screening experiments measure oracle
evaluations per certified-good design at the top-1\% quality bar
($M=100$), comparing oracle-only search against licensed screening with
single-champion verification.

\subsection{Numerical Verification and
Reproducibility}\label{sec:verifmeth}

Every theorem with checkable content is exercised by scripted numerical
tests (more than 50 checks across four suites): exact counterexample
algebra for Theorems~\ref{thm:r2}--\ref{thm:evt}; empirical anytime
coverage under adversarial querying; the deterministic self-confirmation
trace of Theorem~\ref{thm:contamination}; exhaustive verification of the
ordinal equivalence over all pools of small spaces; Monte Carlo agreement
of the selection-tax construction to three decimals; the audit
lower-bound construction of Theorem~\ref{thm:lb}; and the noisy-license
slack of \SIref{si:thm:noisy}, including the asymmetric-noise
counterexample that makes the slack necessary. Every figure is generated
by scripted code with fixed seeds, and all reported statistics were
reproduced from scratch in a fresh computational environment to the
third decimal.

% ====================================================================
\section{Results and Discussion}\label{sec:results}

\subsection{The Selection Tax Is Real, Bounded, and a Property of the
Individual Fit}\label{sec:res-tax}

Figure~\ref{fig:tax} shows the deployed failure rate---the probability that
the champion of an $N$-wide screen is over-predicted beyond
tolerance---across surrogates, pool sizes, training regimes, and ten
independent fits per condition on Task 1. (Terminology, fixed
throughout this section: \emph{deployed failure} $q$ is this champion
over-prediction probability; the \emph{champion-inversion rate} $\qsel$
of Theorem~\ref{thm:auditgap} is the $\varepsilon$-regret audit
statistic, a fresh draw beating the champion; a \emph{license
violation} is a breach of a certified value floor.
Table~\ref{si:tbl:notation} in the SI collects all symbols.) Three observations. First, the
worst-case law of Theorem~\ref{thm:audit} was never violated in any of the
several hundred (fit, regime, $N$) cells measured across the three tasks.
Second, the audited marginal rate says remarkably little about deployed
risk, \emph{in either direction}: interpolating surrogates (random forest,
$k$-NN) audited at $p$ up to $0.33$ deploy at failure rates that fall to
zero as $N$ grows---the audit overstates risk by an order of
magnitude---while the linear ridge surrogate splits, fit by fit, between
the same benign decay and near-certain deployed failure
(Figure~\ref{fig:tax}c,d). Twin ridge fits on Task 2---same recipe, same
regime, audited at $p=0.247$ and $p=0.252$---deployed at failure $0.000$
and $1.000$ respectively at $N=1000$. Third, therefore, the error--selection
alignment that decides whether the tax bites is not a property of the model
class, the training regime, or the audited rate, but of the individual
fit on the individual task: the same ridge recipe is bimodal on Tasks 1--2
yet deterministically catastrophic on Task 3, reaching $95\%$ deployed
failure at $N=100$ and $100\%$ at $N=1000$ in \emph{both} training regimes,
its linear extrapolation systematically optimistic in the combinatorial
architecture space. Reputation does not transfer; recipes do not transfer;
only an audit of the deployed fit discriminates---which is the empirical
face of the per-fit licensing that
Theorems~\ref{thm:audit} and~\ref{thm:pessimism} prescribe.

\begin{figure}[H]
\centering
\includegraphics[width=0.49\textwidth]{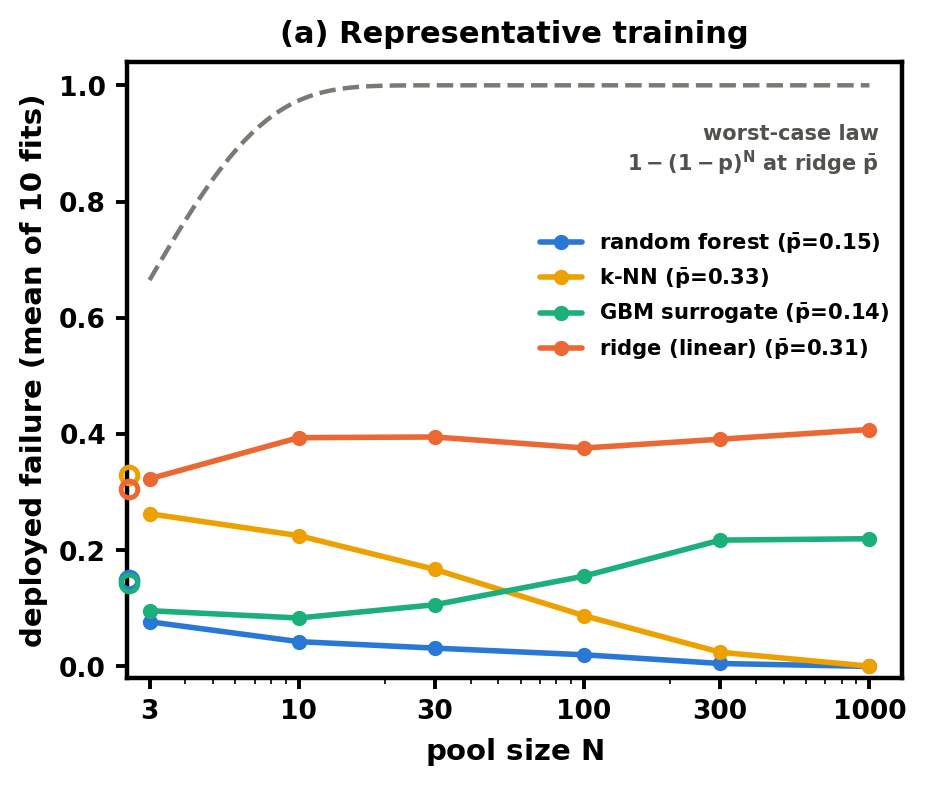}\hfill
\includegraphics[width=0.49\textwidth]{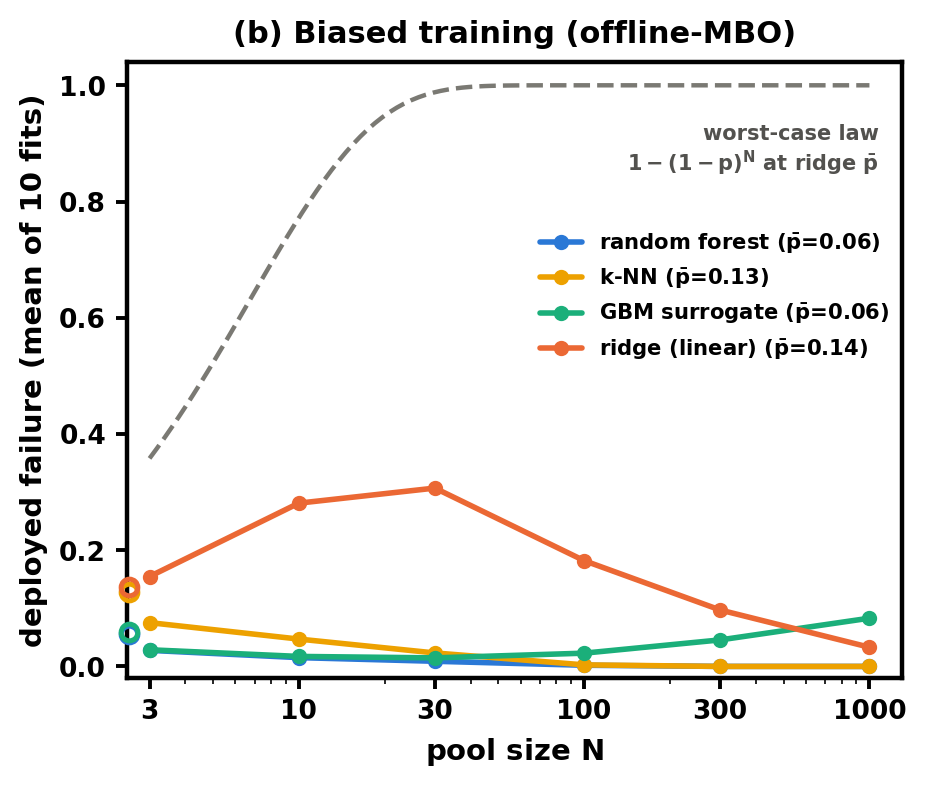}\\[3pt]
\includegraphics[width=0.49\textwidth]{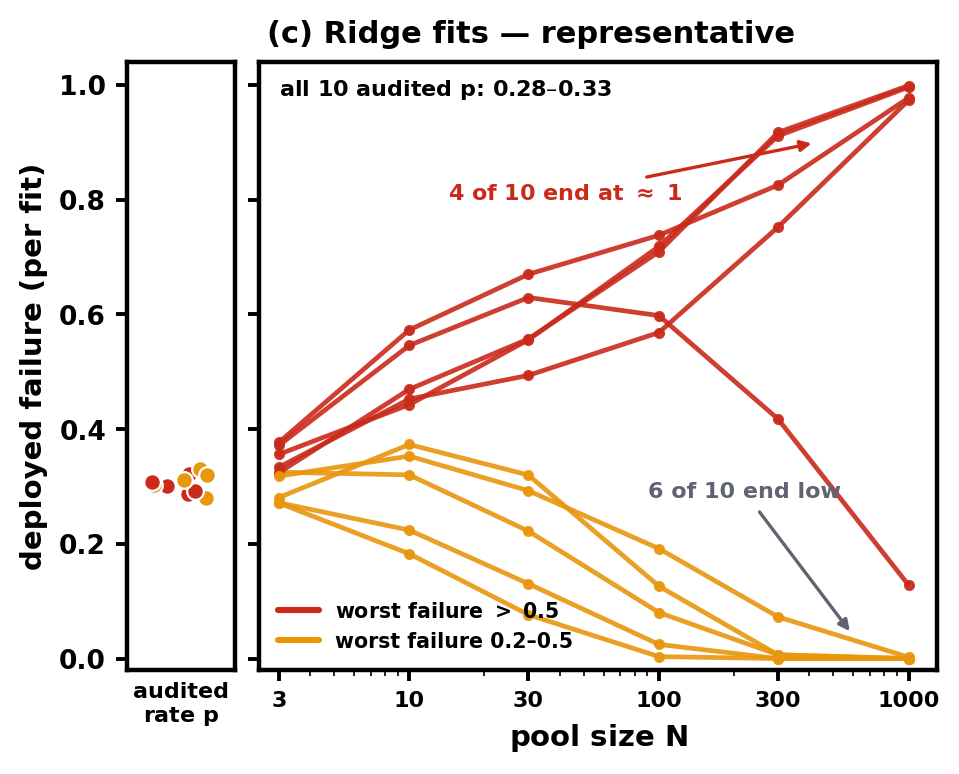}\hfill
\includegraphics[width=0.49\textwidth]{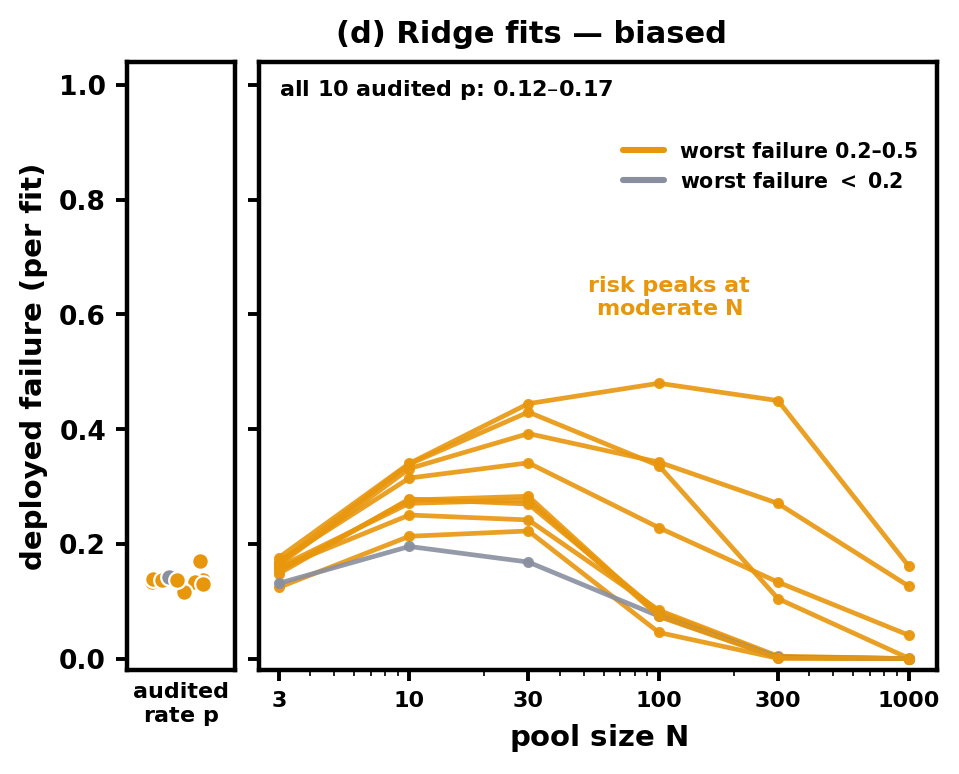}
\caption{The selection tax measured on Task 1 (hyperparameter optimization;
oracle = actual model training; ten fits per condition, train size 120).
(a,b) Mean deployed failure at the selected configuration versus screen
width $N$, under representative and below-median (offline
model-based-optimization\cite{trabucco2022design}) training data. Hollow
circles at the left edge: mean audited marginal rate $p$; dashed curve: the
worst-case law $1-(1-p)^N$ at the ridge surrogate's audited $p$, never
violated. (c,d) The same data disaggregated for the ridge surrogate: each
line is one fit, colored by the worst deployed failure it ever reaches
(gray: stays below $0.2$; amber: peaks between $0.2$ and $0.5$; red:
exceeds $0.5$). The narrow strip to the left of each panel plots every
fit's audited marginal rate $p$ on the \emph{same} vertical scale, colored
by the same outcome: under representative data (c) all ten audited rates
sit in one tight band ($0.28$--$0.33$) while four fits climb to deployed
failure ${\approx}\,1$, five fall below $0.01$, and one settles at
$0.13$; under biased data
(d) the audited band is $0.12$--$0.17$ while every fit's risk peaks at
moderate $N$ (up to $0.48$, a $3.5\times$ understatement) before decaying.
The audit cannot tell the fits apart; deployment can---the
audit--deployment gap of Theorem~\ref{thm:audit} realized fit by fit. A
license must therefore attach to the deployed fit, not to the recipe, and
must be selection-aware to see the difference.}
\label{fig:tax}
\end{figure}

\subsection{Audit Design Determines What an Audit Can
See}\label{sec:res-audit}

The two-sided failure of passive auditing predicted by
Theorem~\ref{thm:audit} appears cleanly in practice. On Task 1, certifying
a deployed failure target of $\rho=1\%$ at $N=100$ by passive means
requires certifying $p\le10^{-4}$---impossible for a conservative surrogate
whose true marginal rate was $p=0.048$ even though its deployment was
measured \emph{perfectly safe} (deployed failure $0.000$); the passive
audit thus rejects a safe deployment. Conversely, the marginal rate does not even determine the
\emph{direction} of deployed risk: the Task-2 ridge twins of
Section~\ref{sec:res-tax} audit at $0.247$ versus $0.252$ yet deploy at
failure $0.000$ versus $1.000$, and on Task 1 the benign $k$-NN carries a
\emph{higher} mean audited rate ($0.33$) than the dangerous ridge
($0.31$). What distinguishes safe from malign lives in \emph{where} the
errors sit relative to the selection, which uniform sampling never
probes; the champion audit reads it off, and resolved every case
correctly, at two oracle calls per round. (Cost accounting: certifying $q\le\rho=1\%$ at $\delta=0.05$
needs $O(\rho^{-1}\log(1/\delta))\approx300$ audit rounds by
Theorem~\ref{thm:audit}; the 15{,}000 rounds our experiment spent is the
\emph{measurement} budget used to estimate deployed rates to three
decimals for validation, not the license cost.) The
practical rule is the one Theorems~\ref{thm:audit} and~\ref{thm:lb} prove:
an audit must sample the deployment workflow---pools, selection,
champions---rather than the data distribution, a prescription that
parallels the ballot-comparison versus ballot-polling distinction in
election auditing.\cite{stark2012rla} The rule also separates the audit
license from its nearest deployed alternative: a one-sided marginal
split-conformal floor,\cite{jin2023selection} given the same 120 oracle
calls and the same nominal $0.20$ level, is honest on fresh random
candidates but violated on up to $99\%$ of deployments \emph{at the
champion} (16 configurations, Tasks 1 and 3; audit license: at most
$15\%$ in every one)---marginal calibration prices errors where
selection does not look (SI Section~\ref{si:sec:emp-conformal}).

\subsection{Which Diagnostic Predicts Deployed
Performance}\label{sec:res-diag}

Across an expanded zoo of 432 fits covering all three tasks (six model
families $\times$ three training sizes $\times$ four seeds, per regime per
task), we correlated three candidate diagnostics with deployed regret at
$N=100$: pool $R^2$, the pairwise $\varepsilon$-margin inversion rate
$\qinv$ (the audit target implied by Theorem~\ref{thm:ordinal}), and the
champion-inversion rate $\qsel$ (the statistic licensed by
Theorem~\ref{thm:auditgap}). All uncertainty intervals are bootstrap 95\%
CIs over fits. The deployment-calibrated statistic dominates in all six
task--regime conditions, with Spearman correlations of $0.89$--$0.99$
(Task 1: $0.96$ representative, $0.99$ biased; NAS-Bench-201: $0.96$
representative, $0.97$ [$0.95$--$0.98$] biased; Task 2: $0.89$ in both
regimes). $R^2$ remains serviceable on the smooth, low-dimensional
hyperparameter tasks ($0.59$--$0.80$) but collapses on the combinatorial
benchmark in \emph{both} regimes ($0.38$ [$0.13$--$0.59$] representative;
$0.33$ [$0.08$--$0.54$] biased)---confidence intervals disjoint from
$\qsel$'s. The pairwise rate $\qinv$ sits between the two
($0.56$--$0.86$), consistent with its $\Theta(N^2)$-versus-$\Theta(N)$
disadvantage (Theorem~\ref{thm:auditgap}). Two qualifications temper this claim. First, at the measurement
budgets above, $\qsel$ is itself a Monte Carlo probe of the deployment
law it tracks---high correlation is calibration, not clairvoyance; the
substantive finding is that a statistic \emph{measurable as an audit}
tracks deployment while the field's accuracy diagnostics do not. Second,
at the \emph{practical} audit budget the framework actually prescribes
($T_a=120$ rounds, as in the drift study), the correlation attenuates to
$0.80$--$0.94$ across the six conditions yet still dominates $R^2$ in
every one---most sharply on the combinatorial benchmark ($0.80$
[$0.70$--$0.86$] versus $0.33$ [$0.08$--$0.54$], CIs disjoint; SI
Section~\ref{si:sec:emp-div}). The one attenuation in the
full-budget table---Task 2, where every diagnostic weakens and $\qsel$
falls to $0.89$---has a structural explanation the theory anticipates: $24.8\%$ of that pool lies
within $\varepsilon$ of the optimum (versus $8.4\%$ on Task 1 and $2.5\%$
on Task 3), so deployed regret is compressed toward zero and every
ranking-based correlation attenuates; in margin-condition terms
(Theorem~\ref{thm:margin}), heavily tied tasks are exactly the ones where
selection matters least. Where surrogates are deployed on the search
spaces and data conditions they actually face, the field's default
diagnostic ranges from serviceable to nearly uninformative---and one
cannot tell which from the diagnostic itself---while the
deployment-calibrated statistic, measurable at two oracle calls per audit
round, is consistently the closest thing to sufficient that we measured.

\subsection{Licenses Hold, and Must Be Earned Per
Fit}\label{sec:res-license}

The no-disappointment license of Theorem~\ref{thm:pessimism} held in every
deployment of a certified-conservative surrogate: on Task 1, two
160-label quantile fits audited at one-sided rates $0.075$ and $0.130$
violated their floors on at most $0.027$ and $0.081$ of deployments at
$N=3$, below each fit's own audited rate at every $N$ and decaying to
$0.000$ by $N=1000$ (Figure~\ref{si:fig:license}). The instructive failure is the license that was \emph{not}
earned: on Task 2, one of two fits of the same quantile-$0.2$
recipe---identical architecture, hyperparameters, and training
protocol---audited at a one-sided rate of $p=0.116$ and violated its
would-be guarantee on $30\%$ of deployments at $N=100$, while its sibling
fit ($p=0.090$) deployed cleanly. Conservatism is a training heuristic;
the license is a measurement, and it attaches to the fit, not the recipe.

Licenses also survive contact with real drift---by expiring and being
cheaply re-earned, exactly as the theory prescribes. We audited
CIFAR-10-trained surrogates on NAS-Bench-201 under the quantile license of
Theorem~\ref{thm:dichotomy}b ($M=100$, $\rho=\delta=0.05$: $T_a=120$
oracle calls per license), then deployed the same fits against two genuinely drifted
objectives, the CIFAR-100 and ImageNet16-120 validation accuracies of the
same architectures. Every one of 12 home licenses held (violation rates
$0.03$--$0.10$ against the $4\rho=0.20$ ceiling), and every one of 24
licenses \emph{renewed} against the drifted truth at the same $120$-call
cost also held ($0.06$--$0.13$)---renewal costs no more than the original
audit and amortizes over the campaign it re-licenses. Shifting the candidate generator instead of the objective tells the
same story at the distribution margin: with drift of exactly known
total variation toward weak architectures, the worst-case $M\beta$
transfer bound is already vacuous at $\beta=0.03$ while measured
stale-license violations first breach the ceiling only at
$\beta\approx0.2$---and renewal under the drifted generator restored
validity in 48 of 48 cases (SI Section~\ref{si:sec:emp-ddrift}).
Licenses void on drift not because they fail gracefully but because
$\beta$ is unobservable; renewal is the measurement that replaces the
unmeasurable bound. The objective drift
itself delivers one more verdict against global diagnostics: the
surrogate's global rank correlation with the drifted objectives,
$0.61$--$0.69$, is indistinguishable from the $0.59$--$0.70$ it scores
against the audited objective itself---the global diagnostic does not
even register the change of truth---yet its champions remain
${\sim}87$th-percentile architectures under those objectives, and the
champion-inversion statistic $\qsel$, re-measured at two oracle calls
per round, correctly reports the preserved top-of-ranking quality
(${\approx}0.08$ under all three truths, SI Section~\ref{si:sec:emp-drift}). A
global correlation cannot see the top of the ranking, which is the only
thing selection deploys; the deployment-calibrated audit prices exactly
that, in either direction.

\subsection{The Measured Economics of Audited Surrogate
Use}\label{sec:res-econ}

Both halves of the dichotomy (Theorem~\ref{thm:dichotomy}) are visible in
the data. In the verified search loop (Figure~\ref{fig:loop}),
surrogate-guided proposals reached the target recommendation quality in
94 oracle evaluations against 204 for random search---a $2.2\times$
saving---while the deliberately biased surrogate cost efficiency
throughout most of the budget (its mean recommendation regret stays
above random's until ${\approx}320$ of the 340 calls) but never the
validity of the oracle-grounded recommendation, exactly the damage
containment Theorem~\ref{thm:invariance} promises. (Running the full
certified-elimination protocol on the same task prices the difference
between recommendations and certificates: $\varepsilon$-certificates at
$5\%$ of the range cost a mean of $480{,}555$ oracle calls---correct in
10 of 10 runs---against 340 calls for the uncertified recommendation
loop; SI Section~\ref{si:sec:emp-certcost}.) Removing verification
inverted the picture: the trusted biased surrogate claimed an objective of
$0.43$ and delivered $0.34$---a claim--delivery gap of $22\%$ of the
objective range, on a selection whose true regret was $38\%$ of that
range---at the moment the verified pipelines were delivering certified
values at par. In the screening experiments, licensed screening certified
top-1\% designs on NAS-Bench-201 for $25\times$ fewer oracle evaluations
than random-draw oracle-only search ($33\times$ on Task 1), with
champions averaging $87.1\%$ test accuracy against a global optimum of
$89.2\%$ at one verification each---the conditional separation of
Theorem~\ref{thm:dichotomy}b realized at $M=100$, with the audit itself
certifying the condition. Two accounting caveats belong beside these
factors. First, they amortize the audit ($T_a=120$ calls) and the 200
training labels to zero; all-in, a Task-3 campaign certifying $T$
designs pays $1/0.253+320/T$ oracle calls per certified-good design---%
$6\times$ cheaper than the baseline at $T=25$, $14\times$ at $T=100$,
$23\times$ at $T=1000$, approaching $25\times$ asymptotically---with the
crossover already at $T\approx4$ designs. Second, the oracle-only
baseline is the $\Omega(1/p_v)$ random-draw bound, minimax in the
i.i.d.-marks model but blind to exploitable structure in $f$;
structure-exploiting oracle-only searches (Bayesian optimization,
successive halving) could close part of the gap, so the factors quantify
licensed screening against unassisted draws, not against the best
conceivable oracle-only practice.

\begin{figure}[H]
\centering
\includegraphics[width=0.55\textwidth]{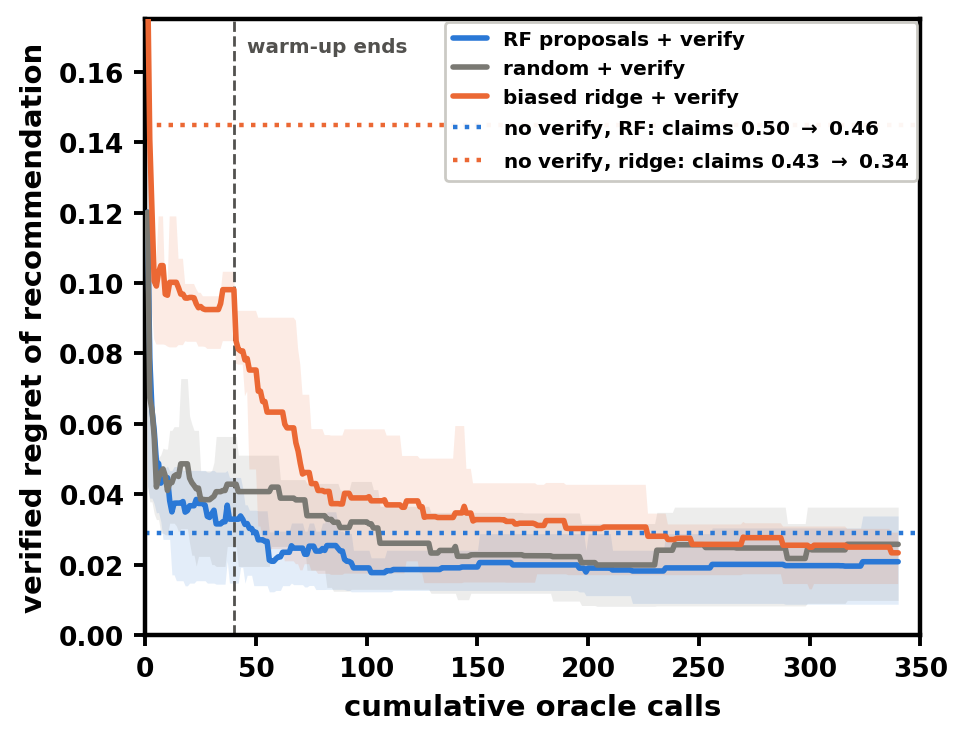}
\caption{Oracle-grounded recommendation quality versus cumulative oracle
cost (Task 1; 20 replicates; band, interquartile range). Surrogate
proposals halve the oracle cost of reaching target recommendation
quality; a biased surrogate's proposals are initially worse than random,
yet the validity of the oracle-grounded recommendation is untouched
(Theorem~\ref{thm:invariance}). Trusting either surrogate without
verification (dotted) delivers less than claimed. (The certified
protocol proper is priced in SI Section~\ref{si:sec:emp-certcost}.)}
\label{fig:loop}
\end{figure}

\subsection{Budgeting Labels between Learning and
Certification}\label{sec:res-budget}

The preceding subsections price certified search \emph{given} a
surrogate. This one opens the ledger one step earlier. Every oracle call
in a campaign buys one of two different goods: \emph{training labels}
buy a better selector (they lower the deployed inversion rate $\qsel$ the
audit will measure), while \emph{audit rounds} buy certainty about the
selector (they improve nothing, but without them nothing is certified).
Writing $m$ for training labels and $T_a$ for audit rounds at two calls
each, the total cost of a license at deployed failure $\rho$ decomposes
as
\begin{equation}\label{eq:frontier}
B(\rho)\;=\;m^{*}(\pi,\rho)\;+\;2\,T_{\min}(N,\rho,\delta),
\qquad
T_{\min}=\min\{\,T:\;N\,\bar q_{\delta}(0,T)\le\rho\,\}
\;\approx\;\tfrac{N}{\rho}\ln\tfrac1\delta,
\end{equation}
where $\bar q_{\delta}(k,T)$ is the one-sided binomial
(Clopper--Pearson) upper confidence limit and $m^{*}(\pi,\rho)$ is the
number of labels the labeling policy $\pi$ needs before the certificate
becomes earnable. The second term is a floor: it is identical for every
policy and every model, cannot be reduced by any audit design
(Theorem~\ref{thm:lb}), and spending beyond it buys nothing (SI
Section~\ref{si:sec:emp-budget}, Proposition~\ref{si:prop:floor}). All
strategy therefore lives in the first term, and the empirical question
is what that term looks like.

We measured the full plane on ground-truthed synthetic landscapes of
three difficulties (a fat-plateau optimum with deceptive
mid-value decoys; 12{,}000 candidates; $N=100$, $\rho=0.10$,
$\delta=0.05$, $\varepsilon$ at $5\%$ of range; protocols and landscape
constructions in SI Section~\ref{si:sec:emp-budget}). At every grid
point $(m,T_a)$---32 training checkpoints $\times$ 20 audit
depths---the map in Figure~\ref{fig:budgetmaps} reports the
\emph{strongest certificate earnable there},
$N\bar q_{\delta}(\text{events},T_a)$, with each license level a
contour. The map's shape classifies the campaign's regime at a glance.
On an easy landscape (two descriptors) the contours are flat: the
learning curve saturates within a few hundred labels and only the
promise costs anything---the campaign is \emph{audit-limited}
everywhere. Lifting the same landscape to eight dimensions (six
nuisance descriptors) bends each contour into an L: a vertical training
wall meets the horizontal audit floor, the corner is that license's
cheapest budget split ($600$ labels $+\,6{,}742$ audit calls
$=7{,}342$ total at $\rho=0.10$ under uniform labeling), and both
costs bind. And on a landscape with near-tied peaks---the crowded-race
geometry of Section~\ref{sec:cost}---the target contours never enter
the plane at any budget: the field saturates at $\approx0.73$, the
value the champion-audit statistic returns there even for a surrogate
scoring at the truth, so no amount of either labeling or auditing moves
it. Remark~\ref{rem:conservative} governs how that saturation may be
read. It is a joint property of the statistic and the landscape and
\emph{not} a demonstration that the guarantee is unattainable on this
population; a perfect selector has zero regret here as everywhere, and
restricting the count to draws the surrogate would have passed over
sends the saturation level toward zero at no additional oracle cost.
What the panel does establish is operational: at these settings, with
this statistic, no budget split earns the
license, so a campaign in this regime should reach for a tighter
estimator where ground truth allows one, and otherwise respecify
($\varepsilon$, $N$, or $\rho$) rather than spend.

\begin{figure}[H]
\centering
\includegraphics[width=0.325\textwidth]{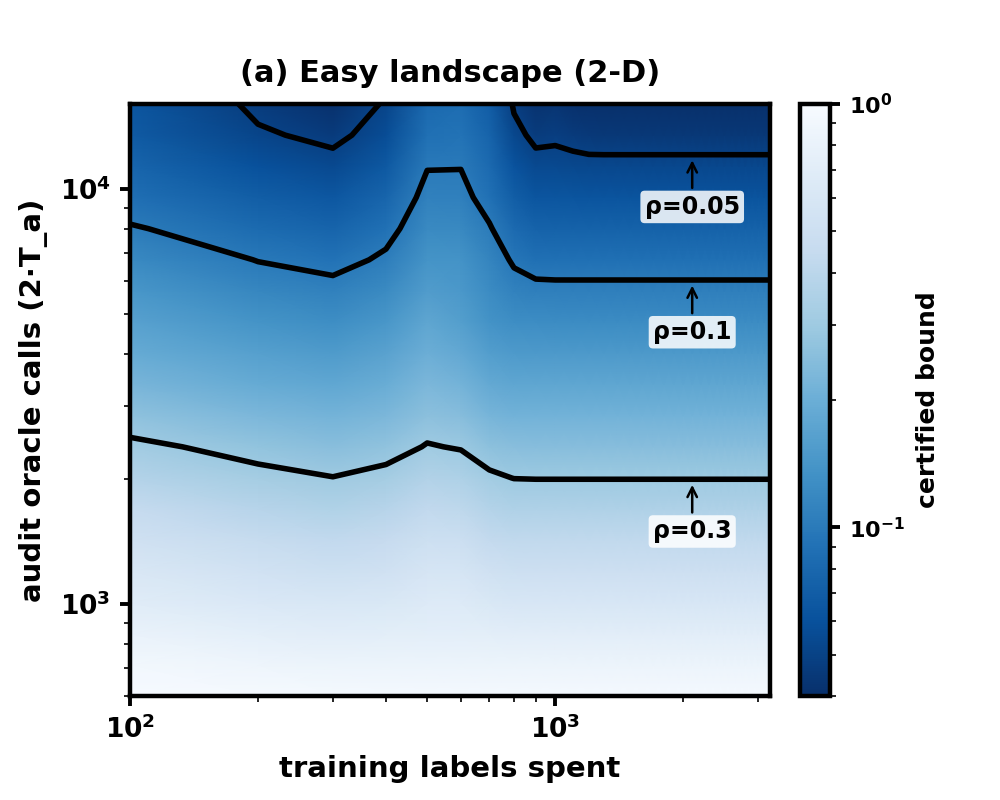}\hfill
\includegraphics[width=0.325\textwidth]{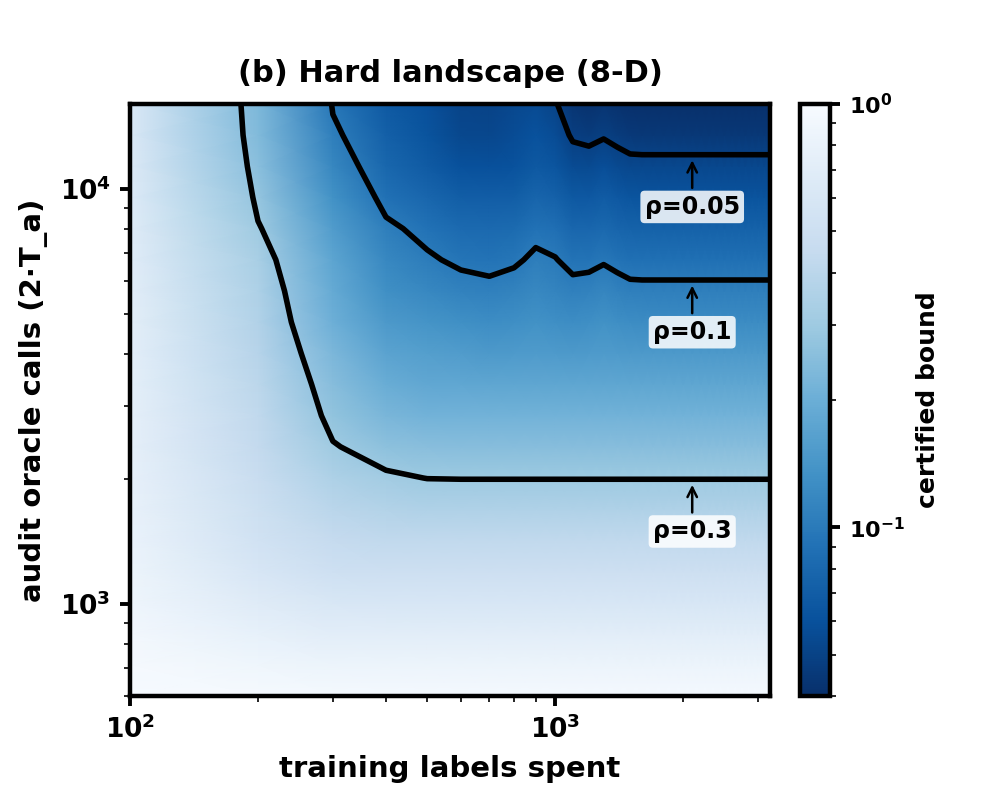}\hfill
\includegraphics[width=0.325\textwidth]{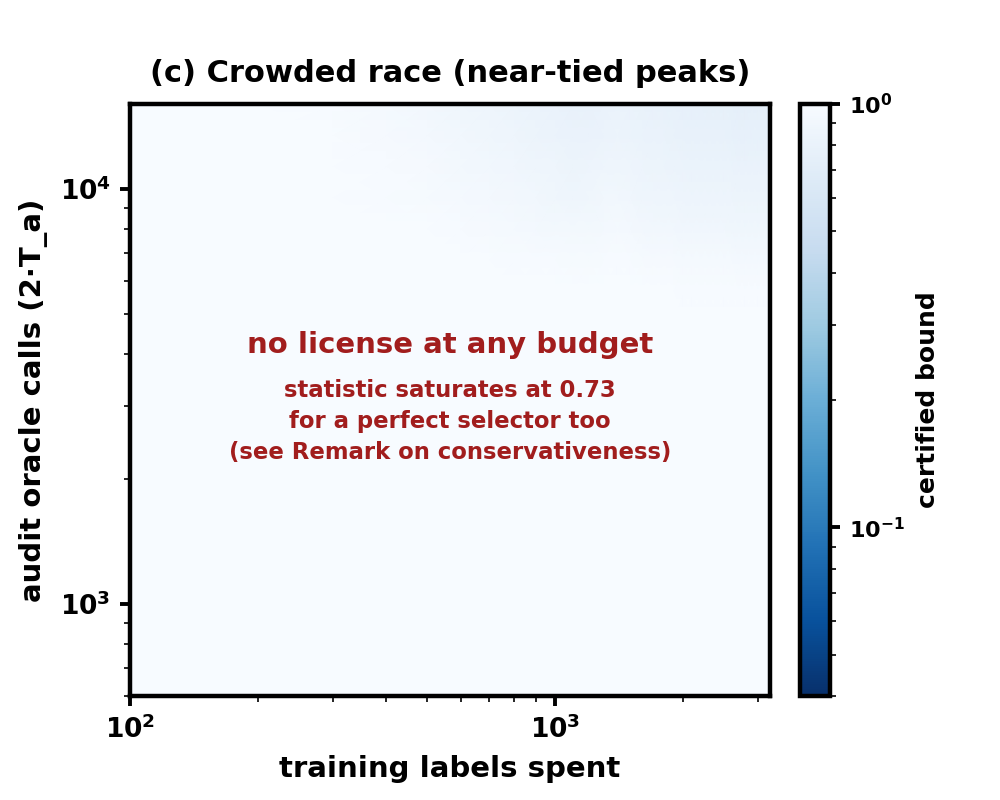}
\caption{The budget plane: strongest earnable certificate over
(training labels, audit calls), license levels as contours; one color
scale. (a) Easy landscape: flat contours---audit-limited everywhere;
only the promise costs. (b) Hard landscape (six nuisance descriptors):
L-shaped contours---learning wall meets audit floor; the corner is the
cheapest license. (c) Crowded race (near-tied peaks): no contour at any budget---the
certificate saturates above every license level, so no budget split
earns one. Panel (c) reports the champion-audit statistic of
Section~\ref{sec:protocols}, which does not vanish for a perfect selector
(Remark~\ref{rem:conservative}); the saturation is therefore a
property of that statistic together with the landscape and is not
evidence that the guarantee is unattainable on this population.}
\label{fig:budgetmaps}
\end{figure}

Where both walls bind, the labeling policy moves the vertical one---and
it moves it in a direction that inverts optimization intuition
(Figure~\ref{fig:policyrace}a). Racing four policies at equal budgets,
the pure exploitation rule---label where the model predicts the
top, the batch analogue of standard Bayesian-optimization
acquisition---was the \emph{most expensive} route to a license on every
landscape tested: $3\times$ the labels of uniform sampling at the
$\rho=0.10$ corner (1{,}800 vs.\ 600), with the same audit depth, and on
the two-descriptor multimodal landscape its certified bound moved
\emph{backward} twice as labels accumulated (decoy lock-in: runs that
latched onto a deceptive region densified it while the true plateau
stayed unmodeled). Uniform and uncertainty sampling certified cheapest
throughout. The mechanism is the paper's central asymmetry one level
up: the license prices the global ordering at the top of the
landscape---including the \emph{suppression of every rival
region}---and exploitation never purchases refutation. Acquisition
rules built to find the maximum are not built to certify the selector.
The three licenses also part company here
(Figure~\ref{fig:policyrace}b): the honesty and value-floor licenses
are earned within ${\sim}250$ audit calls at every training level,
while the selection license cannot be earned below its $5{,}990$-call
audit floor and is unearnable below $400$ labels at \emph{any} audit
depth---a ${\sim}24\times$
spread that makes the selection tax of Section~\ref{sec:cost} legible
directly in budget coordinates, and shows that only the selection
promise couples to the labeling policy at all.

\begin{figure}[H]
\centering
\includegraphics[width=0.49\textwidth]{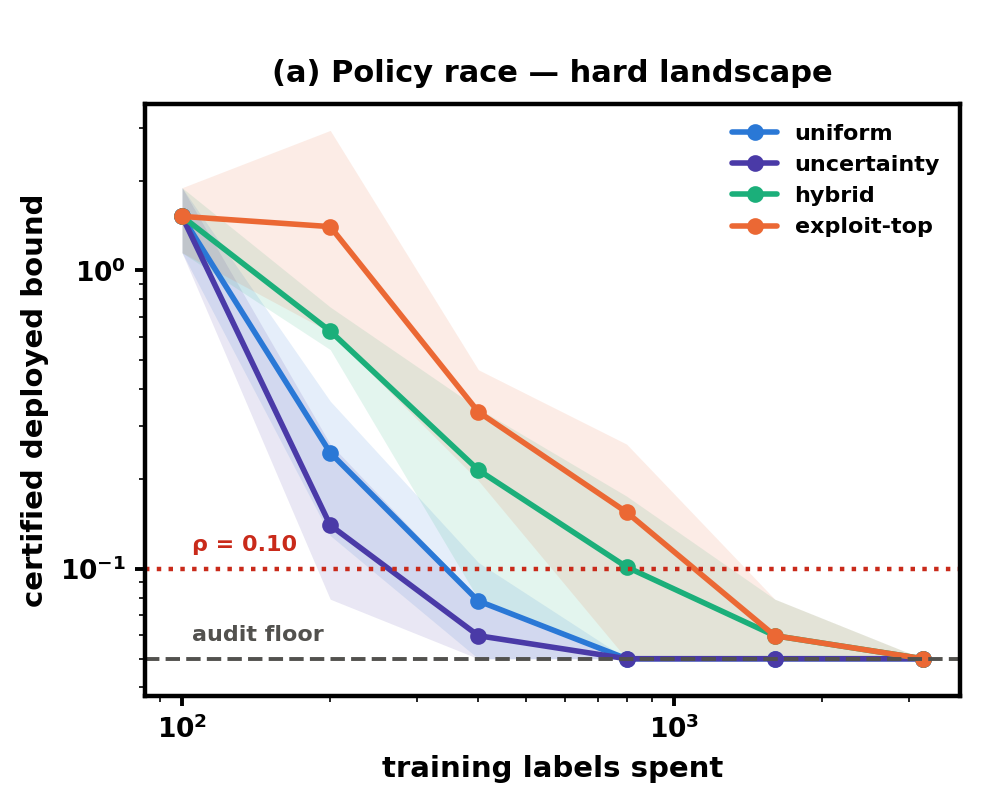}\hfill
\includegraphics[width=0.49\textwidth]{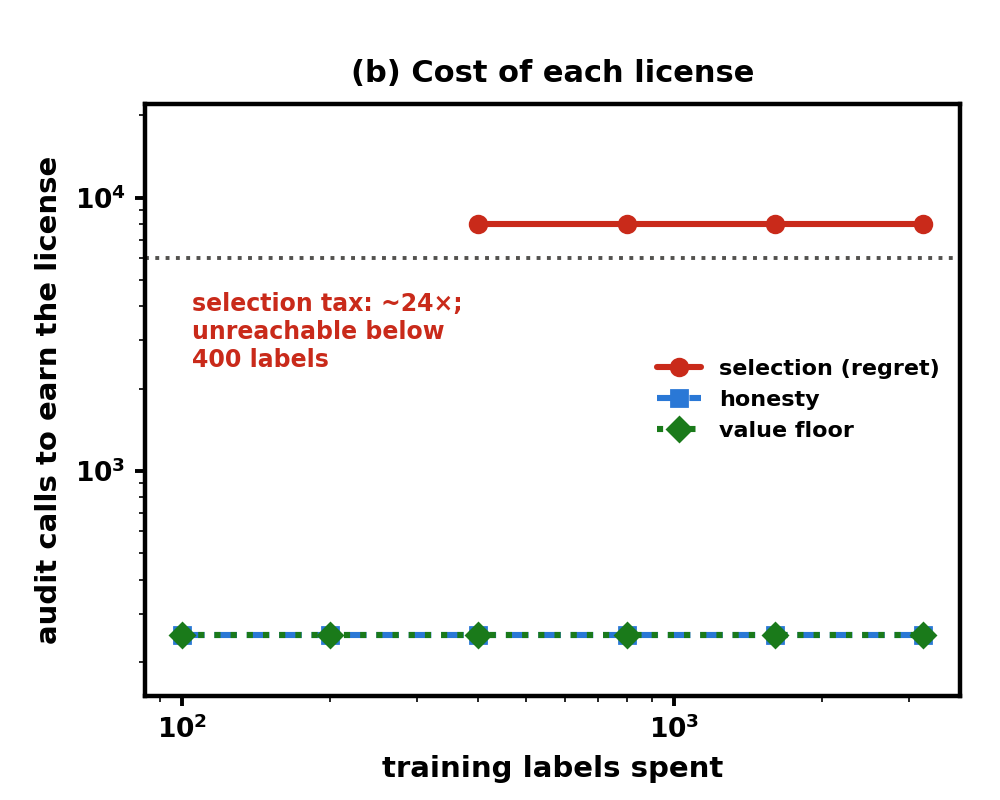}
\caption{(a) Certified deployed bound versus training labels for four
labeling policies (hard landscape; three seeds, min--max bands).
Exploitation-style acquisition is the most expensive license-earner;
uniform and uncertainty sampling reach the audit-resolution floor
first. (b) Audit calls required to earn each license versus training
labels (uniform policy): honesty and value-floor licenses cost
${\sim}250$ calls; the selection license requires at least its
$5{,}990$-call audit floor and ${\ge}400$ labels at any audit depth.}
\label{fig:policyrace}
\end{figure}

Finally, we asked what model accuracy is actually worth by
manufacturing it: perfect predictions corrupted by overestimating a
fraction $f$ of candidates by $s$ (overestimation being the failure
mode selection amplifies), gridded over $(f,s)$
(Figure~\ref{fig:damage}a). The license's accuracy currency is
\emph{not} the fraction of labels that are right: sub-$\varepsilon$
overestimates were free at every rate tested (up to $20\%$ of all
candidates), while the $\rho=0.10$ boundary sits at
$s\approx1.2$--$2.3\,\varepsilon$ across two decades of $f$---a nearly
horizontal line whose message is that error \emph{size at the point of
selection}, in units of $\varepsilon$, is what a license buys and
loses. A model wrong by $\varepsilon$ on $5\%$ of its candidates is
licensed as-is (certified bound $0.051$). When a model does fail, the
repair bill depends almost entirely on \emph{where} one repairs
(Figure~\ref{fig:damage}b): correcting predictions in order of
predicted value---verifying one's champions and writing down the
truth---restored the license after $50$--$530$ corrections, while
random correction failed to restore any tested case within $6{,}400$;
a ${\ge}12$--$130\times$ asymmetry. Champion verification, the
framework's audit primitive, reappears as its cheapest repair policy.

\begin{figure}[H]
\centering
\includegraphics[width=0.49\textwidth]{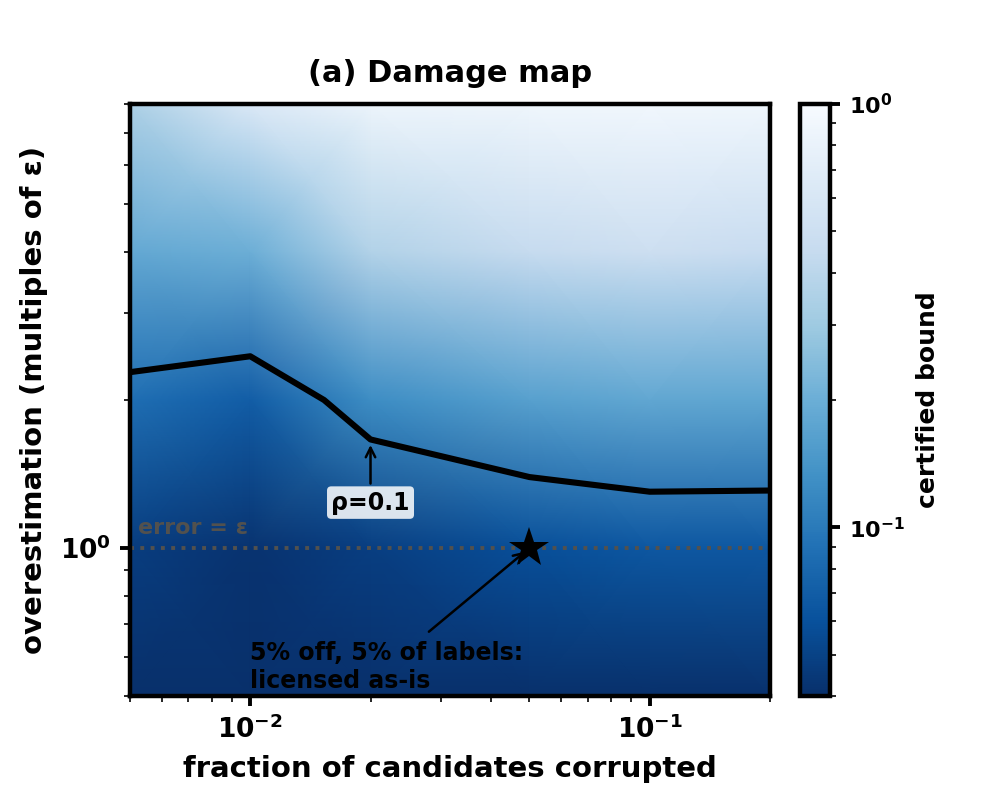}\hfill
\includegraphics[width=0.49\textwidth]{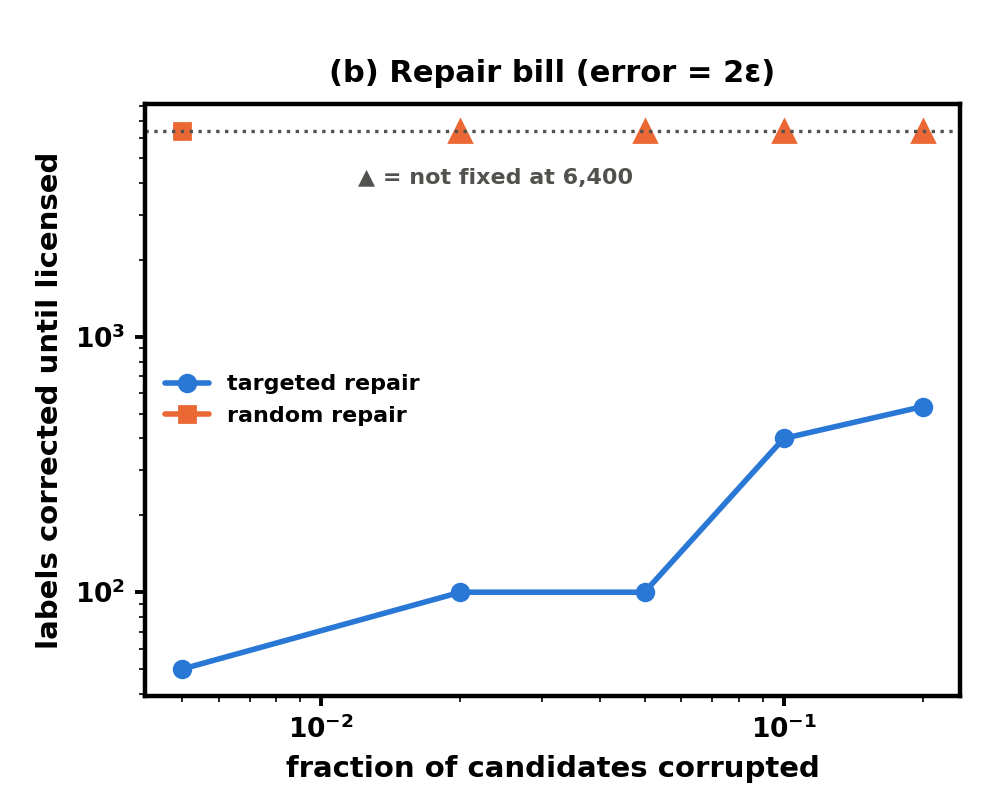}
\caption{(a) Damage map: strongest certificate under controlled
corruption (fraction of candidates overestimated $\times$ error size in
units of $\varepsilon$). Sub-$\varepsilon$ errors are free at any rate;
the license boundary tracks error size, not accuracy percentage; the
starred case ($\varepsilon$-sized errors on $5\%$ of candidates) is
licensed as-is. (b) Repair bill: labels corrected until the license is
restored, fixing highest-predicted candidates first versus at random
(error $=2\varepsilon$).}
\label{fig:damage}
\end{figure}

Three budget rules follow, each attached to a measurement above. The
audit-to-training ratio is a diagnostic: audit-heavy means the
\emph{claim} is the bottleneck (weaken $N$ or $\rho$, or fall back to
the ${\sim}250$-call one-sided licenses), training-heavy means the
problem is. The floor in Eq.~\eqref{eq:frontier} cannot be negotiated,
only the promise can. And the practice landscape collapses onto a
single dial---\emph{how much verification does this model still
owe?}---with measure-everything proposal loops (Bayesian optimization
among them) at one end, paying full price because they never deploy
trust; licensed screening at one verification per batch at the other;
and a refused license sending the model back to propose-and-verify.
The dichotomy of Theorem~\ref{thm:dichotomy} says when crossing that
dial pays; the budget plane above says what the crossing costs.

\subsection{Practical Guidance}\label{sec:res-guidance}

The results compose into an operating doctrine whose every clause is a
theorem with an experiment attached: \emph{train on anything; certify only
with clean oracle statistics; audit where the search actually points;
re-audit when anything moves.} Table~\ref{tbl:summary} maps each
theoretical prediction to its measurement. For practitioners the doctrine
concretizes as: report only oracle-verified values; treat $R^2$ as a
development diagnostic, never a deployment one; before trusting a surrogate
as a selector, run the champion audit at two oracle calls per round and
size the license to the intended screen width $N$; prefer one-sided
(conservative) licenses when the reported value matters more than the
discovery rate; and void any license upon retraining or a shift in the
candidate generator. None of these steps requires assumptions about the
surrogate's architecture or training---which is what makes them
enforceable as review criteria for third-party models.

\begin{table}[H]
\caption{Theoretical predictions and their measurements across the three
ground-truthed tasks.}
\label{tbl:summary}
\small
\begin{tabular}{p{5.7cm}p{6.2cm}p{3.3cm}}
\toprule
Prediction (theorem) & Measurement & Task \\
\midrule
Deployed failure $\le 1-(1-p)^N$, tight only under error--selection
alignment (Thm~\ref{thm:audit}) & no violations in any cell; alignment is
per-fit: twin ridge fits at $p\approx0.25$ deploy at $q=0.000$ and
$q=1.000$; NB201 ridge deterministically at $q=1.0$ in both regimes
& 1, 2, 3 \\
Passive audits err in both directions; champion audits correct
(Thms~\ref{thm:audit},~\ref{thm:lb}) & safe deployment rejected
($p=0.048$, $q=0.000$); malign twin indistinguishable passively;
selection-aware audit correct in every case &
1, 2 \\
Accuracy decoupled from deployment (Thms~\ref{thm:r2}--\ref{thm:evt}) &
$R^2$ flat while deployed failure spans $0\to1.0$ & 1, 2, 3 \\
$\qsel$ is the deployment-calibrated diagnostic (Thm~\ref{thm:auditgap}) &
Spearman with regret $0.89$--$0.99$ in all six task--regime cells
(432 fits); NB201 biased: $0.97$ vs.\ $0.56$ ($\qinv$) vs.\ $0.33$ ($R^2$);
$\qsel$'s CI disjoint from both others & 1, 2, 3 \\
No-disappointment license, per-fit (Thm~\ref{thm:pessimism}) & holds at all
$N$ when earned; same recipe passes one fit ($p=0.090$), fails the other
($p=0.116\to30\%$ violations) & 1, 2 \\
Validity invariant to proposal quality (Thm~\ref{thm:invariance}) &
$2.2\times$ oracle saving from good proposals; biased proposals cost
efficiency only; unverified trust claims $0.43$, delivers $0.34$ & 1 \\
Conditional screening separation (Thm~\ref{thm:dichotomy}) & $25\times$
(Task 3) and $33\times$ (Task 1) certified-cost reduction at $M=100$ &
1, 3 \\
\bottomrule
\end{tabular}
\end{table}

% ====================================================================
\section{Conclusions}\label{sec:conclusions}

We have given a quantitative account of trust in machine-learning
surrogates for design: accuracy cannot anchor it
(Theorems~\ref{thm:r2}--\ref{thm:evt}); an information-flow architecture
secures it unconditionally, and treating predictions as measurements
admits a deterministic self-confirmation failure
(Theorems~\ref{thm:invariance}--\ref{thm:contamination}); its price is set
by close races (Theorems~\ref{thm:se}--\ref{thm:margin}); the criterion
for a model to act as an oracle is ordinal, verifiable only wholesale, and
only by audits that imitate deployment---a design that is provably
optimal among all audits
(Theorems~\ref{thm:ordinal}--\ref{thm:pessimism},~\ref{thm:lb}); and the
substitution pays, by a factor as large as the screening ratio, exactly
when the goal is generative and the audit passes---and, in the tabular
model, never for fixed menus
(Theorems~\ref{thm:dichotomy}--\ref{thm:auditgap}). The framework's deployment-facing predictions were confirmed on
17{,}353 real training outcomes across three tasks, including a community benchmark on which the framework's audit
statistic tracks deployed performance at rank correlation $\ge0.95$
while the field's default diagnostic falls as low as $0.33$; across all
432 fits and six task--regime conditions the audit statistic never fell
below $0.89$ at measurement budgets, nor below $0.80$ at the prescribed
120-round audit budget. The closing budget study prices what remains
free: the region of the (training, auditing) budget plane in which a
certificate is earnable is L-shaped, its two walls being the learning
curve and the irreducible audit floor; acquisition-style labeling is
the most expensive route to a license; and model error is priced at
the point of selection, in units of $\varepsilon$, not in accuracy
points---with champion verification doubling as the cheapest repair
policy when a license is lost.

Two limitations bound the claims. The guarantees cover single-decision
(pool and generative) searches, a boundary that now sits exactly at
credit assignment: the policy-\emph{selection} form of the sequential
problem---certifying the best of $K$ candidate policies from real
episode returns while world models and critics propose freely---follows
from the pool theorems (SI Section~\ref{si:sec:seq}, including a formal
replay-buffer contamination result), whereas certifying a policy
\emph{improvement} step, where a trusted error propagates through later
decisions, remains open. And the strict-improvement license certifies
values rather than regret; regret-type separations in structured model
classes remain open. Both are treated in the SI. The framework itself is domain-agnostic and
assumption-light, and its audits are cheap by construction; applying it as
a standardized review instrument to published surrogates in fields where
oracles are costly and training corpora are biased---computational
screening for carbon capture and catalysis prominent among them---is the
natural next step, and the subject of ongoing work.

\section*{Acknowledgements}
The authors gratefully acknowledge the Ohio Supercomputer Center (OSC)
and the Department of Materials Science and Engineering at the
Massachusetts Institute of Technology for computational resources.

\section*{Code and Data Availability}
The audit protocols and license certificates described here are
implemented in \textsc{ModelAssay}, open-source software available at
\url{https://github.com/MAXMA-OSU/Modelassay}; that repository is the
reference implementation. The experiment code, exhaustive ground-truth
tables, and verification scripts of Section~\ref{si:sec:verify} used to
produce the numbers in this paper are not in the public repository at
the time of writing; they will accompany a subsequent
\textsc{ModelAssay} release.

\section*{Supporting Information Available}
Complete formal statements and proofs of all theorems, with the
main-text correspondence noted at each result (Sections S1--S7);
extended experimental details, additional figures and per-task data
tables (Section~S8); the positioning of each result against prior
literature (Section~S9); limitations and open problems, with a closing summary (Sections
S10 and S11); and the numerical verification suite and reproducibility
statement (Section~S12).

% ==================== SUPPORTING INFORMATION ====================
\clearpage
\setcounter{section}{0}
\setcounter{figure}{0}
\setcounter{table}{0}
\setcounter{theorem}{0}
\renewcommand{\thesection}{S\arabic{section}}
\renewcommand{\thefigure}{S\arabic{figure}}
\renewcommand{\thetable}{S\arabic{table}}
\counterwithin{theorem}{section}
\begin{center}
{\Large\bfseries Supporting Information}\\[4pt]
{\large When May a Model Replace the Experiment? Audits, Licenses,\\
and the Price of Trust in Surrogate-Driven Design}\\[6pt]
\end{center}
\noindent This Supporting Information contains the complete formal setting,
all theorem statements with full proofs (each main-text theorem cites its
S-numbered counterpart below), extended experimental details and
figures, theorem-by-theorem positioning against prior work, and the
numerical verification suite. The reference implementation of the audit and
license protocols is \textsc{ModelAssay}; the experiment tables and
verification scripts will accompany a subsequent \textsc{ModelAssay}
release.
\renewcommand{\theequation}{S\arabic{equation}}
\setcounter{equation}{0}

\section{Setting and Notation}\label{si:sec:setting}

$\X$ is a finite candidate pool, $|\X|=N\ge 2$; the true objective is
$f:\X\to[0,B]$; a surrogate is any $\fh:\X\to\R$, of arbitrary provenance.
Ties in argmaxes are broken by a fixed order unless stated otherwise. When
the oracle is noisy, a query of $x$ returns $f(x)+\eta$ with $\eta$
mean-zero $\sigma$-sub-Gaussian, independent across queries; formally we
use the pre-sampled array (``stack of rewards'') model: an array
$(\eta_{x,k})_{x\in\X,k\ge1}$ of independent noise variables, independent
of algorithmic randomness, with the $k$-th query of $x$ returning
$f(x)+\eta_{x,k}$.\cite{si-lattimore2020bandit} Table~\ref{si:tbl:notation}
collects the recurring notation.

\begin{table}[H]
\caption{Notation used throughout.}
\label{si:tbl:notation}
\centering\small
\begin{tabular}{ll}
\toprule
Symbol & Meaning \\
\midrule
$\X$, $N$ & candidate pool; its size (or the deployment pool size, per context) \\
$f$, $B$ & true objective (oracle), $f:\X\to[0,B]$ \\
$\fh$ & surrogate, $\fh:\X\to\R$ \\
$x^\star$, $\xh$ & true maximizer; surrogate champion $\arg\max\fh$ \\
$r(\xh)$ & simple regret $f(x^\star)-f(\xh)$ \\
$\eplus(x)$, $\eminus(x)$ & over-/under-prediction $[\fh(x)-f(x)]_+$, $[f(x)-\fh(x)]_+$ \\
$\Delta_x$ & gap to the best, $f(x^\star)-f(x)$ \\
$\sigma$, $\delta$ & noise scale; confidence parameter \\
$\varepsilon$, $\tau$ & decision tolerance; over-prediction tolerance \\
$D$, $M$ & candidate distribution; screening pool size (generative setting) \\
$p$, $q$ & audited marginal bad-rate; deployed (champion) bad-rate \\
$\qinv$, $\qsel$ & pairwise $\varepsilon$-inversion rate; champion-inversion rate \\
$p_v$, $\rho$, $\rho_0$ & good-mass at level $v$; license rates \\
\bottomrule
\end{tabular}
\end{table}

\section{Preliminaries (Known Results, Restated with Attribution)}\label{si:sec:prelim}
The results in this section are \emph{not} contributions; each carries its canonical citation and is used later.

\begin{proposition}[regret decomposition; horizon-1 case of Lemma~3.1 of Jin et al.;\cite{si-jin2021pessimism} expectation version in Smith and Winkler\cite{smith2006optimizer}]\label{si:prop:decomp}
For any $f,\fh$ and maximizers $x^\star,\xh$,
\[
f(x^\star)-f(\xh)=\underbrace{[f(x^\star)-\fh(x^\star)]}_{\text{underestimation at }x^\star}+\underbrace{[\fh(x^\star)-\fh(\xh)]}_{\le 0}+\underbrace{[\fh(\xh)-f(\xh)]}_{\text{overestimation at }\xh},
\]
hence $r(\xh)\le \eminus(x^\star)+\eplus(\xh)$.
\end{proposition}
\begin{proof}
The display is an identity; the middle bracket is $\le0$ since $\xh$ maximizes $\fh$; bound the remaining terms by their positive parts.
\end{proof}

\begin{proposition}[simulation lemma\cite{si-kearns2002near,si-bertsekas1996neuro}]\label{si:prop:sim}
If two discounted MDPs share transitions and $\|\hat r-r\|_\infty\le\epsilon_r$, then $\|\hat Q^\star-Q^\star\|_\infty\le \epsilon_r/(1-\gamma)$.
\end{proposition}

\begin{proposition}[half-gap action stability; stated in Farahmand;\cite{si-farahmand2011action} policy-loss form in Singh and Yee\cite{si-singh1994upper}]\label{si:prop:gap}
If $\|\hat Q-Q^\star\|_\infty<\Delta(s)/2$, the greedy action under $\hat Q$ at $s$ is optimal.
\end{proposition}

\begin{proposition}[interval certificate; LUCB stopping rule with $\epsilon=0$;\cite{si-kalyanakrishnan2012pac} UGapE certificate;\cite{si-gabillon2012best} contrapositive of action elimination;\cite{si-evendar2006action} cf.\ Hoeffding races\cite{si-maron1993hoeffding}]\label{si:prop:cert}
If $f(x)\in[L(x),U(x)]$ for all $x\in\X$ simultaneously and $L(\tilde x)>\max_{x\ne\tilde x}U(x)$, then $\tilde x=x^\star$ uniquely.
\end{proposition}
\begin{proof}
For $x\ne\tilde x$: $f(x)\le U(x)<L(\tilde x)\le f(\tilde x)$.
\end{proof}

\section{Prediction Accuracy versus Decision Quality}\label{si:sec:sep}
The phenomenon that accuracy does not control decisions is documented empirically and in adjacent theory;\cite{si-cameron2022perils,si-elmachtoub2022smart,trabucco2021conservative,si-lyu2026learnability} the following statements pin it down sharply.

\begin{theorem}[{$R^2$} neither sufficient nor necessary]\label{si:thm:r2}
Let $R^2(f,\fh)=1-\tfrac{\sum_x(f(x)-\fh(x))^2}{\sum_x(f(x)-\bar f)^2}$
over the pool.
\begin{enumerate}[label=(\roman*),leftmargin=2em,itemsep=1pt]
\item For every $\varepsilon\in(0,1)$, $\beta>0$, and every even
$N\ge 4(1+\beta)^2/\varepsilon$, there is a pool of size $N$ and a pair
$(f,\fh)$ with $R^2\ge1-\varepsilon$ and $r(\xh)=B$, the full range of $f$.
\item For every $M>0$ there is a pair on a pool of size $2$ with
$r(\xh)=0$ and $R^2\le -M$.
\end{enumerate}
\end{theorem}
\begin{proof}
\emph{(i) Near-perfect $R^2$, worst-possible champion.}
Let $f=1$ on half the pool and $0$ on the other half, so $\bar f=\tfrac12$,
$\sum_x(f-\bar f)^2=N/4$, and $B=1$. Set $\fh=f$ everywhere except at one
zero-valued candidate $x_0$, where $\fh(x_0)=1+\beta$. The only error is at
$x_0$, so
\begin{equation}\label{si:eq:r2i}
R^2 \;=\; 1-\frac{(1+\beta)^2}{N/4} \;=\; 1-\frac{4(1+\beta)^2}{N}
\;\ge\; 1-\varepsilon .
\end{equation}
Yet $\fh(x_0)=1+\beta$ strictly exceeds every other prediction, so
$\xh=x_0$ uniquely, and $r(\xh)=1-0=B$.

\emph{(ii) Arbitrarily negative $R^2$, perfect champion.}
Take the two-point pool $f=(0,1)$ and $\fh=c\,f=(0,c)$ with $c>1$. The
ranking is preserved, so $r(\xh)=0$, while $\bar f=\tfrac12$ gives
\begin{equation}\label{si:eq:r2ii}
R^2 \;=\; 1-\frac{(c-1)^2}{1/2} \;=\; 1-2(c-1)^2 \;\xrightarrow[c\to\infty]{}\; -\infty ;
\end{equation}
$c=1+\sqrt{(M+1)/2}$ achieves $R^2\le-M$.
\end{proof}

\begin{theorem}[equal MSE, opposite outcomes]\label{si:thm:mse}
Fix $\Delta>0$, $N>4$, and $a\in(\Delta/\sqrt N,\ \Delta/2)$ (nonempty
since $N>4$). Let $f(x_1)=\Delta$ and $f\equiv0$ elsewhere. There are
surrogates $\fh_A,\fh_B$ with identical MSE $a^2$ such that $r(\xh_A)=0$
and $r(\xh_B)=\Delta$ (maximal).
\end{theorem}
\begin{proof}
\emph{Surrogate A: errors everywhere, decision perfect.} Let
$\fh_A=f+a\,s$ with arbitrary signs $s(x)\in\{\pm1\}$, so
$\mathrm{MSE}(\fh_A)=a^2$. Since $a<\Delta/2$,
\begin{equation}\label{si:eq:mseA}
\fh_A(x_1)\;\ge\;\Delta-a\;>\;a\;\ge\;\fh_A(x)\qquad(x\ne x_1),
\end{equation}
so $\xh_A=x_1$ and $r(\xh_A)=0$.

\emph{Surrogate B: one error, decision maximal.} Let $\fh_B=f$ except at a
single $x_j$, $j\ne1$, where $\fh_B(x_j)=a\sqrt N$. Then
$\mathrm{MSE}(\fh_B)=(a\sqrt N)^2/N=a^2$, identical to A; but
$a\sqrt N>\Delta$ strictly, so $\xh_B=x_j$ with $f(x_j)=0$ and
$r(\xh_B)=\Delta$. The distinguishing information sits entirely in the
upper tail of $\eplus$, which the aggregate suppresses.
\end{proof}

\begin{theorem}[extreme-value amplification; matching bounds]\label{si:thm:evt}
If the errors $\varepsilon(x)=\fh(x)-f(x)$ are mean-zero and
$\sigma$-sub-Gaussian (any dependence), then for $N\ge2$
\begin{equation}\label{si:eq:evt-upper}
\E[\eplus(\xh)]\ \le\ \E\big[(\max_x\varepsilon(x))_+\big]\ \le\
\sigma\sqrt{2\ln(N+1)}\ \le\ 2\sigma\sqrt{\ln N}.
\end{equation}
Conversely, on the flat instance $f\equiv c$ with $\varepsilon(x)$ i.i.d.\
$\mathcal N(0,\sigma^2)$, $\eplus(\xh)=[\max_x\varepsilon(x)]_+$ a.s.\ and
for all $N\ge2$
\begin{equation}\label{si:eq:evt-lower}
\E[\eplus(\xh)]\ \ge\ \Big(1-\tfrac1e\Big)\,\sigma\,\Phi^{-1}\!\Big(1-\tfrac1N\Big),
\qquad
\Phi^{-1}\!\Big(1-\tfrac1N\Big)\ \ge\ \sqrt{\ln N}\quad(N\ge43).
\end{equation}
\end{theorem}
\begin{proof}
\emph{Upper bound.} Pointwise,
$\eplus(\xh)\le[\max_x\varepsilon(x)]_+
=\max(\varepsilon(x_1),\dots,\varepsilon(x_N),0)$, a maximum of $N+1$
$\sigma$-sub-Gaussian variables, whose mean is at most
$\sigma\sqrt{2\ln(N+1)}$;\cite{si-boucheron2013concentration} and
$\ln(N+1)\le2\ln N$ for $N\ge2$, giving \eqref{si:eq:evt-upper}.

\emph{Lower bound.} On the flat instance $\xh=\arg\max_x\varepsilon(x)$.
With $t_N=\sigma\Phi^{-1}(1-1/N)\ge0$ for $N\ge2$,
\begin{equation}\label{si:eq:evt-hit}
\Prob\Big(\max_x\varepsilon(x)\ge t_N\Big)=1-\Big(1-\tfrac1N\Big)^{N}\ge1-\tfrac1e,
\qquad
[\max]_+\ \ge\ t_N\,\mathbf 1\{\max\ge t_N\},
\end{equation}
and taking expectations gives the first claim of \eqref{si:eq:evt-lower}.
For the numeric claim, the Gaussian tail bound
$1-\Phi(t)\ge \varphi(t)\,t/(1+t^2)$ at $t=\sqrt{\ln N}$ reduces it to
$\sqrt N\ge\sqrt{2\pi}\,(1+\ln N)/\sqrt{\ln N}$, which holds for all
$N\ge43$: the two sides have ratio proportional to
$\sqrt{N\ln N}/(1+\ln N)$, which is increasing in $N$ (its logarithmic
derivative in $x=\ln N$ is $(x^2+1)/(2x(1+x))>0$), so a numerical check
at $N=43$ suffices.
\end{proof}

\begin{remark}
Lineage: winner's curse,\cite{capen1971competitive} optimizer's curse,\cite{smith2006optimizer} post-selection inference.\cite{si-berk2013valid} With correlated errors, $\ln N$ is replaced by metric entropy (Sudakov/Dudley), the rigorous form of an ``effective candidate count.''
\end{remark}

\section{The Certification Framework}\label{si:sec:framework}
\textbf{Separation principle.} Labels serve two roles: \emph{proposal} (what to try next; training surrogates) and \emph{certification} (what to trust; when to stop; what to output with a warranty). Surrogate labels may serve the first role without restriction; certification statistics are computed from oracle labels only. Theorems~\ref{si:thm:invariance} and~\ref{si:thm:contamination} prove this boundary is exactly right.

At each round $t$ the learner selects $x_t\in\X$ by an arbitrary measurable rule and either queries the oracle ($G_t=O$, cost 1), records a pseudo-label $\fh_t(x_t)$ ($G_t=L$, cost 0, $\fh_t$ arbitrary), or stops with a recommendation. Let $n_t(x)$ count oracle queries of $x$ and $\hat\mu_t(x)$ their sample mean.

\begin{definition}[admissible confidence system]\label{si:def:cs}
Intervals $[L_t(x),U_t(x)]$, measurable w.r.t.\ the oracle-labeled history only, are $\delta$-admissible if
$\Prob(\forall t\ge1,\forall x\in\X: f(x)\in[L_t(x),U_t(x)])\ge1-\delta$. Call this event $\mathcal G$.
\end{definition}

\begin{lemma}[tabular instantiation; construction standard\cite{si-evendar2006action,si-jamieson2014best}]\label{si:lem:cs}
Under the noise model of Section~\ref{si:sec:setting}, define for $n_t(x)=k\ge1$
\[
r_k=\sigma\sqrt{\tfrac2k\ln\!\big(\tfrac{\pi^2Nk^2}{3\delta}\big)},\qquad
[L_t(x),U_t(x)]=\hat\mu_t(x)\pm r_{n_t(x)},
\]
and $[0,B]$ while $n_t(x)=0$. This system is $\delta$-admissible regardless of the adaptive rule deciding which candidate is queried when, and regardless of any use of pseudo-labels elsewhere.
\end{lemma}
\begin{proof}
On the pre-sampled array fix $(x,k)$; Hoeffding gives $\Pr(|\bar\eta_{x,k}|\ge r_k)\le2\exp(-kr_k^2/2\sigma^2)=6\delta/(\pi^2Nk^2)$. Union over $k\ge1$ ($\sum k^{-2}=\pi^2/6$) and over $x$ bounds the bad event $\mathbf B$ by $\delta$. $\mathbf B$ is defined on the array alone, so its probability is algorithm-independent; on $\mathbf B^c$ every realized sample mean is within $r_{n_t(x)}$ of $f(x)$, and unqueried candidates are covered by $[0,B]$.
\end{proof}

\begin{remark}
Sharper law-of-the-iterated-logarithm radii\cite{si-jamieson2014lil,si-howard2021time} can replace $r_k$; all downstream results use only $\delta$-admissibility. For structured pools (linear/RKHS $f$), self-normalized bounds\cite{si-abbasi2011improved,si-chowdhury2017kernelized} give admissible systems whose widths shrink jointly across candidates, removing the tabular one-query-per-candidate floor.
\end{remark}

\section{Main Results: Safe Selective Pseudo-Labeling over Pools}\label{si:sec:main}

\begin{theorem}[certification is invariant to arbitrary pseudo-labeling]\label{si:thm:invariance}
Fix $\delta\in(0,1)$ and a $\delta$-admissible system with $\hat\mu_t(x)\in[L_t(x),U_t(x)]$ (e.g.\ the midpoint system of Lemma~\ref{si:lem:cs}). Consider any algorithm whose proposal rule, surrogates, pseudo-labeling policy, and query schedule are arbitrary measurable functions of the whole history, subject only to:
(1) it outputs $\xh_{\mathrm{out}}$ flagged \textsc{certified} only when $L_t(\xh_{\mathrm{out}})>\max_{x\ne \xh_{\mathrm{out}}}U_t(x)$ or when $\xh_{\mathrm{out}}$ is the unique non-eliminated candidate, and flagged $\varepsilon$-\textsc{optimal} only when every non-eliminated candidate has interval width $\le\varepsilon/2$ and $\xh_{\mathrm{out}}$ maximizes $\hat\mu_t$ among them (all queried at least once);
(2) eliminations only remove candidates $x$ with $U_t(x)<\max_{x'}L_t(x')$ at removal time.
Then with probability at least $1-\delta$: every candidate ever flagged \textsc{certified} is $x^\star$; $x^\star$ is never eliminated; and an $\varepsilon$-\textsc{optimal} output satisfies $f(\xh_{\mathrm{out}})\ge f(x^\star)-\varepsilon$.
\end{theorem}
\begin{proof}
Work on the coverage event $\mathcal G$, whose probability is $\ge1-\delta$
\emph{independently of the algorithm} (Lemma~\ref{si:lem:cs}: the bad event
lives on the pre-sampled array, which pseudo-labels never touch).

\emph{Step 1 (certificates are correct).} The certificate case is exactly
Proposition~\ref{si:prop:cert}.

\emph{Step 2 ($x^\star$ survives).} Eliminating $x^\star$ at time $t$ would
require $U_t(x^\star)<L_t(x')\le f(x')$ for some $x'$, while on $\mathcal G$
also $U_t(x^\star)\ge f(x^\star)\ge f(x')$---a contradiction. Hence the
unique survivor, if that stopping rule fires, is $x^\star$.

\emph{Step 3 ($\varepsilon$-optimal outputs).} Let $A_t$ be the
non-eliminated set at stopping; by Step 2, $x^\star\in A_t$, and every
interval has width $\le\varepsilon/2$. Chaining the interval bounds through
the empirical means,
\begin{equation}\label{si:eq:invchain}
f(\xh_{\mathrm{out}})\ \ge\ L_t(\xh_{\mathrm{out}})
\ \ge\ \hat\mu_t(\xh_{\mathrm{out}})-\tfrac\varepsilon2
\ \ge\ \hat\mu_t(x^\star)-\tfrac\varepsilon2
\ \ge\ L_t(x^\star)-\tfrac\varepsilon2
\ \ge\ f(x^\star)-\varepsilon. \qedhere
\end{equation}
\end{proof}

\begin{theorem}[contaminated certification can fail with probability one]\label{si:thm:contamination}
Call a protocol contaminated if accepted pseudo-labels enter its certification statistics with the same standing as oracle labels (intervals of width $w(k)\to0$ centered at the running mean of all accepted labels) and it certifies via $L>\max U$. There exists an algorithm of this class (accepting pseudo-labels unconditionally while uncertified), a two-candidate noiseless instance, and a surrogate such that with probability one it certifies the suboptimal candidate and stops with regret $B>0$ after zero oracle calls. Consequently no guarantee of the form of Theorem~\ref{si:thm:invariance}---quantifying over all label policies---can hold for any contaminated class containing this behavior.
\end{theorem}
\begin{proof}
$\X=\{x_1,x_2\}$, $f=(0,B)$, $\fh=(B,0)$; accept one pseudo-label per candidate per round. After $k$ rounds the means are exactly $(B,0)$; once $w(k)<B/2$ the certificate $B-w(k)>w(k)$ fires for $x_1$. Deterministic; regret $B$; zero oracle calls.
\end{proof}

\begin{lemma}[log inversion]\label{si:lem:inv}
For $a,b>0$ with $4a\sqrt b\ge e$: every $k\ge4a\ln(4a\sqrt b)$ satisfies $k\ge a\ln(bk^2)$.
\end{lemma}
\begin{proof}
With $L=\ln(4a\sqrt b)\ge1$, $k_0=4aL$: $a\ln(bk_0^2)=2a[L+\ln L]\le4aL=k_0$; and $k-a\ln(bk^2)$ is increasing for $k\ge2a$, while $k_0\ge4a$.
\end{proof}

\begin{theorem}[oracle complexity of certified elimination]\label{si:thm:se}
Run \textsc{Gated-SE}: in phases, every non-eliminated candidate below target precision receives one oracle query; certified-suboptimal candidates are eliminated; stop at the $\varepsilon$-rule or when one candidate survives; pseudo-labels are used arbitrarily for proposal only. Assume a unique maximizer (with ties, replace $\bar\Delta_{x^\star}$ below by $\varepsilon/4$). Set $\bar\Delta_x=\tfrac14\max(\Delta_x,\varepsilon)$ for suboptimal $x$ and $\bar\Delta_{x^\star}=\tfrac14\max(\Delta_{\min},\varepsilon)$, $\Delta_{\min}=\min\{\Delta_x:\Delta_x>0\}$. With probability $\ge1-\delta$, \textsc{Gated-SE} stops with a certified or $\varepsilon$-optimal recommendation and each candidate's oracle-query count satisfies, whenever $\tfrac{8\sigma^2}{\bar\Delta_x^2}\sqrt{\tfrac{\pi^2N}{3\delta}}\ge e$,
\[
n(x)\ \le\ 1+\frac{8\sigma^2}{\bar\Delta_x^2}\,\ln\!\Big(\frac{8\sigma^2}{\bar\Delta_x^2}\sqrt{\tfrac{\pi^2N}{3\delta}}\Big).
\]
\end{theorem}
\begin{proof}
Correctness is Theorem~\ref{si:thm:invariance}; it remains to count queries
on $\mathcal G$.

\emph{Step 1 (when a suboptimal $x$ leaves the queried set).} With all
active candidates at $k$ queries, $x$ is certified suboptimal once
$r_k<\Delta_x/4$, because then
\begin{equation}\label{si:eq:se-elim}
\hat\mu(x^\star)-\hat\mu(x)\ \ge\ \Delta_x-2r_k\ >\ 2r_k .
\end{equation}
If $x^\star$ froze earlier at radius $r^\star\le\varepsilon/4$: for
$\Delta_x\ge\varepsilon$, elimination needs $\Delta_x>2r^\star+2r_k$, which
is implied by $2r^\star\le\varepsilon/2\le\Delta_x/2$ together with
$r_k<\Delta_x/4$; for $\Delta_x<\varepsilon$ the width trigger
$r_k\le\varepsilon/4$ bounds the count instead. Every candidate also stops
being queried at $r_k\le\varepsilon/4$ regardless.

\emph{Step 2 (inverting the radius).} By Step 1, $x$ is queried only while
$r_k\ge\bar\Delta_x$, i.e.\ while
\begin{equation}\label{si:eq:se-invert}
k\ \le\ \frac{2\sigma^2}{\bar\Delta_x^2}\,
\ln\!\Big(\frac{\pi^2Nk^2}{3\delta}\Big),
\end{equation}
which Lemma~\ref{si:lem:inv} (with $a=2\sigma^2/\bar\Delta_x^2$,
$b=\pi^2N/(3\delta)$) inverts into the stated bound on $n(x)$.

\emph{Step 3 (the winner).} $x^\star$ is queried only until its last
competitor is eliminated (the phase count solving $r_k<\Delta_{\min}/4$) or
until its own width trigger fires, whichever is first---giving
$\bar\Delta_{x^\star}=\tfrac14\max(\Delta_{\min},\varepsilon)$.
\end{proof}

\begin{theorem}[oracle complexity under an $\alpha$-margin law]\label{si:thm:margin}
Assume additionally $\#\{x:0<\Delta_x\le u\}\le N\min(1,Cu^\alpha)$ for all $u>0$, a unique maximizer, at least one suboptimal candidate, and the proviso of Theorem~\ref{si:thm:se} for all $x$. Let $u_\star=C^{-1/\alpha}$ and $\Lambda=\max_x\ln(8\sigma^2\bar\Delta_x^{-2}\sqrt{\pi^2N/(3\delta)})$. On $\mathcal G$ the total cost satisfies $\mathrm{Cost}\le N+8\sigma^2\Lambda\,S$ with
\[
S\le S_{\mathrm{sub}}+16NC^{2/\alpha}+16\max(\Delta_{\min},\varepsilon)^{-2},\qquad
16\max(\Delta_{\min},\varepsilon)^{-2}\le16(NC)^{2/\alpha},
\]
and, for $\varepsilon<u_\star$,
\[
S_{\mathrm{sub}}\ \le\ \frac{32NC}{2-\alpha}\,\varepsilon^{\alpha-2}\ (\alpha<2),\qquad
32NC\ln\frac{u_\star}{\varepsilon}\ (\alpha=2),\qquad
\frac{32NC^{2/\alpha}}{\alpha-2}\ (\alpha>2).
\]
In particular the cost is $\varepsilon$-independent for $\alpha>2$: only an ambiguous core requires high precision, and the winner needs only $\Delta_{\min}$-precision because eliminating its last competitor fires the single-survivor stop.
\end{theorem}
\begin{proof}
\emph{Step 1 (layer-cake form of the cost sum).} Since
$\bar\Delta_x^{-2}=\int_{\bar\Delta_x}^\infty 2u^{-3}\,du$,
\begin{equation}\label{si:eq:margin-cake}
S_{\mathrm{sub}}\ =\ \int_{\varepsilon/4}^\infty 2u^{-3}\,
\#\{x\ne x^\star:\bar\Delta_x<u\}\,du,
\qquad
\#\{\cdot\}\ \le\ N\min\!\big(1,\,C(4u)^\alpha\big)\ \ (u>\varepsilon/4).
\end{equation}
The maximizer contributes its own term
$16\max(\Delta_{\min},\varepsilon)^{-2}$, and the margin law forces
$\Delta_{\min}\ge(NC)^{-1/\alpha}$ (otherwise
$\#\{0<\Delta\le u\}\le NCu^\alpha<1$).

\emph{Step 2 (split at the saturation point $u_\star/4$).} The saturated
tail of \eqref{si:eq:margin-cake} gives
\begin{equation}\label{si:eq:margin-tail}
2N\int_{u_\star/4}^\infty u^{-3}\,du\ =\ 16NC^{2/\alpha},
\end{equation}
while the head evaluates, in the three regimes of $\alpha$, to
\begin{equation}\label{si:eq:margin-head}
2NC\,4^\alpha\int_{\varepsilon/4}^{u_\star/4}u^{\alpha-3}\,du
\ =\
\begin{cases}
\dfrac{32NC}{2-\alpha}\,\varepsilon^{\alpha-2} & (\alpha<2),\\[6pt]
32NC\,\ln\dfrac{u_\star}{\varepsilon} & (\alpha=2),\\[6pt]
\dfrac{32NC^{2/\alpha}}{\alpha-2} & (\alpha>2),
\end{cases}
\end{equation}
using $2NC\,4^\alpha(u_\star/4)^{\alpha-2}/(\alpha-2)=32NC^{2/\alpha}/(\alpha-2)$
in the last case.
\end{proof}

\begin{remark}[the tabular floor]\label{si:rem:floor}
Every candidate needs $\ge1$ oracle query in the tabular model, so $\mathrm{Cost}\ge N$: Theorems~\ref{si:thm:invariance}--\ref{si:thm:margin} establish that, in the pool setting, certified oracle complexity cannot fall below $N$ queries, and pseudo-labels therefore cannot reduce it here; their provable value must come from structured classes (see Section~\ref{si:sec:open}).
\end{remark}

\subsection{Sequential Decisions: Finite Policy Classes}
\label{si:sec:seq}

The main-text limitation ``the guarantees cover single-decision
searches'' deserves a precise boundary. The \emph{policy-selection} form
of the sequential problem---choose the best of $K$ candidate policies,
controllers, prompts, or agent configurations by running real
episodes---is already covered by the pool theorems, because an episode
return is an oracle call. The genuinely open form is per-step credit
assignment \emph{inside} a policy's learning; we close the first and
formalize why the second is the right name for what remains.

\begin{corollary}[policy-level invariance and cost]\label{si:cor:seq}
Consider an episodic environment with returns in $[0,B]$ and a finite
policy class $\Pi=\{\pi_1,\dots,\pi_K\}$; executing $\pi$ for one episode
returns $G(\pi)$ with $\E[G(\pi)]=J(\pi)$ and $G(\pi)-J(\pi)$
sub-Gaussian with scale $\sigma\le B/2$, independent across episodes. A
\emph{pseudo-rollout} is any surrogate-generated return estimate: a
learned world model's simulated episode, a critic's value prediction, an
offline model's score. Identify $\X=\Pi$, $f=J$, one oracle query $=$ one
real episode (pre-sampled array indexed by policy and episode count).
Then Theorems~\ref{si:thm:invariance}, \ref{si:thm:se},
and~\ref{si:thm:margin} hold verbatim with $N=K$:
\begin{enumerate}[label=(\roman*),leftmargin=2em,itemsep=1pt]
\item any agent whose proposal machinery---world models, critics,
pseudo-rollouts, replay reweighting, arbitrary retraining---is
unrestricted, but whose certification statistics (per-policy intervals,
eliminations, stopping, the reported policy and value) are computed from
real episode returns only, certifies a truly best (or
$\varepsilon$-best) policy with probability $\ge1-\delta$, uniformly
over surrogate behavior;
\item each policy costs $O(\sigma^2\Delta_\pi^{-2}\log(K/\delta))$ real
episodes, and under an $\alpha$-margin law on the $J$-gaps the three
cost regimes of Theorem~\ref{si:thm:margin} hold.
\end{enumerate}
\end{corollary}
\begin{proof}
Immediate specialization: bounded returns are $\sigma$-sub-Gaussian with
$\sigma\le B/2$ (Hoeffding's lemma), the pre-sampled array of
Section~\ref{si:sec:setting} is indexed by $(\pi,k)$, and pseudo-rollouts
enter the protocol of Section~\ref{si:sec:framework} exactly as
pseudo-labels ($G_t=L$ records), which never touch the array.
\end{proof}

\begin{proposition}[replay-buffer contamination, formal]\label{si:prop:replay}
Call a certification scheme \emph{contaminated} if it maintains, for
each candidate policy, a value estimate computed from a buffer in which
surrogate-generated returns (model-based rollouts, critic-relabeled
rewards) enter with the same standing as real episode returns, with
interval widths $w(k)\to0$ in the \emph{total} (real $+$ surrogate)
count $k$, and it certifies via the rule $L>\max U$. There exist a
scheme of this class (accepting one model rollout per policy per round
while uncertified), a two-policy, one-step episodic instance ($J=(0,B)$,
deterministic returns), and a world model (predicting returns $(B,0)$)
such that the scheme certifies the \emph{wrong} policy with probability
one after \emph{zero} real episodes. Consequently no guarantee of the
form of Corollary~\ref{si:cor:seq}(i) can hold for any contaminated
class containing this behavior.
\end{proposition}
\begin{proof}
The instance is the construction of
Theorem~\ref{si:thm:contamination} with candidates $=$ policies and
pseudo-labels $=$ model rollouts: buffer means are deterministically
$(B,0)$, so once $w(k)<B/2$ the certificate fires for the $J=0$ policy.
\end{proof}

\begin{remark}[what remains open]\label{si:rem:seq-open}
Corollary~\ref{si:cor:seq} treats whole-episode returns as atomic oracle
calls, which sidesteps compounding model error (the
$1/(1-\gamma)$-factor regime of the simulation lemma,
Proposition~\ref{si:prop:sim}). The open problem is the fine-grained
version: certifying a policy \emph{improvement} step from real
transitions only while the proposal machinery relabels rewards---i.e.,
credit assignment, not selection. Proposition~\ref{si:prop:replay} shows
the separation principle is already non-negotiable there.
\end{remark}

\section{Oracle Admissibility: When May a Model Act as the Oracle?}\label{si:sec:adm}
A criterion $\mathcal K(\fh,f)$ is \emph{admissible} if it is (A1) decision-sufficient, (A2) checkable from oracle data at cost sublinear in the decisions it licenses, and (A3) robust to the search adaptively seeking the surrogate's errors.

\begin{theorem}[the minimal criterion is ordinal]\label{si:thm:ordinal}
For a pool $C\subseteq\X$ let $r_C(\fh)=\max_{x\in C}f(x)-\min_{\xh\in\arg\max_C\fh}f(\xh)$ (worst-case tie-breaking). For any $\varepsilon\ge0$:
\[
\sup_{\varnothing\ne C\subseteq\X}r_C(\fh)\le\varepsilon
\quad\Longleftrightarrow\quad
\forall x,y\in\X:\ f(x)-f(y)>\varepsilon\ \Rightarrow\ \fh(x)>\fh(y).
\]
For $\varepsilon=0$: decision-safety on every pool is exactly strict order preservation (equivalently $\fh=g\circ f$ for strictly increasing $g$ when $f$ is injective).
\end{theorem}
\begin{proof}
($\Leftarrow$) For any pool $C$ and any $\fh$-maximizer $\xh$: if $f(x^\star_C)-f(\xh)>\varepsilon$ then $\fh(x^\star_C)>\fh(\xh)$, contradicting maximality ($x^\star_C\in C$). ($\Rightarrow$) If $f(x)-f(y)>\varepsilon$ with $\fh(y)\ge\fh(x)$, the pool $\{x,y\}$ has $y$ among the $\fh$-maximizers and $r_{\{x,y\}}=f(x)-f(y)>\varepsilon$.
\end{proof}

\begin{theorem}[no per-decision certification at a profit]\label{si:thm:noprofit}
Oracle queries return $f(\cdot)+\mathcal N(0,\sigma^2)$. Fix a pair $\{x,y\}$ with $\fh(x)>\fh(y)$ and tolerance $\varepsilon\ge0$. A certifier is any adaptive procedure outputting \textsc{safe}/\textsc{unsafe} such that for every $f$: $f(y)-f(x)>\varepsilon\Rightarrow\Prob_f(\textsc{safe})\le\delta<1/2$ (soundness). If on some instance $f_0$ with $f_0(x)-f_0(y)=\Delta>0$ it outputs \textsc{safe} with probability $\ge1/2$, then
\[
\E_{f_0}[N_x+N_y]\ \ge\ \frac{2\sigma^2(1-2\delta)^2}{(\Delta+\varepsilon)^2}.
\]
Since deciding the comparison directly costs $O(\sigma^2\Delta^{-2}\log(1/\delta))$, certification at tolerance $\varepsilon=O(\Delta)$ costs, up to constants and a log factor, as much as not using the surrogate at all.
\end{theorem}
\begin{proof}
\emph{Step 1 (an indistinguishable unsafe alternative).} Let $f_1=f_0$
except $f_1(x)=f_0(x)-(\Delta+\varepsilon+\gamma)$; then
$f_1(y)-f_1(x)=\varepsilon+\gamma>\varepsilon$, so soundness forces
$\Prob_{f_1}(\textsc{safe})\le\delta$.

\emph{Step 2 (change of measure).} Transcript laws under $f_0$ and $f_1$
differ only at queries of $x$, with per-query KL
$(\Delta+\varepsilon+\gamma)^2/2\sigma^2$; the divergence decomposition for
adaptively sampled, stopped experiments\cite{kaufmann2016complexity}
gives
\begin{equation}\label{si:eq:noprofit-kl}
\KL\ =\ \E_{f_0}[N_x]\;\frac{(\Delta+\varepsilon+\gamma)^2}{2\sigma^2}
\end{equation}
(trivially true if infinite). By Pinsker,
$\tfrac12-\delta\le\TV\le\sqrt{\KL/2}$, so $\KL\ge(1-2\delta)^2/2$ and
\begin{equation}\label{si:eq:noprofit-nx}
\E_{f_0}[N_x]\ \ge\ \frac{\sigma^2(1-2\delta)^2}{(\Delta+\varepsilon+\gamma)^2};
\end{equation}
let $\gamma\to0$.

\emph{Step 3 (symmetrize and sharpen).} The symmetric alternative raising
$f(y)$ bounds $\E[N_y]$ identically, giving the display. For a certifier
that is also complete at level $\delta$, Bretagnolle--Huber yields
$\KL\ge\ln(1/4\delta)$, recovering the $\log(1/\delta)$ scaling.
\end{proof}

\begin{theorem}[audit--deploy amplification; the selection tax]\label{si:thm:audit}
A fixed $\fh$ is audited, then deployed: a pool of $N$ i.i.d.\ candidates $X_i\sim D$ is screened and $\xh=\arg\max_i\fh(X_i)$ selected. Call $x$ bad if $\fh(x)-f(x)>\tau$ ($\tau\ge0$) and let $p=\Prob_{X\sim D}(X\text{ bad})$. For the audit-cost claims assume a noiseless oracle (with $\sigma$-noise, per-query costs multiply by $O(\sigma^2/\gamma_\tau^2\cdot\log)$ for a separation margin $\gamma_\tau$).
\textbf{(i)} Always $\Prob(\xh\text{ bad})\le1-(1-p)^N\le Np$.
\textbf{(ii)} For every $p\in(0,1)$, $N$, there exist $f,\fh,D$ with $\Prob(\xh\text{ bad})=1-(1-p)^N\ge(1-\tfrac1e)\min(1,Np)$: the deployed failure rate is $\Theta(\min(1,Np))$ worst-case.
\textbf{(iii)} To guarantee deployed failure $\le\rho$ from a passive audit (audit points i.i.d.\ from $D$) one must certify $p\lesssim\rho/N$, which requires $\Omega(N/\rho)$ oracle queries; a selection-aware audit---repeatedly draw a pool, select with $\fh$, verify only the selected point---certifies the deployed rate directly with $O(\rho^{-1}\log(1/\delta))$ oracle queries. Selection-aware auditing is cheaper by a factor $N$.
\end{theorem}
\begin{proof}
\emph{(i) Union ceiling.} $\{\xh\text{ bad}\}\subseteq\{\exists
i:X_i\text{ bad}\}$, an i.i.d.\ union.

\emph{(ii) Matching construction.} Take $D$ uniform on $[0,1]$,
$G\subset[0,1]$ of measure $p$, $f\equiv0$, and $\fh=(\tau+1)\mathbf 1_G$.
A candidate is bad iff it lies in $G$; if any pool member lands in $G$ the
argmax is in $G$ (strictly largest $\fh$), else the selected point is good.
Hence
\begin{equation}\label{si:eq:audit-tight}
\Prob(\xh\text{ bad})\ =\ 1-(1-p)^N\ \ge\ 1-e^{-Np}\ \ge\
\Big(1-\tfrac1e\Big)\min(1,Np),
\end{equation}
the last step by concavity against the chord on $[0,1]$ and monotonicity
beyond.

\emph{(iii) Audit costs.} Requiring $1-(1-p)^N\le\rho$ forces
$p\le1-(1-\rho)^{1/N}\le\tfrac{\rho}{N}(1+O(\rho))$. For the passive lower
bound, argue by hitting: fix $\fh=(\tau+1)\mathbf 1_G$ with
$\mu(G)=2\rho/N$; instance $f_0$ has $f=\tau+1$ on $G$ (nothing bad, safe),
instance $f_1$ has $f=0$ on $G$ (deployed failure $>\rho$ for
$\rho\le1/2$, unsafe). Audit samples $(X,f(X))$ are identical under
$f_0,f_1$ unless $X\in G$ (probability $2\rho/N$ per sample), so
distinguishing with constant advantage needs $m=\Omega(N/\rho)$ samples.
(An \emph{active} auditor could query inside $G$; auditing where the
surrogate's stakes are is precisely the targeted design whose canonical
form is the selection-aware audit, and Theorem~\ref{si:thm:lb} below
closes the remaining loophole by bounding \emph{every} design.) For the
selection-aware audit, rounds are i.i.d.\ $\mathrm{Ber}(q)$ with exactly
the deployed $q$; a Chernoff test certifies $q\le\rho$ in
$O(\rho^{-1}\log(1/\delta))$ rounds, one oracle call each.
\end{proof}

\begin{remark}[license expiry]\label{si:rem:expiry}
The license binds a fixed $\fh$ and a fixed $D$. Retraining voids it (the surrogate cannot renew its own license---Theorem~\ref{si:thm:contamination}); proposal-distribution drift $\beta=\TV(D',D)$ degrades it additively, $\le\rho+N\beta$, via $\TV(D'^{\otimes N},D^{\otimes N})\le N\beta$.
\end{remark}

\begin{corollary}[rank-audit version]\label{si:cor:rank}
With $\qinv=\Prob_{(X,X')\sim D^2}\big(f(X)-f(X')>\varepsilon,\
\fh(X')\ge\fh(X)\big)$ (ordered pairs), an i.i.d.\ pool of size $N$
satisfies
\begin{equation}\label{si:eq:rank-union}
\Prob(r_{\mathrm{pool}}>\varepsilon)\ \le\ N(N-1)\,\qinv
\end{equation}
(Theorem~\ref{si:thm:ordinal} on the pool, then a union over ordered
pairs), and $\Omega(N\qinv)$ is attainable ($D$ places mass $\pi$ on a
point with maximal $f$ and minimal $\fh$; $\qinv=\Theta(\pi)$ while
failure $=\Theta(\min(1,N\pi))$). Both questions this bound raises are settled in
Section~\ref{si:sec:si}: the $N^2$ amplification is
exactly attained (Theorem~\ref{si:thm:ag2}), and the champion audit is a
cheap selection-aware ordinal audit with tax exactly $\Theta(N)$
(Theorem~\ref{si:thm:ag1}), which no audit design of any kind can improve
upon beyond constants and the $\log(1/\delta)$ factor
(Theorem~\ref{si:thm:lb}).
\end{corollary}

\begin{theorem}[one-sided admissibility: pessimism as a checkable license]\label{si:thm:pessimism}
\textbf{(i)} If $\eplus(x)\le\tau$ for all $x\in C$, then $r_C\le\eminus(x^\star_C)+\tau$ and $f(\xh)\ge\fh(\xh)-\tau$ (no-disappointment).
\textbf{(ii)} If $\Prob_{X\sim D}(\eplus(X)>\tau)\le p$, then for an i.i.d.\ pool of size $N$, with probability $\ge1-Np$ both conclusions of (i) hold at the pool level. The factor $N$ is unavoidable by Theorem~\ref{si:thm:audit}(ii) (whose construction is one-sided and defeats the no-disappointment clause; its regret clause is vacuous there), and the selection-aware audit of Theorem~\ref{si:thm:audit}(iii) certifies the deployed rate at $O(\rho^{-1}\log(1/\delta))$.
\end{theorem}
\begin{proof}
(i) Proposition~\ref{si:prop:decomp} with $\eplus(\xh)\le\tau$; and $f(\xh)\ge\fh(\xh)-\eplus(\xh)$. (ii) Union bound over the pool, then (i).
\end{proof}

\subsection{Summary of the Admissibility Criteria}
\begin{center}
\small
\begin{tabular}{p{4.1cm}p{4.4cm}p{5.2cm}}
\toprule
\textbf{Criterion} & \textbf{Verdict} & \textbf{Governing result}\\
\midrule
High $R^2$ / low MSE & rejected & Thms~\ref{si:thm:r2},~\ref{si:thm:mse}\\
Marginal calibration & rejected (blind to selection) & Thms~\ref{si:thm:evt},~\ref{si:thm:audit}(ii)\\
Model's own confidence & rejected outright & Thm~\ref{si:thm:contamination}\\
Per-decision verification & sound but unprofitable & Thm~\ref{si:thm:noprofit}\\
$\varepsilon$-margin rank preservation & the exact minimal criterion & Thm~\ref{si:thm:ordinal}; tax $N$--$N^2$ (Cor.~\ref{si:cor:rank})\\
One-sided no-over-praise $(\tau,p)$ & cheapest usable license & Thm~\ref{si:thm:pessimism}\\
Passive i.i.d.\ audit & understates risk by factor $N$ & Thm~\ref{si:thm:audit}(ii)(iii)\\
Selection-aware audit + expiry & recommended license & Thm~\ref{si:thm:audit}(iii), Rmk~\ref{si:rem:expiry}\\
\bottomrule
\end{tabular}
\end{center}
\noindent\textbf{Conclusion of the admissibility theory.} A usable ``may serve as oracle'' criterion exists, but only in one shape, and that shape is forced: an ordinal or one-sided property (Thms~\ref{si:thm:ordinal},~\ref{si:thm:pessimism}), certified by a selection-aware audit (Thm~\ref{si:thm:audit}(iii)), carrying an expiry (Rmk~\ref{si:rem:expiry}), with all certification statistics oracle-only (Thms~\ref{si:thm:invariance},~\ref{si:thm:contamination}).

\section{Strict Improvement and the Ordinal Audit Gap}\label{si:sec:si}
This section settles the two questions raised by
Corollary~\ref{si:cor:rank} and Remark~\ref{si:rem:floor}, and
establishes a universal audit-design lower bound
(Theorem~\ref{si:thm:lb}) and a noisy-oracle license
(Theorem~\ref{si:thm:noisy}).

\subsection{Strict Improvement: The Dichotomy}
For fixed-menu argmax certification the tabular floor (Remark~\ref{si:rem:floor}) forbids any oracle-cost reduction from pseudo-labels. The generative goal behaves oppositely. Candidates are drawn freely (cost $0$) i.i.d.\ from $D$; the goal is a design $x_{\mathrm{out}}$ with a $(1{-}\delta)$-certificate that $f(x_{\mathrm{out}})\ge v$, where $p_v=\Pr_{X\sim D}(f(X)\ge v)$. In the i.i.d.-marks information model (distinct atoms carry i.i.d.\ $\mathrm{Bernoulli}(p_v)$ marks independent of the identity sequence; values revealed only by querying; joint soundness $\Pr(\textsc{certified}\wedge f(x_{\mathrm{out}})<v)\le\delta$):

\begin{theorem}[oracle-only lower bound; the bound itself is classical---quantile bandits,\cite{chaudhuri2017quantile,aziz2018pure} cf.\ Katz-Samuels and Jamieson\cite{si-katzsamuels2020good}]\label{si:thm:si1}
For $p_v\le 1/2$, any $\delta$-sound oracle-only algorithm succeeding with probability $\ge 1/2$ has $\E[\#\text{queries}]\ \ge\ (\tfrac12-2\delta)/p_v$.
\end{theorem}
\begin{proof}
Let $G$ be the event that some queried candidate is good, and $Q$ the
number of distinct queried candidates with marks $B_1,B_2,\dots$ in query
order.

\emph{Step 1 (success implies a good query).} Split the success event
three ways: queried-and-good ($\subseteq G$); certified-and-bad
(probability $\le\delta$ by joint soundness); and
unqueried-and-good. For the last, an unqueried candidate's mark is
independent of the transcript, so
$\Prob(\text{unqueried}\wedge\text{good})
=\tfrac{p_v}{1-p_v}\Prob(\text{unqueried}\wedge\text{bad})
\le \tfrac{p_v}{1-p_v}\,\delta\le\delta$ for $p_v\le\tfrac12$. Hence
$\tfrac12\le\Prob(G)+2\delta$.

\emph{Step 2 (previsible union bound).} The event $\{Q\ge k\}$ is decided
before the $k$-th mark is seen (marks are independent of the identity
sequence), so $\Prob(Q\ge k,\,B_k{=}1)=p_v\Prob(Q\ge k)$ and
\begin{equation}\label{si:eq:si1-union}
\Prob(G)\ \le\ \sum_{k\ge1}\Prob(Q\ge k,\ B_k=1)
\ =\ p_v\sum_{k\ge1}\Prob(Q\ge k)\ =\ p_v\,\E[Q].
\end{equation}
Combining the steps, $\E[Q]\ge(\tfrac12-2\delta)/p_v$; total queries
dominate $Q$.
\end{proof}

\begin{theorem}[licensed screening]\label{si:thm:si2}
Fix $\rho\in(0,1/4]$, $\delta\in(0,1)$, $T_a\ge (2/\rho)\log(1/\delta)$. Audit: for $t\le T_a$ draw an $M$-pool, query the oracle at its champion, record $Y_t=f(\hat x_t)$; set $\hat v=$ the $\lceil 2\rho T_a\rceil$-th smallest $Y_t$. Then with probability $\ge 1-\delta$ over the audit, every future deployment (pools i.i.d.\ from the audited $D$, $\hat f$ and $M$ fixed) satisfies $\Pr(f(\text{champion})<\hat v)\le 4\rho$, at amortized certified cost $1+(T_a+L)/T$ oracle calls per design ($L$ = surrogate-training labels). If $\hat f$ is $\varepsilon$-margin order-preserving (Thm~\ref{si:thm:ordinal}), $\hat v$ reaches the $(1-\log(1/\rho)/M)$-quantile of $f$ minus $\varepsilon$: good-mass level $p_v \asymp \log(1/\rho)/M$.
\end{theorem}
\begin{proof}
Let $Y$ have the champion-value law, so the audit records i.i.d.\ copies
$Y_1,\dots,Y_{T_a}$.

\emph{Step 1 (a fixed threshold).} Set
$u^\star=\sup\{v: \Prob(Y<v)\le 4\rho\}$; right-continuity of CDFs gives
$\Prob(Y\le u^\star)\ge 4\rho$.

\emph{Step 2 (a bad license forces a rare binomial event).} Suppose
$\Prob(Y<\hat v)>4\rho$. Then $\hat v>u^\star$, so every sample
$\le u^\star$ lies strictly below the $\lceil 2\rho T_a\rceil$-th order
statistic, giving $\#\{t: Y_t\le u^\star\}\le \lceil 2\rho
T_a\rceil-1<2\rho T_a$ (ties only shrink the count). But that count
stochastically dominates $\mathrm{Bin}(T_a,4\rho)$, and by the
multiplicative Chernoff lower tail,
\begin{equation}\label{si:eq:si2-chernoff}
\Prob\big[\mathrm{Bin}(T_a,4\rho)<2\rho T_a\big]
=\Prob\big[\mathrm{Bin}<\tfrac12\mu\big]
\ \le\ e^{-\mu/8}\ =\ e^{-\rho T_a/2}\ \le\ \delta
\quad\text{for } T_a\ge\tfrac2\rho\log\tfrac1\delta .
\end{equation}
So with probability $\ge1-\delta$ over the audit,
$\Prob(Y<\hat v)\le4\rho$.

\emph{Step 3 (deployment).} Future champions are fresh copies of $Y$
(same $D$, $\fh$, $M$; pools independent), giving the display; costs are
by construction. For a $\sigma$-noisy oracle the license carries an
explicit slack at no extra rounds---Theorem~\ref{si:thm:noisy} below.

\emph{Step 4 (the level the license reaches, final clause).} Let
$q=q_{1-c/M}$ be the $(1-c/M)$-quantile of $f(X)$, $X\sim D$, with
$c=\log(1/\rho)$. The pool max falls below $q$ with probability
$(1-c/M)^M\le e^{-c}=\rho$; if $\fh$ is $\varepsilon$-margin
order-preserving the champion is within $\varepsilon$ of its pool max
(Theorem~\ref{si:thm:ordinal}), so $\Prob(Y<q-\varepsilon)\le\rho$.
Hence $\#\{t: Y_t<q-\varepsilon\}$ is stochastically dominated by
$\mathrm{Bin}(T_a,\rho)$, and $\hat v<q-\varepsilon$ requires that count
to reach $\lceil2\rho T_a\rceil\ge2\rho T_a$, twice its mean; by the
multiplicative Chernoff upper tail this has probability at most
$e^{-\rho T_a/3}\le\delta^{2/3}$. So with probability
$\ge1-\delta^{2/3}$ over the audit, $\hat v\ge q_{1-\log(1/\rho)/M}
-\varepsilon$: the license sits at good-mass level
$p_v\asymp\log(1/\rho)/M$.
\end{proof}

\begin{corollary}[separation; the dichotomy]\label{si:cor:si3}
At level $p_v\asymp\log(1/\rho)/M$: oracle-only cost per certified design is $\Omega(M/\log(1/\rho))$ (Thm~\ref{si:thm:si1}); the licensed pipeline pays $(1+(T_a{+}L)/T)/(1-4\rho)\to 1/(1-4\rho)$ per certified-good design---a $\Theta(M/\log(1/\rho))$ separation, \emph{conditional on the audit passing}, with graceful degradation otherwise, in the following precise sense: the $4\rho$ validity of the license (Thm~\ref{si:thm:si2}) holds for \emph{any} surrogate, with no order-preservation assumption; what degrades is only the \emph{level} $\hat v$, for which no lower bound is claimed absent order preservation---a weak surrogate passes its audit honestly and certifies a weak floor, alerting the operator rather than deceiving them. Pseudo-labels pay exactly when the goal is ``find and certify a good design,'' never when it is ``certify the best of this menu.''
\end{corollary}

\subsection{The Ordinal Audit Gap: \texorpdfstring{$\Theta(N)$}{Theta(N)} versus \texorpdfstring{$\Theta(N^2)$}{Theta(N\^{}2)}}
Pool $X_1,\dots,X_N$ i.i.d.\ $\sim D$; champion
$\xh=\arg\max_i \fh(X_i)$; failure event
\begin{equation}\label{si:eq:failev}
\mathrm{FAIL}\ =\ \Big\{\max_i f(X_i)-f(\xh)>\varepsilon\Big\}.
\end{equation}
Two audit statistics, each measurable at 2 oracle calls per round
(the second pits a fresh draw against the champion of an independent
$m$-pool):
\begin{align}
\qinv &= \Prob_{(X,X')\sim D^2}\big(f(X)-f(X')>\varepsilon,\ \fh(X')\ge\fh(X)\big),
\label{si:eq:qinv}\\
\qsel(m) &= \Prob\big(f(X')>f(\xh_m)+\varepsilon\big).
\label{si:eq:qsel}
\end{align}

\begin{theorem}[champion audits: tax $\Theta(N)$]\label{si:thm:ag1}
Always $\Prob(\mathrm{FAIL})\le N\,\qsel(N-1)$; and a one-atom trap family
($\pi\le 1/2$; $f$ constant on the ordinary region) attains
$\Prob(\mathrm{FAIL})\ge(1-1/e)\min\!\big(1,(N-1)\,\qsel(N-1)\big)$.
\end{theorem}
\begin{proof}
\emph{Upper bound.} $\mathrm{FAIL}=\cup_{i\le N} A_i$ with
$A_i=\{f(X_i)>f(\xh)+\varepsilon\}$. On $A_i$ we have $X_i\ne\xh$, and
removing a non-champion from a pool does not change the champion (the
argmax of $\fh$ under a fixed tie order over a set $S\ni\xh$ is still
$\xh$ for any $S'=S\setminus\{z\}$, $z\ne\xh$). Hence on $A_i$, $\xh$
equals the champion of the $(N{-}1)$-pool $\{X_j\}_{j\ne i}$, which is
independent of $X_i$, so $\Prob(A_i)\le\qsel(N-1)$; union over $i$. (The
shape is an exchangeability argument familiar from record values and
conformal rank statistics;\cite{si-beirami2024theoretical} the audit
statement appears new.)

\emph{Tightness.} Fix $\pi\in(0,1/2]$. Let $D$ put mass $\pi$ on an atom
$w$ with $f(w)=B$, $\fh(w)=-1$, and mass $1-\pi$ on a region $R$ with $f$
\emph{constant} at $B-\varepsilon-\gamma$ on $R$ and $\fh\in(0,1)$
arbitrary there. The champion is never $w$ when any $R$-point is present;
if $w$ lands in the pool alongside at least one $R$-point, the pool max
is $w$ while the champion sits in $R$: regret $\varepsilon+\gamma$,
failure. Hence
\begin{equation}\label{si:eq:ag1-fail}
\Prob(\mathrm{FAIL})\ \ge\ 1-(1-\pi)^N-\pi^N ,
\qquad
\qsel(m)=\pi(1-\pi^m)\ \le\ \pi ,
\end{equation}
and for $\pi\le1/2$ the $\pi^N$ correction is absorbed into the stated
constant.
\end{proof}

\begin{theorem}[pairwise audits: tax $\Theta(N^2)$]\label{si:thm:ag2}
Always $\Prob(\mathrm{FAIL})\le N(N-1)\,\qinv$. Conversely, for every $N$
and all small $\pi_a,\pi_b>0$ there are instances with
$\qinv=\pi_a\pi_b$ exactly and
\begin{equation}\label{si:eq:ag2-claim}
\Prob(\mathrm{FAIL})\ \ge\ \tfrac1{15}\,\min(1,N\pi_a)\,\min(1,N\pi_b),
\end{equation}
so the upper bound is tight up to constants (numerically the sharp
constant is $(1-1/e)^2$).
\end{theorem}
\begin{proof}
\emph{Construction.} Fix $0<\delta''<\varepsilon$. $D$ has three parts:
atom $a$ of mass $\pi_a$ with $f(a)=2\varepsilon+\delta''$, $\fh(a)=1$;
atom $b$ of mass $\pi_b$ with $f(b)=0$, $\fh(b)=2$; and ordinary mass
$1-\pi_a-\pi_b$ with $f\in[0,\delta'']$ and $\fh\in(0,1)$
order-preserving there.

\emph{Step 1 ($\qinv$ is exactly the product).} The only ordered pair with
$f$-gap $>\varepsilon$ \emph{and} an inversion is $(a,b)$: pairs
$(a,\mathrm{ord})$ have the gap but no inversion ($\fh(a)=1>\fh(\mathrm{ord})$);
pairs $(\mathrm{ord},b)$ and $(\mathrm{ord},\mathrm{ord}')$ have $f$-gaps
$\le\delta''<\varepsilon$. Hence $\qinv=\pi_a\pi_b$.

\emph{Step 2 (failure is exactly co-occurrence).} If $b$ is present the
champion is $b$ ($\fh(b)=2$ is the global max); the pool max exceeds
$f(b)+\varepsilon$ iff $a$ is also present (ordinary values are within
$\delta''<\varepsilon$ of $f(b)=0$). If $b$ is absent, the champion is $a$
if present (pool max, regret $0$) or the ordinary $\fh$-max $=$ ordinary
$f$-max (regret $0$). So $\Prob(\mathrm{FAIL})=\Prob(a\in\wedge\,b\in)=:J$
with, by inclusion--exclusion,
\begin{equation}\label{si:eq:ag2-J}
J\ =\ 1-(1-\pi_a)^N-(1-\pi_b)^N+(1-\pi_a-\pi_b)^N .
\end{equation}

\emph{Step 3 (the constant).} The presence events are negatively
correlated, so no naive product bound applies; instead use the exact
integral identity (differentiate \eqref{si:eq:ag2-J} twice; $J=0$ on the
axes)
\begin{equation}\label{si:eq:ag2-int}
J\ =\ N(N-1)\int_0^{\pi_b}\!\!\int_0^{\pi_a}(1-s-t)^{N-2}\,ds\,dt
\ \ge\ N(N-1)\,\pi_a\pi_b\,(1-\pi_a-\pi_b)^{N-2}.
\end{equation}
For $N\pi_a,N\pi_b\le1$ and $N\ge2$: $N(N-1)\ge N^2/2$ and
$(1-\pi_a-\pi_b)^{N-2}\ge(1-2/N)^{N-2}\ge e^{-2}$, so
$J\ge\tfrac{1}{2e^2}(N\pi_a)(N\pi_b)\ge\tfrac1{15}\min(1,N\pi_a)\min(1,N\pi_b)$.
Failure is a \emph{joint} event priced by $\qinv$ at \emph{product}
mass---the ordinal winner's curse, one order higher.
\end{proof}

\begin{corollary}[what to audit]\label{si:cor:ag3}
For end-guarantee $\Prob(\mathrm{FAIL})\le\rho_0$: champion audits need
$\qsel\le\rho_0/N$ ($O((N/\rho_0)\log(1/\delta))$ rounds); pairwise audits
need $\qinv\le\rho_0/N^2$ (Theorem~\ref{si:thm:ag2} is tight)---a robust
factor-$N$ separation between the two designs (precedent for ``the audit
statistic sets the audit cost'': risk-limiting election audits,
ballot-comparison vs.\ ballot-polling\cite{stark2012rla}). Two caveats
apply. First, $\qsel$ also contains genuine fresh-improvement mass
($\Theta(1/N)$ even for a perfect surrogate), which is why the dichotomy
of Corollary~\ref{si:cor:si3} licenses \emph{values} rather than regret.
Second, the guarantee so far concerns one audit design rather than all of
them; this caveat is removed by Theorem~\ref{si:thm:lb}: the champion audit's total oracle cost is optimal
among all designs, so the $\Theta(N)$ tax is a property of the ordinal
guarantee itself.
\end{corollary}

\subsection{A Universal Lower Bound for Ordinal-License Audits, and the
Noisy-Oracle License}\label{si:sec:lb}

\begin{theorem}[no audit design certifies deployed selection quality in
$o(N/\rho_0)$ oracle queries]\label{si:thm:lb}
Fix $\varepsilon\in(0,B)$, $\rho_0\in(0,1/4]$, $\delta\in[0,1/2)$, and
$N\ge\log_2(2/\rho_0)$. An \emph{auditor} is any adaptive procedure that
may draw candidates from a known $D$ at no cost, evaluate a known, fixed
$\fh$ anywhere at no cost, query the oracle $f$ at drawn points at unit
cost, and output \textsc{pass} or \textsc{reject}. Call it \emph{sound}
if on every instance with $\Prob(\mathrm{FAIL})\ge\rho_0$ (deployment as
in \eqref{si:eq:failev}) it passes with probability at most $\delta$.
There is a family of instances, sharing one surrogate and one candidate
distribution and containing a single perfectly safe member $f_0$
(deployed failure $0$ at every $N$), on which every sound auditor spends
\begin{equation}\label{si:eq:lb}
\E_{f_0}[\#\text{oracle queries}]\ \ge\
\big(\Prob_{f_0}(\textsc{pass})-\delta\big)
\Big\lfloor \tfrac{N}{4\rho_0}\Big\rfloor .
\end{equation}
In particular, call an auditor \emph{nontrivial for the family} if it
passes the safe instance $f_0$ with probability at least $1/2$---an
auditor failing this cannot license even a deployment on which the
surrogate's mistakes are provably harmless; every sound, nontrivial
auditor spends at least $(\tfrac12-\delta)\lfloor N/(4\rho_0)\rfloor$
expected oracle queries.
Consequently an audit spending $c$ oracle calls per round needs
$\Omega(N/(c\rho_0))$ rounds, and the champion audit of
Corollary~\ref{si:cor:ag3} is optimal among \emph{all} audit designs for
the ordinal (regret) guarantee, up to the $\log(1/\delta)$ factor. The
bound holds a fortiori for noisy oracles.
\end{theorem}
\begin{proof}
\emph{Construction.} $\X=[0,1]$, $D$ uniform. Fix the surrogate across
all instances: $\fh=-1$ on $[0,\tfrac12]$ (demoted region) and
$\fh(x)=x$ on $(\tfrac12,1]$. Partition $[0,\tfrac12]$ into
$K=\lfloor N/(4\rho_0)\rfloor$ cells $C_1,\dots,C_K$ of mass
$\pi=2\rho_0/N$ each (feasible since $K\pi\le\tfrac12$). Instances:
$f_0\equiv0$, and $f_j=B\cdot\mathbf 1\{C_j\}$ for $j=1,\dots,K$.

\emph{Step 1 ($f_0$ is perfectly safe; each $f_j$ is unsafe).} $f_0$ is
constant, so pool regret is identically $0$ and deployed failure is $0$
at every $N$: $f_0$ is the family's safe member.
Under $f_j$, the champion lies in $[0,\tfrac12]$ only when all $N$ draws
land there, probability $2^{-N}\le\rho_0/2$ by the assumption on $N$. If
the pool hits $C_j$ and the champion is outside $[0,\tfrac12]$, the pool
max is $B$ while $f(\xh)=0$: failure. Hence, using
$1-(1-\pi)^N\ge N\pi(1-N\pi/2)$,
\begin{equation}\label{si:eq:lb-unsafe}
\Prob(\mathrm{FAIL})\ \ge\ \big[1-(1-\pi)^N\big]-2^{-N}
\ \ge\ 2\rho_0(1-\rho_0)-\tfrac{\rho_0}{2}\ \ge\ \rho_0
\qquad(\rho_0\le\tfrac14),
\end{equation}
so soundness forces $\Prob_{f_j}(\textsc{pass})\le\delta$ for every $j$.

\emph{Step 2 (indistinguishability).} All instances share $D$ and $\fh$,
and $f_j$ agrees with $f_0$ outside $C_j$. Couple runs of the auditor
under $f_0$ and $f_j$ (same internal randomness, same draws): transcripts
coincide until the first oracle query lands in $C_j$. With
$T_j=\{\text{some oracle query lands in }C_j\}$,
\begin{equation}\label{si:eq:lb-couple}
\big|\Prob_{f_j}(\textsc{pass})-\Prob_{f_0}(\textsc{pass})\big|\ \le\
\Prob_{f_0}(T_j)
\quad\Longrightarrow\quad
\Prob_{f_0}(T_j)\ \ge\ \Prob_{f_0}(\textsc{pass})-\delta\ \
\text{for every } j.
\end{equation}

\emph{Step 3 (counting).} The cells are disjoint, so
$\sum_j\mathbf 1\{T_j\}\le Q$, the total number of oracle queries. Taking
expectations under $f_0$ and summing \eqref{si:eq:lb-couple} over $j$
gives $K\big(\Prob_{f_0}(\textsc{pass})-\delta\big)\le\E_{f_0}[Q]$, which
is \eqref{si:eq:lb}; the nontrivial case substitutes
$\Prob_{f_0}(\textsc{pass})\ge\tfrac12$.
\end{proof}

\begin{remark}
The auditor is \emph{granted} free draws and full knowledge of $D$ and
$\fh$, and the bound still holds: all $K$ cells look identical through
$(D,\fh)$, so hitting the live cell cannot be planned---the construction
is the one-atom trap of Theorem~\ref{si:thm:ag1} hidden among $K$
interchangeable locations, and no query strategy beats blind search over
them. For $N<\log_2(2/\rho_0)$ a bound of the same character holds with
$K=\lfloor\beta N/(2\rho_0)\rfloor$, $\beta=(\rho_0/2)^{1/N}$; the stated
regime is the deployment-relevant one. On the completeness condition:
the statement pins it to the family's safe member $f_0$, which is what
the coupling uses; the bound persists verbatim under any weaker
completeness notion (e.g., ``pass instances with deployed failure
$\le\rho_0/2$'') as long as it covers $f_0$---and $f_0$ is the canonical
safe instance of the family, indistinguishable from the others through
$(D,\fh)$; an auditor that rejects it half the time licenses nothing on
this family.
\end{remark}

\begin{theorem}[noisy-oracle screening license: slack, not
rounds]\label{si:thm:noisy}
In the setting of Theorem~\ref{si:thm:si2}, let each audit query instead
return $\tilde Y_t=f(\xh_t)+\eta_t$ with $\eta_t$ independent, mean-zero,
$\sigma$-sub-Gaussian, and keep the audit unchanged ($\hat v=$ the
$\lceil2\rho T_a\rceil$-th smallest $\tilde Y_t$; $T_a\ge(2/\rho)\log(1/\delta)$).
Then for any slack $\gamma>0$, with probability $\ge1-\delta$ over the
audit, every future deployment satisfies
\begin{equation}\label{si:eq:noisy}
\Prob\big(f(\mathrm{champion})<\hat v-\gamma\big)\ \le\
4\rho+e^{-\gamma^2/(2\sigma^2)} .
\end{equation}
In particular $\gamma=\sigma\sqrt{2\ln(1/\rho)}$ gives level $5\rho$ at
\emph{no} increase in audit rounds; averaging $r$ repeated queries per
champion shrinks the slack like $1/\sqrt r$.
\end{theorem}
\begin{proof}
Let $Y$ be the champion-value law and $\tilde Y=Y+\eta$; the audit
samples are i.i.d.\ copies of $\tilde Y$, so Theorem~\ref{si:thm:si2}
applied verbatim to the law of $\tilde Y$ yields an audit event of
probability $\ge1-\delta$ on which $\Prob(\tilde Y'<\hat v)\le4\rho$ for
a fresh copy $\tilde Y'$. For any fixed $v$ and a fresh deployment
champion with value $Y'$ and hypothetical audit noise $\eta'$,
\begin{equation}\label{si:eq:noisy-incl}
\{Y'<v-\gamma\}\ \subseteq\ \{Y'+\eta'<v\}\cup\{\eta'>\gamma\}
\quad\Longrightarrow\quad
\Prob(Y'<v-\gamma)\ \le\ \Prob(\tilde Y'<v)+e^{-\gamma^2/(2\sigma^2)} .
\end{equation}
Apply at $v=\hat v$ (a function of the audit alone, independent of the
fresh draw) and intersect with the audit event.
\end{proof}

\begin{remark}[how much slack is necessary]\label{si:rem:noisyslack}
For \emph{symmetric} noise, $\gamma=0$ remains valid at a constant-factor
inflation: every point of $Y$-mass below $v$ sends at least half its
noise mass below $v$, so $\Prob(\tilde Y<v)\ge\tfrac12\Prob(Y<v)$ and the
no-slack license holds at level $8\rho$. For general mean-zero
sub-Gaussian noise the slack is necessary: with an atom of $Y$-mass
$4.5\rho$ just below the license value and two-point noise taking $-s$
with probability $q$ and $+sq/(1-q)$ otherwise ($s\approx a(1-q)/q$), the
rare deep drops park below the license level while both atoms are lifted
above it, and $\Prob(Y<\hat v)=1\gg4\rho$ with constant probability while
\eqref{si:eq:noisy} holds throughout (simulated in
Section~\ref{si:sec:verify}). Quantiles, unlike
means, do not
concentrate without a margin; the theorem prices the margin explicitly
and distribution-freely.
\end{remark}

\section{Extended Experimental Details and Results}\label{si:sec:emp}
All experiments use oracles that are \emph{real trainings} with exactly
known ground truth, evaluated exhaustively once: Task~1, the 1{,}152
gradient-boosting configurations (3-fold CV on the diabetes dataset);
Task~2, the 576 SVM configurations (breast-cancer data); and Task~3, all
15{,}625 NAS-Bench-201 architectures. Surrogates are standard performance
predictors fit on 40--160 labeled configurations (Tasks 1--2) or
100--400 (Task 3). The scripts, data tables, and figure-generation code
will accompany a subsequent \textsc{ModelAssay} release. Table~\ref{si:tbl:summary}
maps each theoretical prediction to its measurement in SI numbering.

\begin{table}[t]
\caption{Theoretical predictions and their measurements, in SI numbering.}
\label{si:tbl:summary}
\centering\small
\begin{tabular}{p{5.6cm}p{4.6cm}p{4.0cm}}
\toprule
\textbf{Theory prediction} & \textbf{Measured} & \textbf{Task(s)}\\
\midrule
Selection tax $q\le 1-(1-p)^N$, biting iff errors align with selection
(Thm~\ref{si:thm:audit}) & never violated in any (fit, regime, $N$) cell;
alignment is per-fit: Task-2 ridge twins at $p{=}0.247/0.252$ deploy at
$q{=}0.000/1.000$; ten-fit Task-1 panels in Fig.~\ref{fig:tax} of the main text &
Tasks 1--2\\
Passive audits fail both directions; selection-aware audit correct
(Thms~\ref{si:thm:audit}(iii),~\ref{si:thm:lb}) & passive rejects a safe
deployment ($p{=}0.048$, $q{=}0.000$) and cannot separate the malign twin;
champion-verification audit licenses correctly in every case &
Tasks 1--2\\
Accuracy $\perp$ deployed failure
(Thms~\ref{si:thm:r2}--\ref{si:thm:evt}) & $R^2$ flat while deployed
failure spans $0\to1.0$ & Tasks 1--3\\
One-sided license (Thm~\ref{si:thm:pessimism}) & holds; and the same
quantile recipe passes Task~1, fails Task~2 in one fit
($p{=}0.116\to30\%$ disappointment); licenses are per-fit, per-task &
Tasks 1--2\\
Verified proposals save oracle calls without touching validity
(Thm~\ref{si:thm:invariance}) & 94 vs 204 calls to equal verified quality
($2.2\times$); contaminated variant claims 0.43, delivers 0.34 &
Task 1 (loop)\\
Screening separation (Cor.~\ref{si:cor:si3}) & $\sim33\times$ at
$M{=}100$ (RF champion good-rate $0.33$ at the top-$1\%$ level; training
labels amortized) & Task 1\\
Champion statistic is the deployment-calibrated diagnostic
(Cor.~\ref{si:cor:ag3}) & Spearman with deployed regret $0.89$--$0.99$ in
all six task--regime cells; see Section~\ref{si:sec:emp-div} &
Tasks 1--3\\
\bottomrule
\end{tabular}
\end{table}

\subsection{The Selection Tax, Fit by Fit}\label{si:sec:emp-tax}
The Task-1 selection-tax study uses ten independent fits per (regime,
model) condition at train size 120, with the ridge fits disaggregated in
the main-text figure (Figure~\ref{fig:tax}). The quantities quoted in the
main text are as follows: the worst-case law is never violated;
interpolating surrogates are benign at every $N$ in both regimes (mean
$q(1000)\le10^{-3}$ at mean audited $p$ up to $0.33$); the ridge fits are
bimodal---under representative training, four of ten fits exceed
$q(1000)=0.97$ while five sit below $0.01$, at audited rates within
$\pm0.03$ of one another; under biased training the same recipe peaks at
mid-$N$ (mean $q(30)=0.31$) and decays. The per-fit split, not the
regime, carries the deployment risk on this task; on Task 3 the same
recipe is deterministically catastrophic in both regimes. Together the
three tasks span the full spectrum the theory allows---benign, bimodal,
deterministic---which is exactly why no audit of the recipe (or the
regime) can substitute for an audit of the fit.
Figure~\ref{si:fig:zoo} summarizes decision quality by model family and
regime on this task, and Figure~\ref{si:fig:license} shows the one-sided
(no-disappointment) license check for the conservative surrogate.

\begin{figure}[H]
\centering
\includegraphics[width=0.47\textwidth]{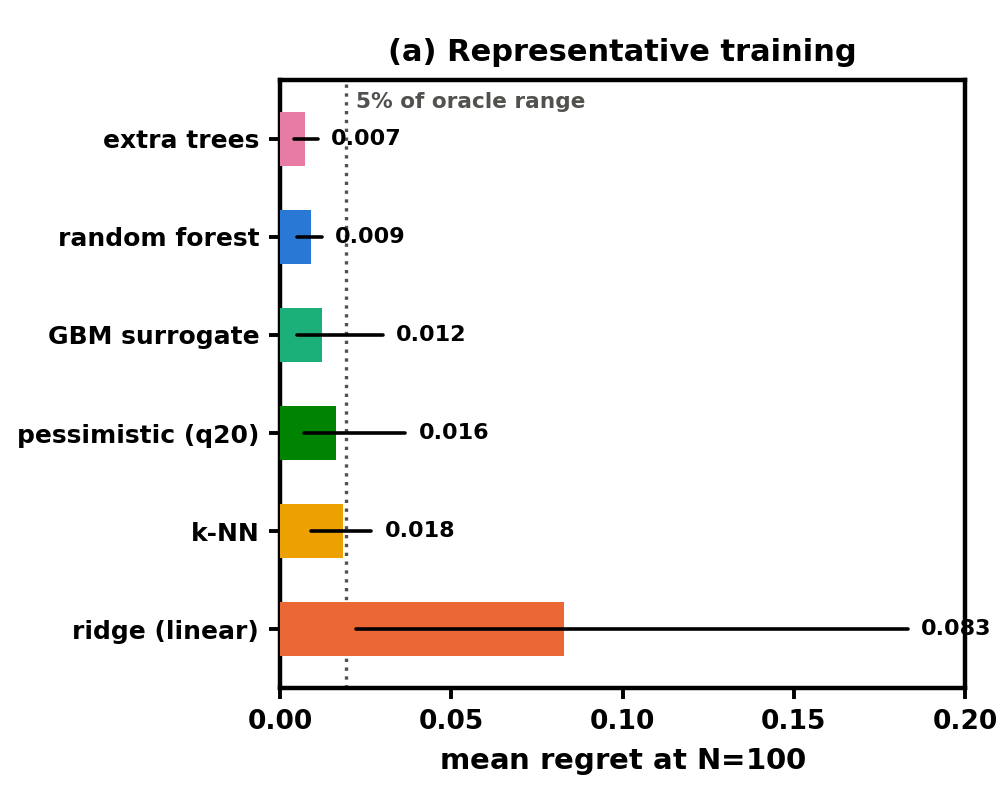}\hfill
\includegraphics[width=0.47\textwidth]{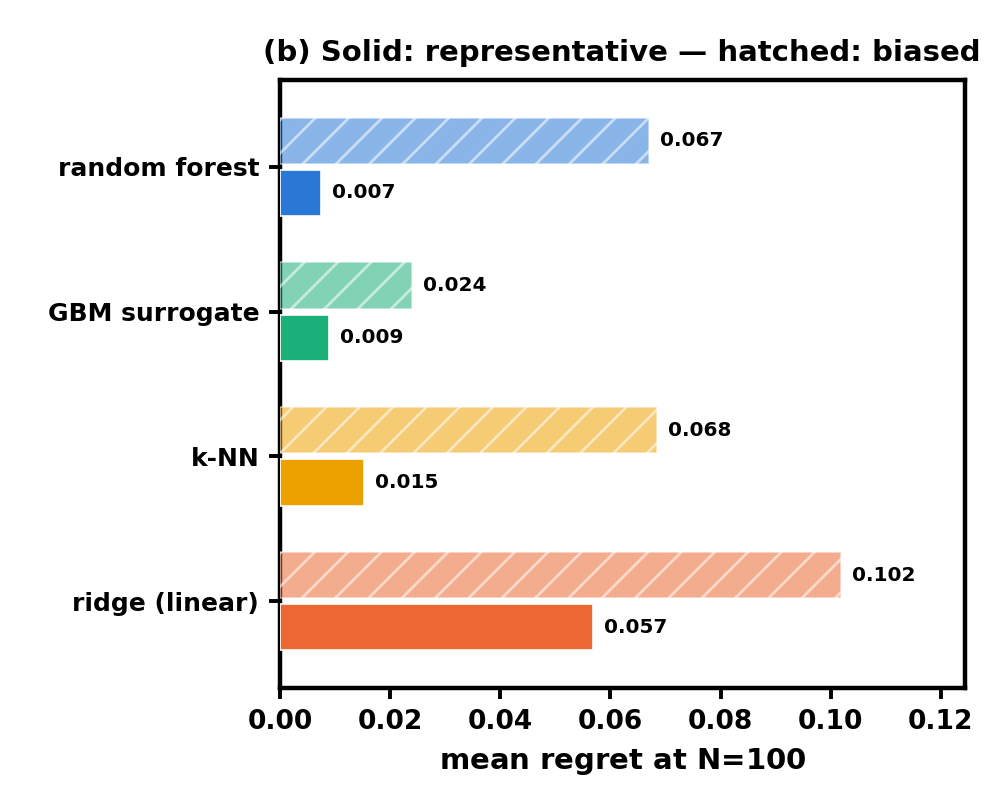}
\caption{Surrogate-zoo decision quality on Task 1: mean regret of the
selected configuration at $N=100$, by model family and training regime.}
\label{si:fig:zoo}
\end{figure}

\begin{figure}[H]
\centering
\includegraphics[width=0.55\textwidth]{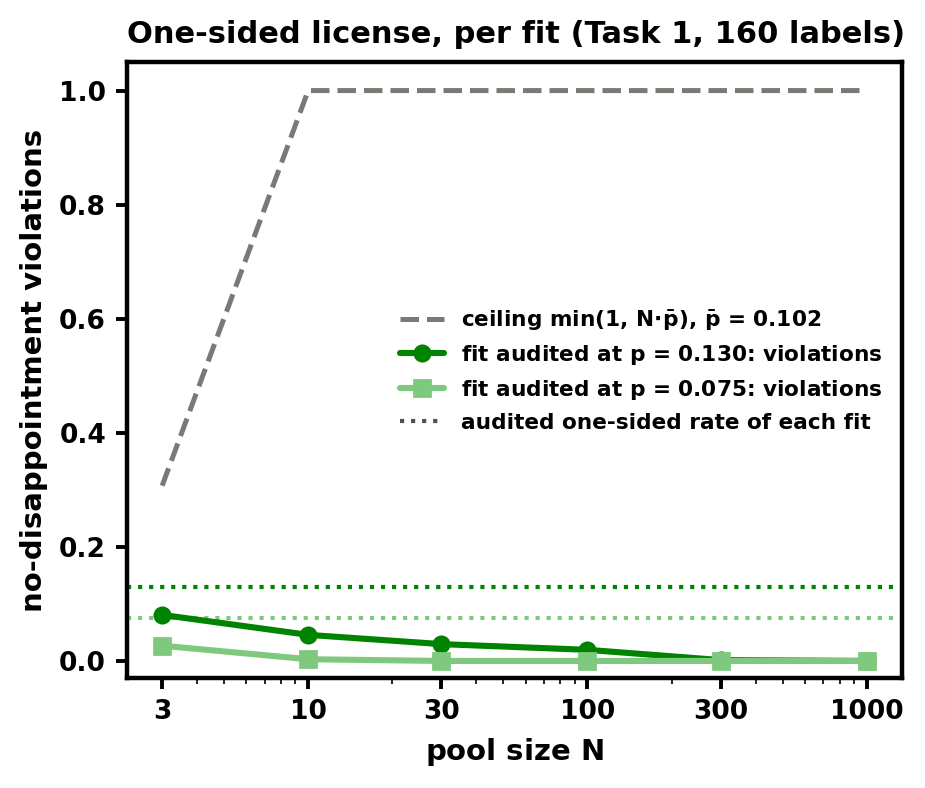}
\caption{The one-sided (no-disappointment) license on Task 1, per fit
(two 160-label quantile fits): measured violation rates stay below each
fit's audited one-sided rate at every $N$ and decay to zero; dashed
curve, the $\min(1,Np)$ ceiling at the mean audited rate.}
\label{si:fig:license}
\end{figure}

\subsection{The Six-Condition Diagnostic Study}\label{si:sec:emp-div}
The expanded diagnostic study covers all three tasks $\times$ both
regimes: six families $\times$ three training sizes $\times$ four seeds
$=72$ fits per condition, 432 total
(bootstrap 95\% CIs over fits). Spearman correlation with deployed regret
at $N=100$:

\begin{center}\small
\begin{tabular}{lccc}
\toprule
Condition & $\qsel$ & $\qinv$ & $R^2$ \\
\midrule
Task 1 / representative & $0.96$ [$0.91$--$0.98$] & $0.86$ [$0.77$--$0.90$] & $0.80$ [$0.68$--$0.87$] \\
Task 1 / biased         & $0.99$ [$0.97$--$1.00$] & $0.78$ [$0.64$--$0.86$] & $0.76$ [$0.62$--$0.85$] \\
Task 2 / representative & $0.89$ [$0.81$--$0.94$] & $0.68$ [$0.50$--$0.80$] & $0.59$ [$0.39$--$0.74$] \\
Task 2 / biased         & $0.89$ [$0.78$--$0.96$] & $0.68$ [$0.50$--$0.81$] & $0.64$ [$0.45$--$0.78$] \\
Task 3 / representative & $0.96$ [$0.92$--$0.98$] & $0.77$ [$0.65$--$0.84$] & $0.38$ [$0.13$--$0.59$] \\
Task 3 / biased         & $0.97$ [$0.95$--$0.98$] & $0.56$ [$0.36$--$0.71$] & $0.33$ [$0.08$--$0.54$] \\
\bottomrule
\end{tabular}
\end{center}

A practical-budget replication
re-estimates $\qsel$ for every fit from exactly $T_a=120$ audit rounds
(two oracle calls each, binomial noise $\pm0.03$) drawn independently of
the evaluation pools: the Spearman correlations become $0.94/0.90$
(Task 1, biased/representative), $0.83/0.81$ (Task 2), and $0.80/0.90$
(Task 3)---attenuated by estimation noise, still the best diagnostic in
all six conditions, and still CI-disjoint from $R^2$ on the
combinatorial benchmark ($[0.70,0.86]$ vs.\ $[0.08,0.54]$ biased;
$[0.84,0.94]$ vs.\ $[0.13,0.59]$ representative). The champion statistic dominates in every
condition at both budgets; its one attenuation in the full-budget table
(Task 2, $0.89$) tracks the near-tie structure of that pool ($24.8\%$ of
configurations within $\varepsilon$ of the optimum, versus $8.4\%$ on
Task 1 and $2.5\%$ on Task 3), which compresses deployed regret toward
zero---the margin-condition regime (Thm~\ref{si:thm:margin}, $\alpha$
large) in which selection matters least. One finding is neutral: on smooth
tasks under representative data, $R^2$ and inversion rates remain
serviceable; the collapse of $R^2$ is specific to the combinatorial
benchmark, and one cannot tell from the diagnostic itself when it will
occur.
Figure~\ref{si:fig:task2} shows the Task-2 replication of the
model-family comparison.

\begin{figure}[H]
\centering
\includegraphics[width=0.55\textwidth]{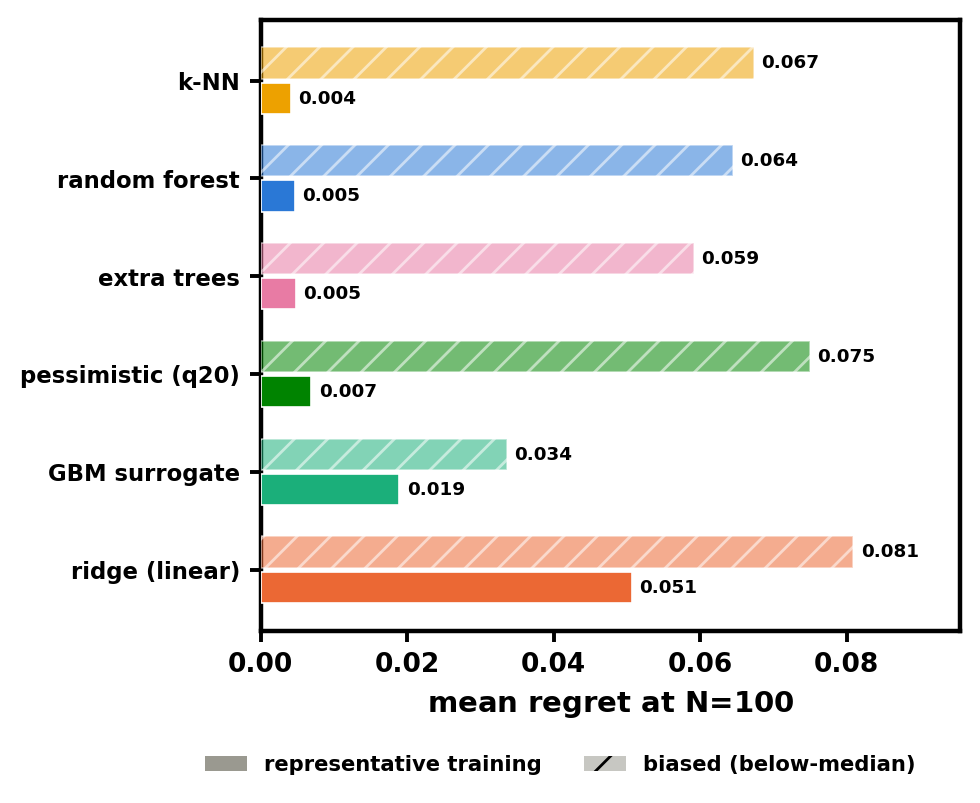}
\caption{Task 2 replication (SVM on breast-cancer data; 576 real
trainings): mean regret of the selected configuration by model family and
training regime.}
\label{si:fig:task2}
\end{figure}

\subsection{NAS-Bench-201}
On the community benchmark (all 15{,}625 architectures with real trained
accuracies; CIFAR-10 validation as the search oracle,
$\varepsilon=\tau=2$ accuracy points) every phenomenon replicates and two
sharpen. (i) The champion statistic $\qsel$ correlates with deployed
regret at $\ge0.96$ in both regimes while $R^2$ collapses in both
(Section~\ref{si:sec:emp-div})---on a recognized benchmark, accuracy is
close to uninformative about deployed performance. (ii) The linear
surrogate's deployed failure reaches $0.95$ at $N=100$ and $1.0$ at
$N=1000$ \emph{even with representative training data}: across our three
tasks, benign-versus-malign error alignment is a (fit, task, data)
property that only an audit can determine. (iii) Audited screening
achieves a $25\times$ measured cost separation at $M=100$, with champions
averaging $87.1\%$ test accuracy against a global best of $89.2\%$ at one
verification call per design.

\subsection{Licenses under Objective Drift}\label{si:sec:emp-drift}
NAS-Bench-201 carries trained validation accuracies for the same 15{,}625
architectures on three datasets, which supplies real \emph{objective}
drift: a surrogate trained and audited against CIFAR-10 validation
accuracy can be deployed against the CIFAR-100 or ImageNet16-120
objective, with the ground truth of both known exhaustively. The protocol uses three surrogate families (random
forest, GBM, ridge) $\times$ four seeds, with 200 CIFAR-10 training
labels each; quantile licenses per Theorem~\ref{si:thm:si2} at $M=100$,
$\rho=0.05$, $\delta=0.05$ ($T_a=120$ champion audits per license);
violation rates measured on 20{,}000 fresh deployment pools. Because the
champion depends only on the surrogate, selections are identical across
objectives; what drifts is the truth the license answers to.

\begin{center}\small
\begin{tabular}{lccc}
\toprule
 & home (CIFAR-10) & CIFAR-100 & ImageNet16-120 \\
\midrule
license violation rate (bound $4\rho=0.20$) & $0.03$--$0.10$ & --- & --- \\
renewed-license violation rate ($T_a=120$ calls) & --- & $0.06$--$0.13$ & $0.08$--$0.13$ \\
champion percentile under that objective & $0.79$--$0.94$ & $0.80$--$0.94$ & $0.78$--$0.93$ \\
global Spearman$(\fh,\ \text{objective})$ & $0.59$--$0.70$ & $0.61$--$0.69$ & $0.61$--$0.66$ \\
$\qsel$ re-measured against that objective & $0.04$--$0.15$ & $0.04$--$0.13$ & $0.03$--$0.14$ \\
\bottomrule
\end{tabular}
\end{center}

Three observations. (i) Every license held: 12 of 12 at home and 24 of 24
after renewal, all below the $4\rho$ ceiling---and a renewal costs the
same $120$ oracle calls as the original audit, amortized over every
decision the renewed license covers.
(ii) The license \emph{value} does not transfer (it is denominated in the
audited objective); the license \emph{procedure} does---which is the
operational content of license expiry: void on drift, re-earn cheaply.
(iii) On this benchmark family, the global diagnostic is blind to the
drift: the surrogate's global Spearman is ${\sim}0.64$ whether scored
against the audited objective ($0.59$--$0.70$) or the drifted ones
($0.61$--$0.69$ and $0.61$--$0.66$; the objectives themselves remain
rank-correlated at $0.93$--$0.97$), while champions remain ${\sim}87$th
percentile and $\qsel$ ${\sim}0.08$ under all three truths: a mid-range
global correlation is again close to uninformative about deployment,
here underselling it, complementing the overselling failure on the
biased-regime study. The audit statistic prices what matters---the top
of the ranking under selection---in both directions, at two oracle
calls per round.

\subsection{Candidate-Distribution Drift and License Renewal}
\label{si:sec:emp-ddrift}
Remark~\ref{si:rem:expiry} bounds license degradation under
proposal-distribution drift additively by $M\beta$,
$\beta=\TV(D',D)$---a worst-case bound that is vacuous at deployment
pool sizes for any non-trivial drift. This experiment measures what actually happens, with
$\beta$ known exactly: $D'=(1-w)D+w\,\mathrm{Unif}(A)$, where $A$ is the
adverse (weak) third of NB201 ($34.5\%$ of cells with ${\ge}2$
\emph{none} operations; mean accuracy $0.647$ vs.\ $0.689$ global), so
$\TV(D',D)=w(1-|A|/n)$ exactly. Twelve licensed fits (the protocol of
Section~\ref{si:sec:emp-drift}) and four drift weights are used; we test
the adverse direction only, since drift
toward stronger candidates can only raise champion values and therefore
only slackens this one-sided license:

\begin{center}\small
\begin{tabular}{lcccc}
\toprule
drift weight $w$ & $0.05$ & $0.15$ & $0.30$ & $0.60$ \\
\midrule
$\beta=\TV(D',D)$ & $0.033$ & $0.098$ & $0.197$ & $0.393$ \\
transfer bound $\min(1,4\rho+M\beta)$ & $1.0$ & $1.0$ & $1.0$ & $1.0$ \\
stale-license violation (12 fits) & $0.07$--$0.18$ & $0.06$--$0.19$ & $0.07$--$0.22$ & $0.08$--$0.27$ \\
renewed-license violation & $0.06$--$0.18$ & $0.06$--$0.15$ & $0.06$--$0.16$ & $0.06$--$0.12$ \\
\bottomrule
\end{tabular}
\end{center}

Violation rates at the home distribution were $0.06$--$0.17$, all within
the $4\rho=0.20$ ceiling.
Three observations. (i) The worst-case $M\beta$ bound is uninformative
at every tested drift (it saturates at $1$ already for $\beta=0.033$),
yet measured stale-license degradation is gradual: the ceiling is first
breached only at $\beta\approx0.2$ (max $0.218$), reaching $0.27$ at
$\beta=0.39$. The gap has a structural reason: the champion is the max
of $M$ draws, so mixing weak candidates into the pool rarely changes
the argmax until the mixture dominates. (ii) Renewal under $D'$ (the
same $T_a=120$-call audit, pools drawn from $D'$) restored validity at
every weight: 48 of 48 renewed licenses within the ceiling. (iii) Since
$\beta$ is unobservable in practice while renewal is cheap, the
operational content of Remark~\ref{si:rem:expiry} is not the bound but
the doctrine: treat licenses as void on drift and re-earn them.

\subsection{Comparison with a Marginal Conformal Floor}
\label{si:sec:emp-conformal}
The simplest deployed alternative to the audit license is a one-sided
\emph{marginal} split-conformal floor:\cite{jin2023selection} label
$n_{\mathrm{cal}}$ i.i.d.\ candidates, take the
$\lceil(n_{\mathrm{cal}}{+}1)(1{-}\alpha)\rceil$-th smallest
over-prediction score $q$, and attach the floor $L(x)=\fh(x)-q$ to any
candidate. We compare at matched oracle cost and matched nominal level
($n_{\mathrm{cal}}=T_a=120$; $\alpha=4\rho=0.2$; 16 configurations:
Tasks 1 and 3 $\times$ RF/ridge $\times$ both regimes $\times$ 2 seeds).
On \emph{fresh}
candidates from the calibration distribution the conformal floor behaves
exactly as advertised (miscoverage $0.12$--$0.23$ around the nominal
$0.20$ across the 16 configurations). At the \emph{champion} of an
$M=100$ screen---the only candidate a screening campaign acts on---its
violation rate ranged from $0.00$ to $0.99$ (Task 1 mean $0.35$; Task 3
mean $0.48$; worst configuration: biased-trained ridge on NB201,
$0.99$), while the champion-audit license held at $0.04$--$0.15$ in
every configuration, against the same $0.20$ guarantee. Tightness does
not rescue the marginal floor either: in the 6 of 16 configurations
where it happened to remain valid at champions, its mean floor sat at
the $0.68$ quantile of $f$ against the audit floor's $0.84$---the audit
was both the valid and the tighter certificate. This is
Theorem~\ref{si:thm:audit}'s audit--deployment gap acting on a
\emph{certificate}: marginal calibration, like marginal auditing,
prices errors where selection does not look. Selection-conditional
conformal methods\cite{jinren2025focal} repair exactly this and
remain the natural machinery to import where their exchangeability
requirements hold (Section~\ref{si:sec:open}); the audit license
achieves selection-validity with marginal-conformal simplicity.

\subsection{The Measured Price of Certificates}
\label{si:sec:emp-certcost}
To complement the recommendation loop of Figure~\ref{fig:loop} with the
\emph{certified} protocol it deliberately is not, we ran \textsc{Gated-SE}
(Theorem~\ref{si:thm:se}) on the full 1{,}152-config Task-1 table with
$\sigma=0.05$ noisy queries, $\delta=0.1$
(10 seeds per tolerance). At
$\varepsilon=5\%$ of the objective range, every run returned a correct
$\varepsilon$-certificate (10/10; mean certified regret $0.000$) at a
mean cost of $480{,}555$ oracle calls; at $\varepsilon=10\%$, 10/10
correct at $274{,}144$ calls. Against the 340-call uncertified
recommendation loop, the certificate premium is three orders of
magnitude ($1{,}413\times$ at $\varepsilon=5\%$)---the tabular floor and near-tie core
(Theorem~\ref{si:thm:margin}; $8.4\%$ of this pool lies within
$\varepsilon$ of the optimum) made concrete, and the quantitative case
for the licensed generative route of Theorem~\ref{si:thm:si2}, which
buys its guarantees per \emph{campaign} rather than per \emph{menu}.

\subsection{Which Fits Go Malign? Label-Free Predictors}
\label{si:sec:emp-align}
Section~\ref{si:sec:emp-tax} documented that malignancy is a per-fit
property; open problem (5) asks which fit-level quantities predict it.
As a first step, we computed four \emph{deployment-label-free} predictors for
120 Task-1 fits (ridge and GBM, train size 120, 30 seeds $\times$ both
regimes; 27 malign fits with $q(1000)>0.5$): the mean feature-space
distance
of the fit's top-10 ranked candidates to the nearest training point; the
champion prediction's overshoot above the largest training label; the
top-10's Mahalanobis distance from the training distribution; and the
top-10 predicted-value spread.

\begin{center}\small
\begin{tabular}{lcccc}
\toprule
AUC for malign ($q(1000)>0.5$) & top-dist & overshoot & Mahalanobis & top-spread \\
\midrule
ridge fits ($n=60$, 20 malign) & $0.97$ & $0.996$ & $0.93$ & $0.62$ \\
GBM fits ($n=60$, 7 malign) & $0.35$ & $0.37$ & $0.44$ & $0.53$ \\
pooled ($n=120$) & $0.64$ & $0.65$ & $0.70$ & $0.56$ \\
\bottomrule
\end{tabular}
\end{center}

The result is sharply two-sided. For the linear family, extrapolation
geometry is almost sufficient: a ridge fit whose top-ranked candidates
sit far from its training data, or whose champion prediction overshoots
everything it has seen, is almost surely malign (AUC $0.97$--$0.996$);
conversely, benign ridge fits are the ones whose top-ranked candidates
sit \emph{at} training points.
For tree ensembles the same geometry carries no usable signal in this
sample (AUC $0.35$--$0.53$ with only $7$ malign GBM fits;
Mann--Whitney $p\ge0.19$): their malignancy arises from over-fit
optimism at the very top of the ranking rather than geometric
extrapolation, and none of the four quantities detects it here. We report this exploratory finding
(one task, one training size) as evidence that per-fit alignment admits
model-\emph{specific} precursors but, so far, no model-agnostic
one---which is precisely the gap the model-agnostic selection-aware
audit fills at two oracle calls per round.

\subsection{The Label Budget: Landscapes, Policies, and Damage
Protocols}\label{si:sec:emp-budget}

This section gives the complete protocols behind
Section~\ref{sec:res-budget}, and records the one formal statement that
study rests on.

\begin{proposition}[the audit term of the budget is a floor]
\label{si:prop:floor}
Fix $N\ge2$, $\rho\in(0,1)$, $\delta\in(0,1)$, and let a license be
issued iff $N\,\bar q_{\delta}(K_T,T)\le\rho$, where $K_T$ is the
number of regret events in $T$ champion-audit rounds and
$\bar q_{\delta}$ the one-sided binomial upper confidence limit. Then
(i) no outcome with $T<T_{\min}:=\min\{T:N\bar q_{\delta}(0,T)\le\rho\}$
issues the license, for any surrogate, so every campaign pays at least
$2T_{\min}$ audit calls; (ii) $T_{\min}\le
\lceil \ln(1/\delta)/\ln(1/(1-\rho/N))\rceil = (N/\rho)\ln(1/\delta)\,
(1+O(\rho/N))$; and (iii) no alternative audit design certifies the
same deployed-failure guarantee below order $N/\rho$ calls.
\end{proposition}

\begin{proof}
(i) is monotonicity of $\bar q_{\delta}(k,T)$ in $k$: the bound at $k=0$
is the best achievable at depth $T$, so if even it exceeds $\rho/N$ no
event count can pass. (ii) At $k=0$ the Clopper--Pearson limit solves
$(1-\bar q)^T=\delta$, giving $\bar q_{\delta}(0,T)=1-\delta^{1/T}$;
solving $N(1-\delta^{1/T})\le\rho$ yields the stated $T$. (iii) is
Theorem~\ref{si:thm:lb}, the universal audit lower bound of
Section~\ref{si:sec:lb}, which applies to arbitrary audit designs.
\end{proof}

\paragraph{Landscapes.} All studies use tables of $12{,}000$ candidates
with descriptors drawn uniformly from $[0,1]^d$ and a noiseless oracle.
The \emph{feasible} landscape places a flat-topped optimum (a plateau of
height $120$ and radius $0.16$ in the first two descriptors, Gaussian
falloff of scale $0.004$ outside it) against three deceptive mid-value
decoys (Gaussian peaks of heights $95$, $90$, $85$ and scales
$0.012$, $0.010$, $0.015$) plus a low-amplitude sinusoidal texture
(amplitude $3$); its top-$\varepsilon$ set has mass ${\approx}0.10$, and the
champion-audit statistic returns $0.037$ for a surrogate scoring at the
truth, which is the floor the maps in
Figure~\ref{fig:budgetmaps} approach from above. The \emph{easy} variant
uses $d=2$; the \emph{hard} variant appends six independent nuisance
descriptors ($d=8$) that leave the oracle unchanged. The
\emph{crowded-race} landscape replaces the plateau with near-tied narrow
peaks (heights $100/80/70/60$, sinusoidal amplitude $25$), for which
the champion-audit statistic returns ${\approx}0.73$ even for a
surrogate scoring at the truth, at the same
$(\varepsilon,N,\rho)$. Per Remark~\ref{rem:conservative} that number
bounds what this statistic can certify here; it does not bound what a
selector can achieve, since regret for a perfect selector is zero on
this landscape as on every other. Throughout, $N=100$, $\rho=0.10$,
$\delta=0.05$, and $\varepsilon=\tau$ at $5\%$ of the oracle range.

\paragraph{Policies and evaluation.} Policies label in batches of $100$
(first batch uniform), to $3{,}200$ labels, three seeds: \emph{uniform}
(random batches); \emph{exploit-top} (label the highest-predicted
unlabeled candidates under a $120$-tree random forest refit each
batch); \emph{uncertainty} (largest tree-ensemble spread);
\emph{hybrid} (half exploit, half uniform). At each checkpoint a
$200$-tree random forest is fit to the labeled set (predictions on
labeled candidates replaced by their labels), and its selector is
audited by the champion audit for $8{,}000$ rounds. Verdicts at
smaller audit depths are evaluated on the audit's prefix, which is
distributionally identical to a shorter audit because rounds are
i.i.d.; this makes the full budget plane cost one audit per
checkpoint. Certified bounds reported are
$N\bar q_{\delta}(\text{events},T_a)$; budget-plane maps average three
seeds and smooth for display only (log-domain Gaussian,
$\sigma=0.8$ grid cells); all quoted numbers are unsmoothed.

\paragraph{Frontier measurements.} On the hard landscape,
labels-to-earnable at $\rho\in\{0.30,0.15,0.10,0.05\}$ were
$200/200/400/800$ (uniform), $200/200/400/800$ (uncertainty),
$400/800/1{,}600/3{,}200$ (hybrid), and
$800/1{,}600/1{,}600/3{,}200$ (exploit-top), against audit floors
$2T_{\min}=1{,}996/3{,}992/5{,}990/11{,}980$; the cheapest
$\rho=0.10$ license on the dense grid is $600+6{,}742=7{,}342$ calls
(uniform) versus $1{,}800+6{,}742=8{,}542$ (exploit-top). On the easy
landscape every policy is within $2\%$ of the audit-only floor. On the
crowded-race landscape no policy improves the bound at any budget: by
Proposition~\ref{si:prop:floor}(i) no outcome below $T_{\min}$ issues
a license, and above it the statistic is already saturated at its
idealized-surrogate value, so labeling cannot help.

\paragraph{Damage and repair protocols.} Starting from perfect
predictions on the easy landscape, a fraction
$f\in\{0.005,\dots,0.20\}$ of candidates is overestimated by
$s\in\{0.5,1,2,4,8\}\times\varepsilon$ (three seeds). The certified
bound stays at the $k{=}0$ Clopper--Pearson floor for $s\le\varepsilon$ at every $f$
tested; the $\rho=0.10$ boundary lies at $s\approx1.2$--$2.3\,
\varepsilon$ across two decades of $f$; the case
$(f,s)=(0.05,\varepsilon)$ certifies at bound $0.051$ after
${\approx}3{,}800$ rounds. For failing models ($s=2\varepsilon$),
predictions are corrected (set to oracle values, one call each) in
doubling batches under two orders---decreasing predicted value
(``verify the champions'') versus uniformly at random---until the
license is restored: $50$--$530$ corrections suffice under the
targeted order (rising with $f$), while random correction restored no
tested case within $6{,}400$.

\section{Relation to Prior Work}\label{si:sec:novelty}
This section separates what is classical from what we claim, result by
result, and states explicitly where an adjacent literature requires
differentiation. We take the most conservative reading available in every
case: where a result is standard in another setting, we say so and cite
it, and we claim only the assembly.

\subsection{Classical or Known Results, Cited and Not Claimed}
Certificates and elimination rules (Prop.~\ref{si:prop:cert}) are LUCB/UGapE/action-elimination;\cite{si-kalyanakrishnan2012pac,si-gabillon2012best,si-evendar2006action} the $\delta/t^2$ schedule and its LIL sharpenings;\cite{si-evendar2006action,si-jamieson2014best,si-jamieson2014lil,si-howard2021time} the regret decomposition;\cite{si-jin2021pessimism,smith2006optimizer} simulation lemma and action-gap stability;\cite{si-kearns2002near,si-singh1994upper,si-farahmand2011action} sub-Gaussian maxima and winner's-curse lineage;\cite{si-boucheron2013concentration,smith2006optimizer,capen1971competitive,si-berk2013valid} the change-of-measure technique of Theorem~\ref{si:thm:noprofit};\cite{kaufmann2016complexity,si-mannor2004sample} the upper bound of Theorem~\ref{si:thm:audit}(i) is an elementary union bound.

\subsection{Adjacent Work Requiring Explicit Differentiation}
\begin{itemize}[leftmargin=1.4em,itemsep=2pt]
\item \emph{Theorem~\ref{si:thm:ordinal}.} Sufficiency of exact order preservation for whole-space argmax is Theorem~1 of Tan et al.;\cite{lyu2025rank} ranking-error upper bounds on regret appear in Lyu et al.\cite{si-lyu2026learnability} The pool-uniform, $\varepsilon$-margin \emph{if-and-only-if} was not found; we position it as a clarifying lemma rather than a principal contribution.
\item \emph{Theorem~\ref{si:thm:noprofit}.} Verification-versus-identification equivalence is known for structured bandits;\cite{karnin2016verification} sequential tests for ``is this arm best'' with matching lower bounds exist.\cite{kaufmann2016complexity,si-kaufmann2021mixture} Our framing (auditing an \emph{external surrogate's} decision, tolerance inflating the effective gap) was not found stated; the mathematical content is standard.
\item \emph{Theorem~\ref{si:thm:pessimism}.} This is the result with the greatest overlap with prior work. Conformal selection controls FDR among prediction-selected units;\cite{jin2023selection} selection-conditional conformal coverage---including top-$K$ and hence argmax selection---delivers one-sided post-selection bounds with exchangeable calibration and no union-bound $N$-factor;\cite{jinren2025focal} coverage for designed/selected candidates under feedback covariate shift exists for MBO.\cite{fannjiang2022conformal} What remains ours is only the assumption-light packaging (abstract audited rate $p$, explicit $Np$ deployed-error accounting, license framing); where exchangeability holds, the machinery of Jin and Ren\cite{jinren2025focal} is the sharper instrument and should be preferred.
\item \emph{Theorem~\ref{si:thm:audit}.} The qualitative phenomenon---selection amplifies rare over-estimates---is the optimizer's curse\cite{smith2006optimizer} and appears empirically in verifier-based resampling\cite{stroebl2024flaws} and reward-overoptimization studies;\cite{gao2023scaling,si-beirami2024theoretical} deployment-simulating evaluation has been advocated empirically;\cite{si-williams2026simulating} audit-cost lower bounds of a different origin appear in Srivastava.\cite{si-srivastava2026limits} We found no prior statement of the matching worst-case construction making $\Theta(\min(1,Np))$ tight \emph{as a function of an audited failure rate}, nor of the $\Omega(N/\rho)$-versus-$O(\rho^{-1}\log(1/\delta))$ oracle-query separation between passive and selection-aware audit designs; these two elements are the defensible core of Theorem~\ref{si:thm:audit}.
\end{itemize}

\subsection{The Claimed Contributions}
\begin{enumerate}[leftmargin=1.4em,itemsep=2pt]
\item The \textbf{separation-principle pair} (Thms~\ref{si:thm:invariance},~\ref{si:thm:contamination}): a certification architecture provably invariant to arbitrary pseudo-label bias, with a probability-one failure construction showing the architecture is necessary, in an adaptive inverse-design loop where the surrogate's own outputs re-enter as labels. Checked against four components (adaptive search; surrogate labels inserted as observations; decision-stability abstention; a regret--cost guarantee under adaptive selection), no paper we located has all four; the nearest, each with three of the four, are the selective-sampling framework of Sekhari et al.\cite{si-sekhari2023selective} (no labels entering the learner), active-query RLHF\cite{si-ji2024rlhf} (uncertainty gate, not decision stability), MF-GP-UCB\cite{kandasamy2019mfbo} (given physical fidelities, not the loop's own surrogate), and CARE\cite{si-wang2025care} (no guarantee).
\item The \textbf{margin-law oracle-complexity bound} for certified search over a finite pool in three regimes (Thm~\ref{si:thm:margin}); the condition and technique are known in adjacent settings,\cite{si-perchet2013bandits,si-farahmand2011action,si-castro2008minimax,si-simchowitz2019gap} but we did not find this statement assembled.
\item The \textbf{audit--deploy separation} of Theorem~\ref{si:thm:audit}(ii)(iii), as above.
\item Sharp two-directional statements of the accuracy/decision separations (Thm~\ref{si:thm:r2}; known in spirit\cite{si-cameron2022perils,si-elmachtoub2022smart,si-lyu2026learnability}).
\end{enumerate}
No literature search can establish the absence of prior work, and these
claims should be read as the outcome of a systematic but finite one.

\subsection{Positioning of the Strict-Improvement Results}
Theorem~\ref{si:thm:si1} is known: the $\Omega(1/p_v)$ lower bound
is classical quantile-bandit
territory;\cite{chaudhuri2017quantile,aziz2018pure} the verifiable-versus-unverifiable distinction of Katz-Samuels and Jamieson\cite{si-katzsamuels2020good} is the closest conceptual neighbor for audit-conditioning. Theorem~\ref{si:thm:si2} is a recombination: calibration rates and abstention are Learn-then-Test and selective-guarantee territory,\cite{si-angelopoulos2025ltt,si-geifman2017selective} selection-conditional validity is that of Jin and Ren;\cite{jinren2025focal} the nearest single relatives are best-arm identification with LLM judges\cite{ao2026llmjudges} (cheap LLM judge + limited human audits; no quantile license, no separation) and MF-GP-UCB\cite{kandasamy2019mfbo} (\emph{assumes} the cheap-signal error bound that we \emph{audit}). For Corollary~\ref{si:cor:si3} we located no counterpart: no screened-versus-unscreened certified-oracle-cost separation with a matching baseline lower bound appears in cascades, virtual screening (enrichment factors are never converted to budget theorems), conformal selection, multi-fidelity BO, or best-of-$n$; the nearest neighbor requiring differentiation is the coverage-audit bound of Anthony and Salehzadeh Nobari\cite{si-anthony2026audit} (audit cost of a screen's missed recall---same $1/p$ family, different certified object). The ingredients of Theorem~\ref{si:thm:ag1} are classical---record values and exchangeability---but we did not find the audit statement itself. For the tightness half of Theorem~\ref{si:thm:ag2} we located no $\Theta(N^2)$ matching-gap statement in learning-to-rank regret transfer (pairwise-to-listwise bounds are one-directional\cite{si-chen2009ranking}), noisy max-finding, or selective inference. The prescription of Corollary~\ref{si:cor:ag3} appears to be new, with risk-limiting election audits\cite{stark2012rla} as precedent. One further neighbor deserves explicit treatment: conformal candidate certification for offline
MBO\cite{si-choi2026ccc} attaches calibrated one-sided lower bounds to proposed designs, but its guarantee is, by that paper's own statement, \emph{marginal over a fresh candidate} and does not imply post-selection validity conditional on selection---i.e., it is the non-selection-aware ancestor of our Theorem~\ref{si:thm:si2} licenses, and its explicit disclaimer marks precisely the gap the selection-aware audit fills.

\subsection{Proxy-Assisted Bandits and Winner's-Curse Corrections}
Five further neighbors require differentiation.
(1)~PROBE:\cite{ma2026probe} fixed-confidence best-arm
identification where every costly pull is paired with a cheap correlated
proxy of \emph{known} marginal mean; control-variate analysis attains
near-oracle sample complexity. The distinction is fixed-menu
identification with an assumed proxy model versus our generative-goal
certification with a \emph{measured} license: their guarantee consumes
proxy correlation as an input, whereas our audit outputs it. (2)~MLA-UCB:\cite{si-ji2025mlaucb}
cumulative-regret bandits with ML-generated surrogate rewards, provably
robust to arbitrary surrogate-mean misalignment. The distinction is
regret minimization versus certification, with no license object.
(3)~Analytic\cite{iyengar2023oic} and order-statistic\cite{mclatchie2024order}
corrections for the optimizer's-curse bias of a selected
maximum---the estimation-side complement of
Theorem~\ref{si:thm:evt}'s worst-case bounds (they de-bias the reported
value; we bound and audit the deployed failure). (4)~Instance-optimal multi-fidelity best-arm
identification\cite{poiani2024mfbai}---the sharpest
form of the assumed-fidelity paradigm that
Theorem~\ref{si:thm:si2} replaces with a measured one.
(5)~Conditional inference for winners\cite{si-bakshi2026winners} and the
inflated-argmax line:\cite{si-adrian2025inflated} post-selection inference and stability for
selected winners---inference on the winner's \emph{value} after
selection, complementary to our pre-deployment licensing of the
\emph{selector}. For Theorem~\ref{si:thm:lb}, no lower bound over
arbitrary audit designs for a deployed-selection guarantee was located in
the property-testing, active-testing,\cite{si-kossen2021active} or
safety-evaluation\cite{si-srivastava2026limits} literatures; the nearest results bound
passive-sample audits only. Theorem~\ref{si:thm:noisy}'s slack-not-rounds
form was not found; noisy-quantile analyses assume margins rather than
price them.

\section{Limitations and Open Problems}\label{si:sec:open}
The theory presented here settles five questions: strict improvement for
value targets (the dichotomy of Section~\ref{si:sec:si}); the ordinal
audit gap, $\Theta(N)$ against $\Theta(N^2)$, tight; the audit-design
floor, in the strong total-query form of Theorem~\ref{si:thm:lb} rather
than merely over two-call designs; the noisy-oracle license
(Theorem~\ref{si:thm:noisy}); and the policy-selection half of the
sequential extension, with the replay-buffer contamination result made
formal (Section~\ref{si:sec:seq}). Five questions remain open.
(1) \textbf{Regret-type strict improvement} in structured classes
(linear/RKHS)---the tabular floor of Remark~\ref{si:rem:floor} forbids
it in pools, and Theorem~\ref{si:thm:si2} licenses values, not regret;
this is the principal
open problem. (2) \textbf{Credit assignment}: certifying a policy
\emph{improvement} step (not a selection among $K$ fixed policies) from
real transitions only, in the compounding-error regime
(Remark~\ref{si:rem:seq-open}); the selection half is settled by
Corollary~\ref{si:cor:seq}. (3) Sharper selection-conditional licenses by importing the
selection-conditional conformal machinery of Jin and Ren\cite{jinren2025focal}
into the loop, where exchangeability survives
adaptivity. (4) Constants: the gap between the
$(\tfrac12-\delta)\lfloor N/4\rho_0\rfloor$ floor of
Theorem~\ref{si:thm:lb} and the champion audit's
$O((N/\rho_0)\log(1/\delta))$ upper bound, and noise-optimal audit
designs interpolating between Theorem~\ref{si:thm:noisy}'s slack and
repeated-query averaging. (5) A distributional theory of the per-fit
benign/malign bimodality documented in Section~\ref{si:sec:emp-tax}:
which fit-level quantities of a trained surrogate predict its alignment?
Section~\ref{si:sec:emp-align} gives a first empirical
answer---extrapolation geometry is nearly sufficient for linear fits (AUC up to
$0.996$) yet carries no signal for tree ensembles---so the open problem
sharpens to: is there a model-agnostic, deployment-label-free precursor
of malignancy, or is the audit's model-agnosticism essential?

Our guarantees are for the single-decision (pool/generative) setting; results are distribution-free but finite-sample constants are not optimized; the strict-improvement license certifies \emph{values}, not regret, and all cost gains are conditional on audits passing---we view the conditionality as a feature (trust is measured, never assumed), but it means no unconditional speedup is claimed.

\section{Conclusion}
Train on anything; certify only with clean oracle statistics; audit where the search actually points; re-audit when anything moves. Each clause of that sentence is now a theorem with an experiment attached. The framework converts ``is this surrogate good enough to trust?'' from a judgment call about accuracy into a measurable, license-shaped question---one we intend to ask, next, of surrogate models proposed across the design sciences.

\section{Numerical Verification and Reproducibility}\label{si:sec:verify}
Every checkable claim is exercised by scripted numerical tests (four
suites comprising more than 50 checks, all of which pass; to be released
with a subsequent \textsc{ModelAssay} release): exact algebra of
Theorems~\ref{si:thm:r2}--\ref{si:thm:evt}; empirical anytime coverage of
Lemma~\ref{si:lem:cs} under an adversarial adaptive query rule; the
deterministic failure trace of Theorem~\ref{si:thm:contamination};
exhaustive verification of the Theorem~\ref{si:thm:ordinal} equivalence
over all pools of small ground sets; Monte Carlo agreement of
Theorem~\ref{si:thm:audit}(ii) to three decimals; the cost-regime
behavior of Theorem~\ref{si:thm:margin}; the unsafe-instance parameters
and hitting-probability identity of Theorem~\ref{si:thm:lb}; and the
noisy-license slack of Theorem~\ref{si:thm:noisy}, for which the
asymmetric mean-zero noise counterexample of
Remark~\ref{si:rem:noisyslack} is simulated directly: the $\gamma=0$
license fails in $100\%$ of simulated audits while the slack license
holds in all of them (symmetric noise cannot break the no-slack license;
it only inflates the rate to $8\rho$). All experiments and verification
suites were re-run from scratch in a fresh computational environment,
reproducing every reported statistic to the third decimal.

\end{document}